%% file: main.tex
\documentclass[twoside,11pt]{article}

\usepackage[preprint]{jmlr2e}

\usepackage{bm}
\usepackage{amsmath,amsfonts,amssymb}
\usepackage{mathtools}
\usepackage{dsfont}
\usepackage{enumitem}
\usepackage{Mydef}

\usepackage[toc,page,header]{appendix}
\usepackage{titletoc}

\newtheorem{assum}{Assumption}

\allowdisplaybreaks

\newcommand{\parhead}[1]{%
  \paragraph*{\normalfont\scshape #1}%
  \ignorespaces
}

\ShortHeadings{RLD: Strong Convergence of GEM}{Zhan and Sugiyama}
\firstpageno{1}

\begin{document}

\title{Riemannian Langevin Dynamics: Strong Convergence of Geometric Euler–-Maruyama Scheme}

\author{\name Zhiyuan Zhan \email zhan@ms.k.u-tokyo.ac.jp \\
      \addr The University of Tokyo, Japan \\
      RIKEN Center for AIP, Japan
      \AND
      \name Masashi Sugiyama \email sugi@k.u-tokyo.ac.jp\\
      \addr RIKEN Center for AIP, Japan \\
      The University of Tokyo, Japan}

\editor{My editor}

\maketitle

\begin{abstract}%   <- trailing '%' for backward compatibility of .sty file
    Low-dimensional structure in real-world data plays an important role in the success of generative models, which motivates diffusion models defined on intrinsic data manifolds. Such models are driven by stochastic differential equations (SDEs) on manifolds, which raises the need for convergence theory of numerical schemes for manifold-valued SDEs. In Euclidean space, the Euler--Maruyama (EM) scheme achieves strong convergence with order $1/2$, but an analogous result for manifold discretizations is less understood in general settings. In this work, we study a geometric version of the EM scheme for SDEs on Riemannian manifolds and prove strong convergence with order $1/2$ under geometric and regularity conditions. As an application, we obtain a Wasserstein bound for sampling on manifolds via the geometric EM discretization of Riemannian Langevin dynamics.
\end{abstract}


\input{main_text}

% Acknowledgements and Disclosure of Funding should go at the end, before appendices and references

\acks{ZZ was supported by Institute for AI and Beyond at the University of Tokyo. MS was supported by JST ASPIRE Grant Number JPMJAP25B1.}


%Save Old Theorem Eviroment
\let\oldthm\thm \let\endoldthm\endthm
\let\oldlem\lem \let\endoldlem\endlem
\let\oldprop\prop \let\endoldprop\endprop
\let\oldcor\cor \let\endoldcor\endcor
\let\oldexam\exam \let\endoldexam\endexam
\let\oldrmk\rmk \let\endoldrmk\endrmk
\let\olddefn\defn \let\endolddefn\enddefn

\newpage

\appendix

\begin{center}
    \Large\bfseries Contents of the Appendices
\end{center}

\startcontents[sections]
\printcontents[sections]{l}{1}{\setcounter{tocdepth}{2}}
%\newpage

\numberwithin{equation}{section}

%Redefien Theorem Style in Main body
\newtheorem{thmApp}{Theorem}[section]
\newtheorem{lemApp}[thmApp]{Lemma}
\newtheorem{propApp}[thmApp]{Proposition}
\newtheorem{corApp}[thmApp]{Corollary}

\newtheorem{defnApp}{Definition}[section]
\newtheorem{examApp}{Example}[section]

\newtheorem{rmkApp}{Remark}[section]

\let\thm\thmApp \let\endthm\endthmApp
\let\lem\lemApp \let\endlem\endlemApp
\let\prop\propApp \let\endprop\endpropApp
\let\cor\corApp \let\endcor\endcorApp
\let\exam\examApp \let\endexam\endexamApp
\let\rmk\rmkApp \let\endrmk\endrmkApp
\let\defn\defnApp \let\enddefn\enddefnApp

\input{appendix_prel}
\input{appendix_geo}
\input{appendix_p_conv_gem}
\input{appendix_further_gem}
\input{appendix_rld}

\input{appendix_technique}

\vskip 0.2in
\bibliography{references}

\end{document}

%% file: main_text.tex
\section{Introduction}

Recently, diffusion models have shown strong performance in learning data distributions on Euclidean spaces \citep{ho2020denoising,song2021scorebased,dhariwal2021diffusion}. One key factor behind their success is the ability to capture the low-dimensional structure present in practical data sets \citep{chen2023score,oko2023diffusion,loaiza-ganem2024deep,wang2025}. It is widely believed that high-dimensional real-world data concentrate near a low-dimensional manifold; this phenomenon is known as the manifold hypothesis \citep{bengio2013representation}. Consequently, there has been growing interest in understanding the convergence rates of Euclidean diffusion models under the manifold hypothesis \citep{debortoli2022convergence,li2024adapting,farghly2025diffusion}. In parallel, several recent works construct diffusion models directly on the data manifold, often referred to as Riemannian diffusion models \citep{de2022riemannian,huang2022riemannian,jo2024generative,jo2025continuous}, by using techniques from manifold-valued stochastic differential equations (SDEs) \citep{hsu2002stochastic}.

As in the Euclidean case, Riemannian diffusion models consist of a forward and a backward process \citep{de2022riemannian,huang2022riemannian}. Both processes are instances of Riemannian Langevin dynamics (RLD), a canonical Riemannian diffusion process for sampling a target distribution $\mu_\phi$ with density proportional to $e^{-\phi}$ with respect to the Riemannian volume measure \citep{wang2020fast,cheng2022efficient,karthik2025sampling}. Specifically, given a Riemannian manifold $\M$ and a smooth function $\phi$ on $\M$, RLD is defined by the $\M$-SDE \citep{hsu2002stochastic}
\begin{equation*}
    \dx{\bm{X}_t} = V(\bm{X}_t)~\dx{t} + \dx{\bm{B}^{\M}_t},\quad t \in [0,T],
\end{equation*}
where $V = - \frac{1}{2} \nabla \phi$ (the gradient on $\M$) is the drift vector field and $\bm{B}^\M$ is a standard Brownian motion on $\M$. When $\M$ and $\phi$ satisfy the Bakry--\'Emery curvature condition \citep{villani2008optimal}, it is well-known that the distribution $\mu_t$ of $\bm{X}_t$ converges to $\mu_\phi$ exponentially fast in $t$ in Wasserstein distance. To discretize RLD on $[0,T]$, one typically uses the Geometric Euler--Maruyama (GEM) scheme \citep{jorgensen1975central,piggott2016geometric,de2022riemannian}: letting $h = T/N$ for some $N \in \N$ and $t_k = kh$, define
\begin{equation*}
    \bm{X}^h_{k+1} = \exp_{\bm{X}^h_k}\bc{hV(\bm{X}^h_k) + \sqrt{h}\bm{\xi}_{\bm{X}^h_k}},\quad k = 0,\ldots,N-1,
\end{equation*}
where $\exp_x$ is the exponential map and $\bm{\xi}_x$ is a standard Gaussian on the tangent space.

The weak convergence, i.e., convergence in distribution, of the GEM scheme has been well studied \citep{jorgensen1975central,kuwada2012convergence}. Similar to the Euclidean Euler--Maruyama (EM) scheme \citep{kloeden1992numerical}, GEM achieves weak convergence of order $1$ \citep{karthik2025sampling}. In contrast, the $p$-strong (pathwise) convergence,
\begin{equation*}
    \E\bj{\max_{0 \leq k \leq N}d_{\M}(\bm{X}^h_k,\bm{X}_{t_k})^p}^{1/p} \lesssim h^\gamma,
\end{equation*}
where $d_\M$ is the intrinsic Riemannian distance on $\M$, is less understood. To the best of our knowledge, establishing the same order $\gamma = 1/2$ for GEM as that for the Euclidean EM \citep{kloeden1992numerical} remains an open gap in general settings. Existing results focused either on special manifolds such as spheres \citep{solo2021convergence}, special orthogonal group $\op{SO}_n$ \citep{piggott2016geometric,wang2025tangent}, or more broadly on Lie groups \citep{marjanovic2018numerical,muniz2022higher,muniz2023strong,solo2024stratonovich}. Other works studied weaker notions of approximation, such as bounds in Wasserstein distance \citep{cheng2022efficient,de2022riemannian}.

\parhead{Our results.} This work establishes the strong convergence of the GEM on an embedded Riemannian submanifold $\M \subset \R^n$. Under the assumptions that $\M$ has bounded extrinsic curvatures and is globally well-embedded in $\R^n$, we show that the GEM achieves the $p$-strong pathwise convergence with order $\gamma = 1/2$, matching the classical Euclidean EM rate. We first state our main results informally as follows.

\begin{thm}[$p$-strong convergence of GEM, see Theorem \ref{thm:intr_p_strong_conv_GEM}]\label{thm:intr_p_strong_conv_GEM_inform}
    Let $1 \leq p < \infty$. Assume that $\M \subset \R^n$ is a Riemannian submanifold with bounded extrinsic curvatures and is globally well-embedded, and that $V$ satisfies smoothness conditions. For $\bm{X}^h_k$ constructed by the GEM, we have
    \begin{equation*}
         \E\bj{\max_{0 \leq k \leq N}d_{\M}(\bm{X}^h_k,\bm{X}_{t_k})^p} \lesssim h^{p/2}.
    \end{equation*}
\end{thm}

Building on the $p$-strong convergence of the GEM together with the Bakry--\'Emery curvature condition, we obtain a $p$-Wasserstein bound for RLD under the GEM discretization.

\begin{thm}[$p$-Wasserstein convergence of RLD, see Theorem \ref{thm:intr_convergence_RLD}]\label{thm:intr_convergence_RLD_inform}
    Let $1 \leq p < \infty$. Assume that $\M$ and $ V = -\frac{1}{2}\nabla \phi$ satisfy the assumptions in Theorem \ref{thm:intr_p_strong_conv_GEM_inform}, and that $\phi$ satisfies the Bakry--\'Emery curvature condition. For $\bm{X}^h_k$ constructed via the GEM, let $\hat{\mu}_k$ be the distribution of $\bm{X}^h_k$. Then we have
    \begin{equation*}
        \mathcal{W}_p(\mu_\phi,\hat{\mu}_N) \lesssim e^{-T} + h^{1/2}.
    \end{equation*}
\end{thm}

For any compact Riemannian manifold $\M$ and any vector field $V$ on it, we show that the assumptions in Theorem \ref{thm:intr_p_strong_conv_GEM_inform} are satisfied, independent of how $\M$ is embedded into $\R^n$ via Nash's embedding theorem. Consequently, the GEM on compact Riemannian manifolds admits the $p$-strong convergence of order $1/2$ for any $p \geq 1$ (Corollary \ref{cor:p_strong_conv_GEM_cpt}). For the non-compact case, we also provide two classical examples, graphs and level sets, that satisfy our geometric boundedness requirements.

\parhead{Technical overview.} Since $\M \subset \R^n$, an $\M$-Brownian motion $\bm{B}^\M$ can be realized by projecting a standard Brownian motion on $\R^n$, and the $\M$-SDE can be written in an extrinsic form as an SDE in $\R^n$ \citep{hsu2002stochastic}. This makes it possible to use the Euclidean EM scheme. However, two problems arise: (\rnum{1}) the coefficients in the extrinsic $\R^n$-SDE are only defined on $\M$, so its Euclidean EM discretization $\bm{Y}^h_k$ is not directly well-defined; (\rnum{2}) even after defining $\bm{Y}^h_k$, one needs to control strong pathwise discrepancy between the extrinsic EM $\bm{Y}^h_k$ and the intrinsic GEM $\bm{X}^h_k$.

To resolve (\rnum{1}), we study the extrinsic geometry of $\M$ and use the geometric boundedness of $\M$, together with the regularity assumptions on the drift vector field $V$, to construct an $\R^n$-SDE that is well-posed on $\R^n$ and agrees with the original $\M$-SDE on $\M$. This makes the Euclidean EM scheme $\bm{Y}^h_k$ well-defined and allows us to apply classical strong convergence results in Euclidean space. To resolve (\rnum{2}), we use a Taylor expansion of the exponential map $\exp_x$ and bound the corresponding remainder terms uniformly via the geometric boundedness of $\M$. This reduces the comparison between $\bm{Y}^h_k$ and $\bm{X}^h_k$ to estimates in Euclidean space, and the error can be controlled by the geometric properties of the equivalent $\R^n$-SDE. Combining the results of (\rnum{1}) and (\rnum{2}), we can prove Theorem \ref{thm:intr_p_strong_conv_GEM_inform}.

\parhead{Contributions.} In summary, our contributions are as follows.
\begin{enumerate}[label=(\roman*)]
    \item We establish the $p$-strong convergence of the GEM with order $1/2$ for Riemannian submanifolds embedded in $\R^n$ under geometric boundedness and drift regularity. Moreover, combined with the Bakry--\'Emery curvature condition, we provide a $\mathcal{W}_p$ guarantee for RLD under the GEM discretization.
    \item On the technical side, we study the extrinsic geometry of embedded Riemannian submanifolds in $\R^n$ and develop an extrinsic extension-and-comparison framework for analyzing intrinsic manifold-valued SDEs.
\end{enumerate}

The remainder of this paper is organized as follows. Section \ref{sec:related_works} reviews related works, and Section \ref{sec:preliminaries} provides the necessary preliminaries. Section \ref{sec:p_strong_convergence_of_gem} presents the $p$-strong convergence of the GEM and related results, and Section \ref{sec:p_wasserstein_convergence_of_rld} applies these results to obtain a $p$-Wasserstein bound for RLD under the GEM discretization. Section \ref{sec:conclusion} concludes the paper and discusses limitations. Notation is summarized in Appendix \ref{appen:notation}.

\section{Related Works}\label{sec:related_works}

\parhead{Numerical solutions of manifold SDEs.} Most existing results on strong convergence of numerical schemes for manifold-valued SDEs, such as GEM, focused on special classes of manifolds. For example, \citet{piggott2016geometric} studied GEM on the special orthogonal group $\op{SO}_n$ and the special Euclidean group $\op{SE}_n$, and obtained a $p=2$-strong convergence order of $\gamma = (1-\varepsilon)/2$ for arbitrarily small $\varepsilon>0$. \citet{solo2021convergence} proved the same order for the sphere using tangent space parametrization. More generally, building on the structure of Lie groups \citep{marjanovic2015simple,marjanovic2018numerical}, \citet{muniz2022higher} and \citet{solo2024stratonovich} established analogous results on Lie groups, and \citet{muniz2023strong} considered a special case where the manifold is equipped with a Lie group action.

For a weaker notion of approximation, \citet{cheng2022efficient} obtained $p$-Wasserstein bounds for the GEM with $p=1,2$ at rate $\tilde{O}(h^{1/2})$, where $\tilde{O}$ hides polynomial dependence on $\log h$, by bounding $\E[d_\M(\bm{X}^h_N,\bm{X}_T)^p]$. \citet{de2022riemannian} extended this result to time-inhomogeneous drift vector fields, i.e., time-dependent $V(t,x)$.

Compared with the works above, we consider a more general class of manifolds, embedded Riemannian submanifolds in $\R^n$, since we do not require additional structures such as a Lie group structure. Moreover, we establish the $p$-strong convergence of the GEM at the rate $\mathcal{O}(h^{1/2})$ (Theorem \ref{thm:intr_p_strong_conv_GEM_inform}), consistent with the Euclidean EM case.

\parhead{Embedded manifolds.} \citet{armstrong2022curved} and \citet{laurent2022order} also considered $\M \subset \R^n$ as an embedded submanifold, but they pre-assumed that the SDE coefficients are well-defined on $\R^n$ with the required regularities. We study SDEs defined intrinsically on $\M$, and use an extrinsic extension only as a bridge for proving error bounds. This viewpoint also explains why our results apply to arbitrary compact Riemannian manifolds (Corollary \ref{cor:p_strong_conv_GEM_cpt}), without requiring $\M$ to be pre-assumed as a given submanifold of some $\R^n$.

Moreover, instead of using GEM, \citet{armstrong2022curved} used a Euclidean ordinary differential equation (ODE) to approximate the SDE, so the resulting discrete process does not lie on $\M$. \citet{laurent2022order} studied the level-set case, i.e., $\M = F^{-1}(0)$ for some $F \colon \R^n \sto \R^k$, and considered the ergodic convergence of the Euclidean EM scheme. Our work focuses on the intrinsic GEM scheme, for which the discrete process $\bm{X}^h_k$ stays on $\M$, and establish the analogous strong convergence result to the Euclidean EM scheme.

The level-set case is also the main object in \citet{wang2025geometry}, which used GEM to approximate SDEs on $\M$. They considered $2$-strong convergence only measured by the Euclidean norm, and obtained the order $\gamma = (1-\varepsilon)/2$ for arbitrarily small $\varepsilon>0$. Compared with their work, our results apply to more general embedded manifolds and also provide intrinsic error bounds in $d_\M$. In particular, we include the level-set case as a special case, and our assumptions on it are consistent with the setting in \citet{wang2025geometry}.

\parhead{GEM for RLD.} For compact manifolds, \citet{wang2020fast} and \citet{li2023riemannian} provided Kullback--Leibler (KL) divergence bounds between the distribution induced by the GEM and the target distribution $\mu_\phi$, which also imply $2$--Wasserstein bounds via the Talagrand inequality. \citet{karthik2025sampling} studied the weak and ergodic convergence of the GEM discretization for RLD on compact manifolds. For the non-compact case, \citet{gatmiry2022convergence} studied Hessian manifolds and showed KL convergence for the GEM discretization of RLD. In contrast, our work focuses on the strong convergence, and obtains a $p$-Wasserstein bound by combining the discretization error from the $p$-strong convergence of the GEM with the mixing error of RLD to $\mu_\phi$ under the Bakry--\'Emery curvature condition.

\section{Preliminaries}\label{sec:preliminaries}

In this section we briefly review some basic concepts in Riemannian manifold and stochastic analysis on manifold.

\subsection{Riemannian Submanifold}\label{sub:pre_riemannian_manifold}

\parhead{Submanifold.} A subset $\mathcal{M} \subset \R^n$ is called an $m$-dimensional (embedded) submanifold of $\R^n$ (without boundary) if there are a family of open sets $\bb{\mathcal{V}_\alpha}_{\alpha \in \Lambda}$ in $\R^n$, a family of open sets $\bb{\mathcal{U}_\alpha}_{\alpha \in \Lambda}$ in $\R^m$, and a family of smooth ($C^\infty$) maps $\bb{\phi_\alpha}_{\alpha \in \Lambda}$ such that
\begin{equation*}
    \mathcal{M} \subset \bigcup_{\alpha \in \Lambda} \mathcal{V}_\alpha,\text{ and }\phi_\alpha \colon \mathcal{U}_\alpha \rightarrow \mathcal{V}_\alpha \cap \mathcal{M}
\end{equation*}
is a diffeomorphism, i.e., $\phi_\alpha^{-1} \colon \mathcal{V}_\alpha \cap \mathcal{M} \sto \mathcal{U}_\alpha$ is also smooth. For any $x \in \M$, the tangent space at $x$ is defined as
\begin{equation*}
    T_x\mathcal{M} \defeq \bb{\gamma^\prime(0) \colon \exists~ \varepsilon>0,~ \gamma \colon (-\varepsilon,\varepsilon) \sto \mathcal{M} \text{ smooth, } \gamma(0) = x} \subset \R^n,
\end{equation*}
and the tangent bundle $T\M \defeq \bigcup_{x \in \M}T_x\M$. Furthermore, for any $x \in \M$, we call $N_x\M \defeq \bc{T_x\M}^\perp$, i.e., the orthogonal complement in $\R^n$, the normal space, and define the normal bundle $N\M \defeq \bigcup_{x \in \M}N_x\M$. A smooth map $V \colon \M \sto T\M$ with $V(x) \in T_x\M$ is called a (smooth) vector field on $\M$, written $V \in \Gamma(T\M)$; similarly, $\eta \in \Gamma(N\M)$ denotes a normal vector field. We refer to \citet{lee2012smooth} for further details on smooth manifolds.

\parhead{Riemannian submanifold.} Let $\M \subset \R^n$ be a submanifold. For any $x \in \M$, since $T_x\M \subset \R^n$, we equip $T_x\M$ with the canonical Euclidean inner product. Then $\M$ becomes a Riemannian submanifold embedded in $\R^n$. We further assume that $\M \subset \R^n$ is connected and closed. Thus, \emph{throughout this paper, $\M \subset \R^n$ is called a Riemannian submanifold if it is an embedded Riemannian submanifold that is connected and closed.}

Let $d_{\M}$ be the induced Riemannian distance defined by
\begin{equation*}
    d_\M(x,y) \defeq \inf\bb{L(\gamma) \colon \gamma \colon [0,1] \sto \M \text{ piecewise smooth, } \gamma(0) = x,~\gamma(1)=y},
\end{equation*}
where the length of $\gamma$ is $L(\gamma) \defeq \int_0^1 \norm*{\gamma^\prime(t)}~\dx{t}$. Equivalently, $d_\M(x,y)$ is the smallest length among all curves in $\M$ connecting $x$ and $y$. Note that $(\M,d_\M)$ is a complete metric space.

\parhead{Derivative.} Let $D$ denote the directional derivative on $\R^n$, i.e., $Dg(y)[v] = D_v g(y) \defeq \inn{v,\nabla^{\R} g(y)}$, where $\nabla^{\R} g$ denotes the gradient of $g$. We can also consider the \emph{gradient along $\M$}. For any $f \in C^\infty(\M)$, i.e., a smooth function defined on $\M$, its gradient along $\M$ is given by $\nabla f(x) = (\nabla^{\R}\tilde{f}(x))^\top \in T_x\M$, the orthogonal projection of $\nabla^{\R}\tilde{f}(x)$ onto $T_x\M$, where $\tilde{f}$ is an arbitrary smooth extension of $f$ to $\R^n$. We further consider derivatives of vector fields, given by the \emph{Levi-Civita connection} $\nabla \colon \Gamma(T\M) \times \Gamma(T\M) \sto \Gamma(T\M)$. It is similarly defined by $\nabla_XY(x) \defeq (D_{\tilde{X}}\tilde{Y}(x))^\top$ for arbitrary extensions $\tilde{X},\tilde{Y}$. Moreover, the \emph{second fundamental form} $\mathrm{II}$ is defined by $\mathrm{II}_x(X,Y) \defeq (D_{\tilde{X}}\tilde{Y}(x))^\perp \in N_x\M$, the orthogonal projection onto $N_x\M$.

A smooth curve $\gamma \colon [0,1] \sto \M$ is called a \emph{geodesic} if $\nabla_{\gamma^\prime(t)}\gamma^\prime(t) \equiv 0$. Note that for any $x \in \M$ and any $v \in T_x\M$, there exists a unique geodesic $\gamma_v \colon [0,1] \sto \M$ with $\gamma_v(0)=x$ and $\gamma_v^\prime(0) = v$. Thus, the \emph{exponential map} at $x$ is $\exp_x \colon T_x\M \sto \M$ defined by $\exp_x(v) \defeq \gamma_v(1)$. We refer to \citet{lee2018introduction} for further details on Riemannian manifolds.

\parhead{Tubular Neighborhood.} By the Tubular Neighborhood Theorem \citep[Theorem 5.25]{lee2018introduction}, there exists a smooth function $\varepsilon \colon \M \sto (0,\infty)$ such that the \emph{tubular map} $T(x,\xi)=x+\xi \colon N\M \sto \R^n$ restricts to a diffeomorphism from
\begin{equation*}
    \mathcal{U}_0 \defeq \bb{(x,\xi)\in N\M \colon \norm{\xi}<\varepsilon(x)}
\end{equation*}
onto $\mathcal{V}_0 \defeq T(\mathcal{U}_0)$, which is called a \emph{tubular neighborhood} of $\M$ in $\R^n$. In particular, each $y \in \mathcal{V}_0$ admits a unique decomposition $y=x+\xi$ with $x\in\M$ and $\xi\in N_x\M$, and we define the \emph{canonical projection} $\pi \colon \mathcal{V}_0 \sto \M$ by $\pi(y)=x$. Furthermore, if $\varepsilon_0 \defeq \inf_{x\in\M}\varepsilon(x)>0$, then (after shrinking $\mathcal{U}_0$) one may take $\varepsilon(x) \equiv \varepsilon_0$, and
\begin{equation*}
    \mathcal{V}_0 = \bb{y\in\R^n \colon d(y,\M)<\varepsilon_0},\quad d(y,\M)\defeq \inf\bb{\norm{y-x}\colon x\in\M},
\end{equation*}
which is called a \emph{uniform tubular neighborhood} of $\M$ with uniform radius $\varepsilon_0$. The existence of a uniform tubular neighborhood is closely related to a global metric property of $\M$; see Appendix \ref{appen:reach} for further details.

\subsection{SDE on Riemannian Manifold}

\parhead{It\^o SDE vs. Stratonovich SDE.} Let $\bm{W}= \bc{\bm{W}^1,\ldots,\bm{W}^n}$ be a standard Brownian motion on $\R^n$. A Stratonovich SDE on $\R^n$ driven by $\bm{W}$ is formulated as
\begin{equation*}
    \dx{\bm{X}_t} = b(\bm{X}_t)~\dx{t}+\sum_{\alpha = 1}^n\sigma_\alpha(\bm{X}_t) \circ \dx{\bm{W}^\alpha_t},
\end{equation*}
where $b,\sigma_\alpha \colon \R^n \sto \R^n$, and $\circ$ denotes the Stratonovich integral \citep{karatzas1991brownian}. It is equivalent to the It\^o formulation
\begin{equation*}
    \dx{\bm{X}_t} = \bc{b(\bm{X}_t)+\frac{1}{2}\sum_{\alpha = 1}^n D_{\sigma_\alpha}\sigma_\alpha(\bm{X}_t)}~\dx{t}+\sum_{\alpha = 1}^n\sigma_\alpha(\bm{X}_t)~\dx{\bm{W}^\alpha_t}.
\end{equation*}
For more details about SDEs, we refer to \citet{le2016brownian} and \citep{karatzas1991brownian}.

\parhead{SDE on Manifold.} Let $\M \subset \R^n$ be a Riemannian submanifold and let $V_\alpha \in \Gamma(T\M)$ for $\alpha = 0,1,\ldots,n$. Let $\bm{W}$ be a standard Brownian motion on $\R^n$. An $\M$-valued SDE is symbolically written as
\begin{equation*}
    \dx{\bm{X}_t} = V_0(\bm{X}_t)~\dx{t} + \sum_{\alpha = 1}^nV_\alpha(\bm{X}_t) \circ \dx{\bm{W}^\alpha_t},
\end{equation*}
which means that $\bm{X}_t$ is a solution if, for any $f \in C^\infty(\M)$,
\begin{equation*}
    \dx{f(\bm{X}_t)} = \inn{\nabla f(\bm{X}_t), V_0(\bm{X}_t)}~\dx{t} + \sum_{\alpha = 1}^n\inn{\nabla f(\bm{X}_t),V_\alpha(\bm{X}_t)} \circ \dx{\bm{W}^\alpha_t}.
\end{equation*}
So an $\M$-valued SDE can be viewed as a Stratonovich SDE on $\R^n$. Note that, given an initial condition, the solution exists and is unique. \citet{hsu2002stochastic} provides further details on manifold-valued SDEs.

\parhead{Brownian Motion on $\M$.} Let $\M \subset \R^n$ be a Riemannian submanifold equipped with a smooth orthogonal projection $P(x)$, i.e., for each $x \in \M$, $P(x) \colon \R^n \sto \R^n$ is the orthogonal projection onto $T_x\M$, and the map $P$ is smooth on $\M$. Let $\bb{e_\alpha}_{\alpha = 1}^n$ be the canonical basis of $\R^n$, and define $P_\alpha \defeq P e_\alpha \in \Gamma(T\M)$. Then we consider the $\M$-valued SDE,
\begin{equation}\label{eq:def_brown_mfd}
    \dx{\bm{B}_t} = \sum_{\alpha = 1}^nP_\alpha(\bm{B}_t) \circ \dx{\bm{W}^\alpha_t} = P(\bm{B}_t) \circ \dx{\bm{W}_t},
\end{equation}
where $\bm{W}$ is a standard Brownian motion on $\R^n$. Let $\bm{B}^\M$ be the unique solution of SDE (\ref{eq:def_brown_mfd}) with $\bm{B}^\M_0 = 0$. Then $\bm{B}^\M$ is called a standard Brownian motion on $\M$ \citep{hsu2002stochastic}.

\section{\texorpdfstring{$p$-Strong Convergence of GEM}{p-Strong Convergence of GEM}}\label{sec:p_strong_convergence_of_gem}

Let $\M \subset \R^n$ be a Riemannian submanifold. Consider an $\M$-valued SDE
\begin{equation}\label{eq:mfd_sde_Bm}
    \dx{\bm{X}_t} = V(\bm{X}_t)~\dx{t} + \dx{\bm{B}^\M_t}, \quad t \in [0,T],
\end{equation}
where $V \in \Gamma(T\M)$ and $\bm{B}^\M$ is a standard Brownian motion on $\M$. Let $h = T/N$ for $N \in \N$, and define $t_k = kh$ for $k=0,1,\ldots,N$. The GEM algorithm constructs a discrete-time process $\bb{\bm{X}^h_k}_{k=0}^N$ by setting $\bm{X}^h_0 = \bm{X}_0$ and iterating
\begin{equation}\label{eq:gem}
    \bm{X}^h_{k+1} = \exp_{\bm{X}^h_k}\bc{hV(\bm{X}^h_k) + \sqrt{h}\bm{\xi}_{\bm{X}^h_k}},\quad k = 0,\ldots,N-1,
\end{equation}
where $\exp_x \colon T_x\M \sto \M$ is the exponential map, and $\bm{\xi}_x$ is a standard Gaussian on $T_x\M$ with respect to the metric $g(x)$. The main goal of this section is to establish the $p$-strong convergence of $\bm{X}^h_k$ to $\bm{X}_{t_k}$.

\subsection{Main Results}

We first introduce the following two assumptions, which ensure the required geometric boundedness of $\M$.

\begin{assum}\label{assum:geometric_iota}
    Let $\iota \colon \M \hookrightarrow \R^n$ be the canonical embedding map. Assume that the covariant derivatives of $d\iota$ satisfy
    \begin{equation*}
        \sup_{x \in \M}\norm*{\nabla d \iota(x)}_{\op{op}}<\infty,\quad\sup_{x \in \M} \norm*{\nabla^2 d \iota(x)}_{\op{op}}<\infty.
    \end{equation*}
\end{assum}

\begin{assum}\label{assum:uniform_tubular}
    Assume that $\M$ admits a uniform tubular neighborhood $\mathcal{V}$ in $\R^n$.
\end{assum}

\begin{rmk}\label{rmk:rmk_of_geometric_assum}
    We make the following remarks concerning these two assumptions.
    \begin{enumerate}[label=(\roman*)]
        \item $\nabla d\iota$ is essentially the extrinsic curvature of $\M$. Rigorous definitions of $\nabla d\iota$ and $\nabla^{2} d\iota$, as well as the geometric boundedness of $\M$ under Assumption \ref{assum:geometric_iota}, are provided in Appendix \ref{appen:geometric_properties}. Since Assumption \ref{assum:geometric_iota} is formulated in terms of intrinsic derivatives and may be difficult to verify directly, we also give in Appendix \ref{appen:extrinsic_criterion_for_assumption_ref_assum_geometric_iota} a sufficient extrinsic criterion implying Assumption \ref{assum:geometric_iota}.

        \item In Appendix \ref{appen:examples_of_manifolds}, we provide three classical classes of manifolds satisfying Assumption \ref{assum:geometric_iota} and Assumption \ref{assum:uniform_tubular}: compact manifolds, graphs, and level sets. The level-set example coincides with the settings studied in \citet{wang2025geometry}.

        \item Assumption \ref{assum:geometric_iota} enforces local geometric regularity of $\M$, whereas Assumption \ref{assum:uniform_tubular} is a global requirement. In particular, Assumption \ref{assum:uniform_tubular} yields $\sup\norm{\nabla d\iota}_{\op{op}}<\infty$, but not, in general, $\sup\norm{\nabla^2 d\iota}_{\op{op}}<\infty$. Conversely, Assumption \ref{assum:geometric_iota} does not guarantee the existence of a uniform tubular neighborhood. The relationship between these two assumptions is discussed in Appendix \ref{appen:relationship_of_assumption_iota_assum_tubular}. However, if the canonical embedding $\iota$ is globally bi-Lipschitz continuous (Appendix \ref{appen:discuss_embedding_map}), then $\sup\norm{\nabla d\iota}_{\op{op}}<\infty$ implies Assumption \ref{assum:uniform_tubular}; see Appendix \ref{appen:assumption_iota_imply_assum_tubular}.
    \end{enumerate}
\end{rmk}

We further assume standard regularity conditions on the drift vector field $V$.

\begin{assum}\label{assum:vector_field}
    Let $V \in \Gamma(T\M)$. Assume that $V \in C_b^1$, i.e.,
    \begin{equation*}
        \sup_{x \in \M} \norm*{V(x)} < \infty,\quad\sup_{x \in \M} \norm*{\nabla V(x)}_{\op{op}} < \infty.
    \end{equation*}
\end{assum}
\begin{rmk}\label{rmk:rmk_of_vec_assum}
    We provide two remarks on this assumption.
    \begin{enumerate}[label=(\roman*)]
        \item The assumption concerns only intrinsic properties of $V \in \Gamma(T\M)$, and therefore is independent of how $\M$ is embedded in $\R^n$. However, the uniform boundedness of $V$ itself can be restrictive in practice. When working with $\M \subset \R^n$, this assumption can be relaxed to
        \begin{equation*}
            \sup_{x_1 \neq x_2 \in \M} \frac{\norm*{V(x_1) - V(x_2)}}{\norm{x_1-x_2}} < \infty,\quad \sup_{x \in \M} \frac{\norm*{V(x)}}{1 + \norm{x}^{1/3}} < \infty;
        \end{equation*}
        see Appendix \ref{appen:relaxing_assumption_ref_assum_vector_field} for details.

        \item Assumption \ref{assum:vector_field} together with Assumption \ref{assum:geometric_iota} guarantees existence and uniqueness \citep{hsu2002stochastic}, as well as non-explosion \citep{li1994strong,stroock2000introduction}, of the solution to SDE (\ref{eq:mfd_sde_Bm}). The same conclusion can be obtained from our extrinsic extension; see Corollary \ref{cor:exist_unique_nonexplode_sol_sde_M}.
    \end{enumerate}
\end{rmk}

Then we can formally state our main result.

\begin{thm}\label{thm:ext_p_strong_conv_GEM}
    Let $1 \leq p < \infty$. Let $\M \subset \R^n$ be a Riemannian submanifold and $V \in \Gamma(T\M)$. Assume that $\M$ satisfies Assumption \ref{assum:geometric_iota} and Assumption \ref{assum:uniform_tubular}, and that $V$ satisfies Assumption \ref{assum:vector_field}. Let the SDE (\ref{eq:mfd_sde_Bm}) be defined on $\M$ with solution $\bm{X}_t$, where $\E\bj{\norm{\bm{X}_0}^p} < \infty$. For $\{\bm{X}^h_k\}_{k=0}^N$ constructed by GEM (\ref{eq:gem}) with $\bm{X}^h_0 = \bm{X}_0$, there exists $C_p(T)$ such that
    \begin{equation*}
        \E\bj{\max_{0 \leq k \leq N} \norm*{\bm{X}^h_k - \bm{X}_{t_k}}^p} \leq C_p(T)h^{p/2},
    \end{equation*}
    where $C_p(T) = \mathcal{O}(\exp(T^p))$.
\end{thm}
\begin{rmk}
    Theorem \ref{thm:ext_p_strong_conv_GEM} extends directly to a time-inhomogeneous case, i.e., considering a time-dependent vector field $V(t,x)$ such that Assumption \ref{assum:vector_field} holds uniformly for $t \in [0,T]$; see Appendix \ref{appen:time_inhomogeneous_case} for details.
\end{rmk}

Theorem \ref{thm:ext_p_strong_conv_GEM} establishes the strong convergence measured by the extrinsic distance in $\R^n$, rather than by the intrinsic Riemannian distance on $\M$. To measure errors in the intrinsic distance $d_\M$, we additionally assume that the canonical embedding $\iota \colon \M \hookrightarrow \R^n$ is globally bi-Lipschitz continuous (Appendix \ref{appen:discuss_embedding_map}). Under this condition, Assumption \ref{assum:uniform_tubular} is unnecessary. Therefore, as an application of Theorem \ref{thm:ext_p_strong_conv_GEM}, we obtain Theorem \ref{thm:intr_p_strong_conv_GEM}; its proof is provided in Appendix \ref{appen:proof_of_theorem_ref_thm_intr_p_strong_conv_gem}.

\begin{thm}\label{thm:intr_p_strong_conv_GEM}
    Let $1 \leq p < \infty$. Let $\M \subset \R^n$ be a Riemannian submanifold such that the canonical embedding $\iota \colon \M \hookrightarrow \R^n$ is globally bi-Lipschitz continuous. Assume that $\M$ satisfies Assumption \ref{assum:geometric_iota} and that $V \in \Gamma(T\M)$ satisfies Assumption \ref{assum:vector_field}. Let $\bm{X}_t$ be a solution of the SDE (\ref{eq:mfd_sde_Bm}) on $\M$ with $\E\bj{\norm{\bm{X}_0}^p} < \infty$. For $\{\bm{X}^h_k\}_{k=0}^N$ constructed by GEM (\ref{eq:gem}) with $\bm{X}^h_0 = \bm{X}_0$, there exists $C_p(T) = \mathcal{O}(\exp(T^p))$ such that
    \begin{equation*}
        \E\bj{\max_{0 \leq k \leq N} d_\M\bc{\bm{X}^h_k, \bm{X}_{t_k}}^p} \leq C_p(T)h^{p/2},
    \end{equation*}
    where $d_\M$ is the Riemannian distance on $\M$.
\end{thm}

For any Riemannian manifold $\M$, it admits an isometric embedding into some Euclidean space $\R^n$ by Nash's embedding theorem. If $\M$ is further compact, then all assumptions in Theorem \ref{thm:intr_p_strong_conv_GEM} are satisfied, independently of the choice of embedding into $\R^n$. Therefore, for the GEM on an arbitrary compact Riemannian manifold, we obtain the following Corollary \ref{cor:p_strong_conv_GEM_cpt}; see Appendix \ref{appen:proof_of_corollary_ref_cor_p_strong_conv_gem_cpt} for the proof.

\begin{cor}\label{cor:p_strong_conv_GEM_cpt}
    Let $1 \leq p < \infty$. Let $\M$ be a compact Riemannian manifold. For a solution $\bm{X}_t$ of SDE (\ref{eq:mfd_sde_Bm}) and $\{\bm{X}^h_k\}_{k=0}^N$ constructed by GEM (\ref{eq:gem}) with $\bm{X}^h_0 = \bm{X}_0$, there exists $C_p(T) = \mathcal{O}(\exp(T^p))$ such that
    \begin{equation*}
        \E\bj{\max_{0 \leq k \leq N} d_\M\bc{\bm{X}^h_k,\bm{X}_{t_k}}^p} \leq C_p(T)h^{p/2}.
    \end{equation*}
\end{cor}

\subsection{Ideas for Proving Theorem \ref{thm:ext_p_strong_conv_GEM}}

Because $\M \subset \R^n$, there exists a smooth orthogonal projection $P(x)$ onto $T_x\M$; see Appendix \ref{appen:existence_of_orthogonal_projection}. Then $\bm{B}^\M$ can be realized as the projection of a standard $\R^n$-Brownian motion $\bm{W}$ \citep{hsu2002stochastic,emery2012stochastic,thalmaierStochGeom2023}, and (\ref{eq:mfd_sde_Bm}) is equivalent to the $\M$-valued SDE
\begin{equation*}
    \dx{\bm{X}_t} = V(\bm{X}_t)~\dx{t} + P(\bm{X}_t) \circ \dx{\bm{W}_t},\quad t \in [0,T],
\end{equation*}
which can also be viewed as a Stratonovich SDE on $\R^n$ \citep{hsu2002stochastic}. So, by the It\^o's formula \citep{karatzas1991brownian}, it can be written in It\^o form as
\begin{equation}\label{eq:mfd_sde_eucl_w}
    \dx{\bm{X}_t} = U(\bm{X}_t)~\dx{t} + P(\bm{X}_t)~\dx{\bm{W}_t},\quad t \in [0,T],
\end{equation}
where
\begin{equation}\label{eq:def_U_A}
    U(x) = V(x) + A(x),\quad A^j(x) \defeq \frac{1}{2} \sum_{\alpha = 1}^n P_\alpha\bc{P_{j\alpha}}(x),\quad j=1,\ldots,n.
\end{equation}
Here, with respect to the standard basis $\bb{e_\alpha}_{\alpha=1}^n$ of $\R^n$, $P(x) = \bc{P_{j\alpha}(x)}$ is the matrix representation of $P(x)$, and $P_\alpha(x) \defeq P(x)e_\alpha \in T_x\M \subset \R^n$ is viewed as a vector field. For $f \in C^1(\mathcal{V})$ with an open $\mathcal{V} \supset \M$,
\begin{equation*}
    P_\alpha(f)(x) \defeq \sum_{j=1}^n P_{j\alpha}(x) \frac{\partial f}{\partial x^j}(x),\quad x \in \M.
\end{equation*}
Since $P$ is defined on an open neighborhood of $\M$ (Lemma \ref{lem:projec_of_Dpi}), $P_\alpha\bc{P_{j\alpha}}(x)$ is well-defined for all $x \in \M$.

\parhead{Main ideas.} Note that SDE (\ref{eq:mfd_sde_eucl_w}) is formulated in $\R^n$, so it can be viewed as a standard Euclidean SDE. Our main idea is to first apply the Euclidean EM algorithm to approximate (\ref{eq:mfd_sde_eucl_w}) and then compare the resulting Euclidean discrete process with the intrinsic process $\bm{X}^h_k$. However, we cannot directly apply the Euclidean EM algorithm to (\ref{eq:mfd_sde_eucl_w}) because the coefficients $U(x)$ and $P(x)$ are only well-defined on $\M$ or on a small neighborhood of $\M$. More specifically, our approach consists of two steps:

\begin{enumerate}[label=(\roman*)]
    \item Extension on Euclidean space: Since $\M \subset \R^n$ is closed, a basic strategy for extending the coefficients $U(x)$, i.e., $V(x)$ and $A(x)$, and $P(x)$ to $\R^n$ is to combine the Tubular Neighborhood Theorem with the Urysohn Lemma \citep{munkres2013topology}. The key point is to ensure that the extended coefficients are globally Lipschitz continuous, which requires additional work. This is where the geometric boundedness of $\M$ (Assumption \ref{assum:geometric_iota} and Assumption \ref{assum:uniform_tubular}) and the regularity of $V$ (Assumption \ref{assum:vector_field}) enter. We discuss the construction of a well-posed Euclidean extension of (\ref{eq:mfd_sde_eucl_w}) in Section \ref{sub:well_posed_extension_on_euclidean_space}.

    \item Comparison of discrepancy: For the extended SDE, we apply the Euclidean EM algorithm and obtain a discrete process $\bm{Y}^h_k$, which strongly approximates the solution of the original SDE (\ref{eq:mfd_sde_Bm}) with order $1/2$ \citep{kloeden1992numerical,milstein2013stochastic}. To prove Theorem \ref{thm:ext_p_strong_conv_GEM}, it then suffices to control the discrepancy between $\bm{Y}^h_k$ and $\bm{X}^h_k$. Using the geometric boundedness and Lipschitz properties, we establish strong convergence of $\bm{Y}^h_k$ to $\bm{X}^h_k$ via a Taylor expansion of the exponential map. We provide the details in Section \ref{sub:comparison_of_extrinsic_and_intrinsic_discretization}.
\end{enumerate}

\subsection{Well-posed Extension on Euclidean Space}\label{sub:well_posed_extension_on_euclidean_space}

\parhead{Extension of coefficients.}  Let $F \colon \M \sto \R^n$ (or $\R^{n\times n}$) be a smooth map. Since $\M$ is closed, by the Urysohn Lemma \citep{munkres2013topology} we can choose two tubular neighborhoods $\mathcal{V}_0,\mathcal{V}$ with $\M \subset \clo{\mathcal{V}_0} \subset \mathcal{V}$, and a bump function $\chi \colon \R^n \sto [0,1]$ such that $\chi|_{\M} \equiv 1$ and $\chi|_{\R^n \backslash \clo{\mathcal{V}_0}} \equiv 0$. Define
\begin{equation*}
    \widetilde{F}(y) \defeq \chi(y)F(\pi(y)),\quad \forall~ y \in \R^n,
\end{equation*}
where $\pi$ is the canonical projection of $\mathcal{V}$. Then $\widetilde{F}$ extends $F$ to $\R^n$. This is our main strategy for extending the coefficients $V(x)$, $A(x)$, and $P(x)$ in SDE (\ref{eq:mfd_sde_eucl_w}). The main issue is to ensure the extensions being globally Lipschitz continuous, which is why we assume geometric boundedness of $\M$ and suitable regularity of $V$; see Appendix \ref{appen:global_lipschitz_continuity_of_extension} for details.

For $V$ and $P$, we have the following lemma; see Appendix \ref{appen:proof_of_lemma_global_lip_vec_and_proj} for the proof.
\begin{lem}\label{lem:global_lip_ext_of_Vec_and_Proj}
    Assume that $\M$ satisfies Assumption \ref{assum:geometric_iota} and Assumption \ref{assum:uniform_tubular}, and that $V \in \Gamma(T\M)$ satisfies Assumption \ref{assum:vector_field}. Let $P(x)$ be the orthogonal projection onto $T_x\M$, defined smoothly on $\M$. Then there exist extensions $\widetilde{V}$ and $\widetilde{P}$ of $V$ and $P$, defined on $\R^n$, such that both $\widetilde{V}$ and $\widetilde{P}$ are globally Lipschitz continuous.
\end{lem}

For $A$, we need a more explicit understanding of its geometric structure, as stated in Proposition \ref{prop:geometry_of_A}; see Appendix \ref{appen:proof_of_proposition_prop_geometry_of_a} for the proof. This characterization is useful both for constructing a well-posed extension and for controlling the discrepancy between the Euclidean EM discretization $\bm{Y}^h_k$ and the intrinsic EM discretization $\bm{X}^h_k$.

\begin{prop}\label{prop:geometry_of_A}
    Let $\M \subset \R^n$ be an $m$-dimensional Riemannian submanifold. For any $x \in \M$, let $A(x) \in \R^n$ be defined as (\ref{eq:def_U_A}). Let $\bb{E_i(x)}_{i=1}^m$ be an orthonormal basis of $T_x\M$. Then
    \begin{equation}\label{eq:formula_A_second_form}
        A(x) = \frac{1}{2}\sum_{i=1}^m\mathrm{II}_x\bc{E_i(x),E_i(x)},
    \end{equation}
    where $\mathrm{II}$ is the second fundamental form of $\M$.
\end{prop}

Based on the formula (\ref{eq:formula_A_second_form}) and the geometric boundedness of $\M$, we can construct an extension of $A$ that is globally Lipschitz continuous on $\R^n$, as shown in Lemma \ref{lem:global_lip_ext_of_A}; see Appendix \ref{appen:proof_of_lemma_global_lip_ext_of_a} for the proof.

\begin{lem}\label{lem:global_lip_ext_of_A}
    Assume that $\M$ satisfies Assumption \ref{assum:geometric_iota} and Assumption \ref{assum:uniform_tubular}. For any $x \in \M$, let $A(x) \in \R^n$ be defined as (\ref{eq:def_U_A}). Then there exists an extension $\widetilde{A}$ of $A$ defined on $\R^n$, which is globally Lipschitz continuous.
\end{lem}

\parhead{Extended SDE on $\R^n$.} Therefore, we can choose extensions $\widetilde{V}$, $\widetilde{A}$, and $\widetilde{P}$ that are well-defined on $\R^n$ and globally Lipschitz continuous. Consider the SDE
\begin{equation}\label{eq:mfd_sde_eucl_w_extent}
    \dx{\bm{X}_t} = \widetilde{U}(\bm{X}_t)~\dx{t} + \widetilde{P}(\bm{X}_t)~\dx{\bm{W}_t},\quad t \in [0,T],
\end{equation}
where $\widetilde{U} = \widetilde{V} + \widetilde{A}$. Then (\ref{eq:mfd_sde_eucl_w_extent}) is a well-defined Euclidean SDE with globally Lipschitz continuous coefficients. Moreover, by Proposition 1.2.8 in \citet{hsu2002stochastic}, if $\bm{X}_0 \in \M$, then SDE (\ref{eq:mfd_sde_eucl_w}) and SDE (\ref{eq:mfd_sde_eucl_w_extent}) admit the same solution. Therefore, we continue to write $V$, $A$, and $P$ for $\widetilde{V}$, $\widetilde{A}$, and $\widetilde{P}$, and we do not distinguish SDE (\ref{eq:mfd_sde_eucl_w_extent}) from SDE (\ref{eq:mfd_sde_eucl_w}). 

As a direct application of this construction, standard Euclidean SDE theory \citep[Theorem 6.1]{baudoin2014diffusion} gives existence and uniqueness of a non-explosive solution to SDE (\ref{eq:mfd_sde_eucl_w}), or equivalently SDE (\ref{eq:mfd_sde_Bm}).

\begin{cor}\label{cor:exist_unique_nonexplode_sol_sde_M}
    Assume that $\M$ satisfies Assumption \ref{assum:geometric_iota} and Assumption \ref{assum:uniform_tubular}, and that $V \in \Gamma(T\M)$ satisfies Assumption \ref{assum:vector_field}. Let $1 \leq p <\infty$. For SDE (\ref{eq:mfd_sde_Bm}), if $\E\bj{\norm{\bm{X}_0}^p} < \infty$, then it admits a unique solution $\bm{X}_t$ such that
    \begin{equation*}
        \E\bj{\sup_{t\in[0,T]} \norm*{\bm{X}_t}^p} < \infty.
    \end{equation*}
\end{cor}

A second consequence is central to proving Theorem \ref{thm:ext_p_strong_conv_GEM}. Since SDE (\ref{eq:mfd_sde_eucl_w}) is a well-defined Euclidean SDE with globally Lipschitz continuous coefficients, it can be approximated using the Euclidean EM algorithm. Let $\bm{Y}^h_0 = \bm{X}_0$ and,
\begin{equation}\label{eq:em_to_gem}
    \bm{Y}^h_{k+1} = \bm{Y}^h_k + hU(\bm{Y}^h_k) + \sqrt{h}P(\bm{Y}^h_k)\bm{\xi},\quad k = 0,\ldots,N-1,
\end{equation}
where $\bm{\xi} \sim \mathcal{N}(0,I_n)$, a standard Gaussian on $\R^n$. Let $1 \leq p < \infty$. If $\E\bj{\norm{\bm{X}_0}^p} < \infty$, standard results \citep{kloeden1992numerical,milstein2013stochastic} give
\begin{equation}\label{eq:pf_thm1_extr_bound}
    \E\bj{\max_{0 \leq k \leq N}\norm*{\bm{X}_{t_k}-\bm{Y}_k^h}^p} \leq E_p(T) h^{p/2},
\end{equation}
where $E_p(T) = \mathcal{O}(\exp(T^p))$; see Appendix \ref{appen:euler_maruyama_algorithm_on_euclidean_space} for details. Therefore, to prove Theorem \ref{thm:ext_p_strong_conv_GEM}, it remains to control the discrepancy between $\bm{Y}^h_k$ and $\bm{X}^h_k$.

\subsection{Comparison of Extrinsic and Intrinsic Discretization}\label{sub:comparison_of_extrinsic_and_intrinsic_discretization}

To compare the intrinsic discretization $\bm{X}^h_k$ in (\ref{eq:gem}) with the extrinsic discretization $\bm{Y}^h_k$ in (\ref{eq:em_to_gem}), we first employ an idea from \citet{wang2025geometry} on the Taylor expansion of the exponential map $\exp_x \colon T_x\M \sto \M$. By Lemma \ref{lem:taylor_exponential}, if $\M$ satisfies Assumption \ref{assum:geometric_iota}, then
\begin{equation*}
    \exp_x(v) = x + v + \frac{1}{2}\mathrm{II}_x(v,v) + R_3(x,v),\quad \forall~v \in T_x\M,
\end{equation*} 
where the remainder is uniformly controlled, i.e., $\sup_x\norm{R_3(x,v)} \leq C_R\norm{v}^3$ for some $C_R > 0$. This reduction allows the comparison to be carried out entirely in Euclidean space. A second key formula, induced by the geometric structure of $A(x)$ (Proposition \ref{prop:geometry_of_A}), is
\begin{equation*}
    A(x) = \frac{1}{2}\E\bj{\mathrm{II}_x(P(x)\bm{\xi},P(x)\bm{\xi})},
\end{equation*}
when $\bm{\xi} \sim \mathcal{N}(0,I_n)$; see Appendix \ref{appen:induced_formula}. This formula controls the one-step discrepancy, 
\begin{equation*}
    \norm*{\E\bj{\bm{X}^h_{k+1} \mid \bm{X}^h_{k} = x} - \E\bj{\bm{Y}^h_{k+1} \mid \bm{Y}^h_{k} = x}} \leq \mathcal{O}(h^{3/2}).
\end{equation*}
Combining these results, we obtain Lemma \ref{lem:comparison_EM_GEM}; see Appendix \ref{appen:proof_of_lemma_comparison_em_gem} for the proof.

\begin{lem}\label{lem:comparison_EM_GEM}
    Assume that $\M$ satisfies Assumption \ref{assum:geometric_iota} and Assumption \ref{assum:uniform_tubular}, and that $V \in \Gamma(T\M)$ satisfies Assumption \ref{assum:vector_field}. Let $1 \leq p < \infty$. For $\bm{X}_h^k$ constructing as (\ref{eq:gem}) and $\bm{Y}^h_k$ constructing as (\ref{eq:em_to_gem}), we have
    \begin{equation*}
        \E\bj{\max_{0 \leq k \leq N}\norm*{\bm{X}^h_k-\bm{Y}_k^h}^p} \leq G_p(T) h^{p/2},
    \end{equation*}
    where $G_p(T) = \mathcal{O}(\exp(T^p))$.
\end{lem}

\subsection{Proof of Theorem \ref{thm:ext_p_strong_conv_GEM}}\label{sub:proof_of_thm_ext_p_strong_conv_gem}

Combining the above results, we obtain a direct proof of Theorem \ref{thm:ext_p_strong_conv_GEM}.

\begin{proof}
    As shown in Section \ref{sub:well_posed_extension_on_euclidean_space}, SDE (\ref{eq:mfd_sde_Bm}) is extended to SDE (\ref{eq:mfd_sde_eucl_w}), i.e., SDE (\ref{eq:mfd_sde_eucl_w_extent}), which is well-defined on $\R^n$ with globally Lipschitz continuous coefficients. Therefore, for the Euclidean EM discrete process $\bb{\bm{Y}^h_k}_{k=0}^N$ constructed as (\ref{eq:em_to_gem}), Theorem \ref{thm:p_conver_eucl_EM} shows that
    \begin{equation*}
        \E\bj{\max_{0 \leq k \leq N}\norm*{\bm{X}_{t_k}-\bm{Y}_k^h}^p} \leq E_p(T) h^{p/2},
    \end{equation*}
    where $E_p(T) = \mathcal{O}(\exp(T^p))$.

    Moreover, by Lemma \ref{lem:comparison_EM_GEM}, there exists $G_p(T) = \mathcal{O}(\exp(T^p))$ such that
    \begin{equation*}
        \E\bj{\max_{0 \leq k \leq N}\norm*{\bm{X}^h_k-\bm{Y}_k^h}^p} \leq G_p(T) h^{p/2}. 
    \end{equation*}

    Therefore, for $C_p(T) \defeq E_p(T) + G_p(T) = \mathcal{O}(\exp(T^p))$, we have
    \begin{equation*}
        \E\bj{\max_{0 \leq k \leq N} \norm*{\bm{X}^h_k - \bm{X}_{t_k}}^p} \leq \E\bj{\max_{0 \leq k \leq N}\norm*{\bm{X}^h_k-\bm{Y}_k^h}^p} + \E\bj{\max_{0 \leq k \leq N}\norm*{\bm{X}_{t_k}-\bm{Y}_k^h}^p} \leq C_p(T) h^{p/2}.
    \end{equation*}
\end{proof}

\section{\texorpdfstring{$p$-Wasserstein Convergence of RLD}{p--Wasserstein Convergence of RLD}}\label{sec:p_wasserstein_convergence_of_rld}

Note that for probability distributions $\mu,\nu$ on $\M$, the $p$-Wasserstein distance is defined by
\begin{equation*}
    \mathcal{W}_p(\mu,\nu) \defeq \inf \bb{\E\bj{d_\M(\bm{X},\bm{Y})^p}^{1/p} \colon \bm{X} \sim \mu,\bm{Y} \sim \nu};
\end{equation*}
see \citet{villani2008optimal} for details. Using the $p$-strong convergence of the GEM, we can quantify the $\mathcal{W}_p$ convergence rate of RLD when it is discretized via the GEM.

Let $\M \subset \R^n$ be a Riemannian submanifold with the canonical Riemannian volume measure $\op{vol}$ \citep{jost2008riemannian}. Let $\phi \in C^\infty(\M)$ with $C_\phi \defeq \int_\M \phi ~\dx{\op{vol}} < \infty$, and define the probability distribution $\mu_\phi$ on $\M$ by $\dx{\mu_\phi} \defeq C_\phi^{-1} e^{-\phi}~\dx{\op{vol}}$. To sample from $\mu_\phi$, consider the RLD defined by
\begin{equation}\label{eq:RLD}
    \dx{\bm{X}_t} = -\frac{1}{2}\nabla \phi(\bm{X}_t)~\dx{t} + \dx{\bm{B}^{\M}_t}, \quad \bm{X}_0 =x, \qquad t \in [0,T],
\end{equation}
where $x \in \M$ and $\nabla \phi$ denotes the gradient vector field of $\phi$. Furthermore, to approximate (\ref{eq:RLD}), consider the GEM given by
\begin{equation}\label{eq:RLD_gem}
    \bm{X}^h_{k+1} = \exp_{\bm{X}^h_k}\bc{-\frac{1}{2}h\nabla \phi(\bm{X}^h_k) + \sqrt{h}\bm{\xi}_{\bm{X}^h_k}},\quad k = 0,\ldots,N-1,
\end{equation}
where $\exp_x \colon T_x\M \sto \M$ is the exponential map, and $\bm{\xi}_x$ is a standard Gaussian on $T_x\M$.

Let $\mu_t$ denote the distribution of $\bm{X}_t$, and let $\hat{\mu}_k$ denote the distribution of $\bm{X}_k^h$. Then the total Wasserstein error between $\hat{\mu}_N$ and $\mu_\phi$ can be bounded by
\begin{equation*}
    \mathcal{W}_p\bc{\mu_\phi,\hat{\mu}_N} \leq \mathcal{W}_p\bc{\mu_\phi,\mu_T} + \mathcal{W}_p\bc{\mu_T,\hat{\mu}_N},
\end{equation*}
i.e., the mixing error $\mathcal{W}_p\bc{\mu_\phi,\mu_T}$ of (\ref{eq:RLD}) and the discretization error $\mathcal{W}_p\bc{\mu_T,\hat{\mu}_N}$ of (\ref{eq:RLD_gem}).

\parhead{Mixing error.} It is well-known that if $\phi$ satisfies the Bakry--\'Emery curvature condition \citep{kuwada2010duality,bakry2013analysis,wang2020exponential}, i.e.,
\begin{equation}\label{eq:BE_cond_RLD}
    \op{Ric} + \nabla^2 \phi \geq \lambda_\kappa g,
\end{equation}
for some $\lambda_\kappa > 0$, where $\op{Ric}$ denotes the Ricci curvature tensor; see Appendix \ref{appen:further_preliminaries}. Then
\begin{equation}\label{eq:mix_error_RLD}
    \mathcal{W}_p\bc{\mu_\phi,\mu_T} \leq e^{-\lambda_\kappa T}\mathcal{W}_p\bc{\mu_\phi,\delta_x},
\end{equation}
where $\delta_x$ is the point-distribution at $x$. Condition (\ref{eq:BE_cond_RLD}) also implies that the initial error $\mathcal{W}_p\bc{\mu_\phi,\delta_x}$ is finite; see Appendix \ref{appen:initial_error}.

Moreover, Lemma \ref{lem:low_bound_ricci} shows that if $\M \subset \R^n$ satisfies Assumption \ref{assum:geometric_iota}, then there exists a $\kappa_{\op{ric}} > 0$ such that $\op{Ric} \geq - \kappa_{\op{ric}} g$. Consequently, to bound the mixing error (\ref{eq:mix_error_RLD}), it suffices to require that there exists a $\lambda_\kappa > 0$ such that
\begin{equation}\label{eq:assum_phi_BE}
    \nabla^2 \phi \geq (\lambda_\kappa - \kappa_{\op{ric}})g.
\end{equation}

\parhead{Discretization error.} For the discretization error induced by the GEM, we can use the strong convergence result to obtain an upper bound. Under the geometric boundedness of $\M$ (Assumption \ref{assum:geometric_iota} and the bi-Lipschitz embedding) and the regularity of $\nabla \phi$ (Assumption \ref{assum:uniform_tubular}), Theorem \ref{thm:intr_p_strong_conv_GEM} shows
\begin{equation}\label{eq:discre_error_RLD}
    \mathcal{W}_p\bc{\mu_T,\hat{\mu}_N} \leq \E\bj{d_\M\bc{\bm{X}_T,\bm{X}^h_N}^p}^{1/p} \leq \mathcal{O}\bc{\exp(T^p)h^{1/2}}.
\end{equation}

Combining (\ref{eq:mix_error_RLD}) and (\ref{eq:discre_error_RLD}) gives the following theorem.

\begin{thm}\label{thm:intr_convergence_RLD}
    Let $1 \leq p < \infty$. Let $\M \subset \R^n$ be a Riemannian submanifold such that $\iota \colon \M \hookrightarrow \R^n$ is globally bi-Lipschitz continuous. Let $\phi \in C^\infty(\M)$ with $C_\phi \defeq \int_\M \phi ~\dx{\op{vol}} < \infty$ and $\dx{\mu_\phi} \defeq C_\phi^{-1} e^{-\phi}~\dx{\op{vol}}$. Assume that $\M$ satisfies Assumption \ref{assum:geometric_iota}, that $\phi$ satisfies (\ref{eq:assum_phi_BE}), and that $\nabla \phi$ satisfies Assumption \ref{assum:vector_field}. Then for $\bm{X}^h_k \sim \hat{\mu}_k$ constructed by (\ref{eq:RLD_gem}), we have
    \begin{equation*}
        \mathcal{W}_p\bc{\mu_\phi,\hat{\mu}_N} \leq \mathcal{O}\bc{e^{-T}+\exp(T^p)h^{1/2}}.
    \end{equation*}
\end{thm}

\begin{rmk}
    For Theorem \ref{thm:intr_convergence_RLD}, we provide two additional discussions.
    \begin{enumerate}[label=(\roman*)]
        \item Since $\M \subset \R^n$, $\mu_\phi$, $\mu_t$, and $\hat{\mu}_t$ can be viewed as probability distributions on $\R^n$. Hence, we may use the extrinsic Wasserstein distance induced by the Euclidean distance instead of $d_\M$. For this extrinsic distance, the global bi-Lipschitz continuity of $\iota$ can be replaced by Assumption \ref{assum:uniform_tubular}, since Theorem \ref{thm:ext_p_strong_conv_GEM} applies under Assumption \ref{assum:uniform_tubular} and yields the same discretization bound; see Appendix \ref{appen:extrinsic_wasserstein_bound}.

        \item When $\M$ is a compact Riemannian manifold, as Corollary \ref{cor:p_strong_conv_GEM_cpt} states, the GEM has the strong convergence property. Therefore, if $\phi$ further satisfies the Bakry--\'Emery curvature condition, we obtain the $p$-Wasserstein convergence bound as in the theorem above. This result has also been established in \citet{wang2020fast}, which proved the $2$-Wasserstein convergence as a direct implication of the convergence of KL divergence.
    \end{enumerate}
\end{rmk}

\section{Conclusion}\label{sec:conclusion}

In this work, we established the $p$-strong convergence of the GEM scheme for numerically solving SDEs defined on an embedded Riemannian submanifold $\M \subset \R^n$, and applied it to obtain a $p$-Wasserstein bound for RLD under the GEM discretization. Under the assumptions that $\M$ has bounded extrinsic curvatures and is globally well-embedded in $\R^n$, and that the drift vector field $V$ satisfies standard smoothness conditions, we proved that the GEM achieves the $p$-strong convergence with order $1/2$; moreover, for any compact Riemannian manifold $\M$, we show that these assumptions hold regardless of how $\M$ is Nash-embedded into a Euclidean space, and hence the same conclusion holds on any compact $\M$. For RLD, the $p$-Wasserstein error is decomposed into a mixing error term, controlled under the Bakry--\'Emery curvature condition, and a discretization error term, controlled by our $p$-strong convergence result. Technically, our analysis is based on an extrinsic extension of the $\M$-SDE to $\R^n$ together with a comparison between the resulting Euclidean EM discretization and the intrinsic GEM scheme.

\parhead{Limitations and Future Works.} (\rnum{1}) Our bound for the GEM discretization involves constants that depend exponentially on $T$, i.e., $\mathcal{O}(\exp(T^p))$. Improving this time dependence, both in the Euclidean EM strong convergence analysis and in controlling strong discrepancy between the extrinsic and intrinsic discretized processes, is an interesting direction for future work. (\rnum{2}) Our method relies on the exponential map and standard Gaussian noise on the tangent space, but both can be difficult to compute in practice. Extending the analysis to retractions (as approximations of the exponential map) and to more general noise distributions with suitable moment bounds is an important future direction. (\rnum{3}) Our assumptions on $\M$ depend on extrinsic properties, i.e., how $\M$ is embedded in $\R^n$. Developing analogous results under fully intrinsic geometric conditions on $\M$ is another natural topic for future work.

%% file: appendix_prel.tex
\section{Notation}\label{appen:notation}

The symbols used throughout this paper are clarified below.

\begin{enumerate}[label=(\roman*)]
    \item \textbf{Letters:} We use bold letters, such as $\bm{X}$ and $\bm{W}$, for random variables (vectors), and unbold letters for deterministic variables. Since we mainly work with vectors in $\R^n$, we do not distinguish scalars and vectors, and they are understood from the context. In particular, $I_n \in \R^{n\times n}$ denotes the identity matrix, and $\op{id}$ denotes the identity map. Moreover, for any set $A$, $\mathbb{I}_A$ denotes the characteristic function, i.e., $\mathbb{I}_A(x) = 1$ for $x \in A$ and $0$ otherwise.

    \item \textbf{Norms:} On $\R^n = \bb{(x_1,\ldots,x_n) \colon x_i \in \R}$, the canonical inner product is $\inn{x,y} = \sum_i x_iy_i$, and the norm is given by $\norm{x} = \sqrt{\inn{x,x}}$. For vector spaces $\mathcal{V},\mathcal{W}$ equipped with norms $\norm{\cdot}_{\mathcal{V}}$ and $\norm{\cdot}_{\mathcal{W}}$, respectively, a linear map $F \colon \mathcal{V} \sto \mathcal{W}$ is called bounded if there exists $K>0$ such that for any $v \in \mathcal{V}$,
    \begin{equation*}
        \norm{F(v)}_{\mathcal{W}} \leq K\norm{v}_{\mathcal{V}}.
    \end{equation*}
    In this case, we define the operator norm
    \begin{equation*}
        \norm*{F}_{\op{op}} \defeq \sup_{0\neq v \in \mathcal{V}}\frac{\norm{F(v)}_{\mathcal{W}}}{\norm{v}_{\mathcal{V}}} \leq K < \infty.
    \end{equation*}
    In particular, if $Q \in \R^{n\times k}$, then $\norm{Q}_{\op{op}} = \sup_{v\neq 0} \norm{Qv} / \norm{v}$. Furthermore, the Frobenius norm is defined by $\norm{Q}_{\op{F}} \defeq \sqrt{\tr(QQ^T)}$.

    \item \textbf{Smoothness:} For an open set $\mathcal{V} \subset \R^n$, a function $f \colon \mathcal{V} \sto \R$ is called smooth, denoted by $f \in C^\infty(\mathcal{V})$, if it is infinitely differentiable ($f \in C^k(\mathcal{V})$ if it has continuous derivatives up to order $k$). For $\M \subset \R^n$ a Riemannian submanifold, we say $f \in C^\infty(\M)$ if for any $x \in \M$, there exist an open neighborhood $\mathcal{V}_x \subset \R^n$ of $x$ and a $\tilde{f} \in C^\infty(\mathcal{V}_x)$ such that $\tilde{f}|_{\mathcal{V}_x \cap \M} = f|_{\mathcal{V}_x \cap \M}$. A map $F = (F^1,\cdots,F^k) \colon \M \sto \R^k$ is called smooth ($C^\infty$) if each $F^i \in C^\infty(\M)$. For two Riemannian submanifolds $\M \subset \R^n$ and $\mathcal{N} \subset \R^k$, a map $G \colon \mathcal{N} \sto \M$ is called $C^\infty$ if $\iota \circ G \colon \mathcal{N} \sto \R^n$ is $C^\infty$, where $\iota \colon \M \hookrightarrow \R^n$ is the canonical embedding. Finally, $G \colon \mathcal{N} \sto \M$ is called a diffeomorphism if it is a bijection and both $G$ and $G^{-1}$ are smooth.

    \item \textbf{Integrable:} On a probability space $(\Omega,\mathcal{F},\Pb)$, an $\R^n$-valued random variable $\bm{X}$ is called $L^p$ integrable for $1 \leq p <\infty$, denoted by $\bm{X} \in L^p$, if $\E\bj{\norm{\bm{X}}^p} < \infty$.

\end{enumerate}

\section{Further Preliminaries}\label{appen:further_preliminaries}

In this section, we provide further preliminaries that are necessary for the following analysis.

\subsection{More for Riemannian Submanifold}

Let $\M \subset \R^n$ be a Riemannian submanifold of dimension $m$, and let $\iota \colon \M \hookrightarrow \R^n$ be the canonical embedding. The following concepts appear in our analysis; see \citet{lee2018introduction} for further details.

\parhead{Differential of function.} Let $f \in C^\infty(\M)$. For any $x \in \M$ and any $v \in T_x\M$,
\begin{equation*}
    v(f) \defeq \lv{\frac{\mathrm{d}}{\mathrm{dt}}}_{t = 0}f(\gamma(t)),
\end{equation*}
where $\gamma$ is a smooth curve on $\M$ with $\gamma^\prime(0) = v$. The map $df_x \colon T_x\M \sto \R$ is the $1$-form defined by $df_x(v) \defeq v(f)$. Moreover, for $X \in \Gamma(T\M)$, $X(f) \in C^\infty(\M)$ is defined by $X(f)(x) \defeq X(x)(f)$, and $df(X) \in C^\infty(\M)$ is given by $df(X) = X(f)$.

\parhead{Covariant derivative.} Let $\nabla$ be the Levi-Civita connection. For $X,Y \in \Gamma(T\M)$, by the definition as shown in Section \ref{sec:preliminaries}, also called the Gauss formula \citep[Theorem 8.2]{lee2018introduction}
\begin{equation*}
    \nabla_XY = D_XY - \mathrm{II}(X,Y),\footnote{$D$ is applied to extensions $\tilde{X},\tilde{Y}$ on $\R^n$; we write $X,Y$ for $\tilde{X},\tilde{Y}$ since the result is extension-independent.}
\end{equation*}

For a smooth curve $\gamma \colon [0,1] \sto \M$,  $\gamma$ is also viewed as smooth curve on $\R^n$. So in the following, we denote $\gamma^\prime(t)$ by $\dot{\gamma}(t)$ to emphasize $\gamma^\prime(t) \in T_{\gamma(t)}\M$. Moreover, by the Gauss formula,
\begin{equation*}
    \gamma^{\prime\prime}(t) = D_{\gamma^{\prime}(t)}\gamma^\prime(t) = \nabla_{\dot{\gamma}(t)}\dot{\gamma}(t) + \mathrm{II}(\dot{\gamma}(t),\dot{\gamma}(t)).
\end{equation*}

For $f \in C^\infty(\M)$, the Hessian of $f$, $\nabla^2 f \colon \Gamma(T\M) \times \Gamma(T\M) \sto C^\infty(M)$, is defined by
\begin{equation*}
    \nabla^2f(X,Y) = \nabla_Y(\nabla f)(X) = Y(X(f)) - \nabla_YX (f),
\end{equation*}
which is $C^\infty(\M)$-bilinear and symmetric.

\parhead{Ricci curvature.} For any $X,Y,Z \in \Gamma(T\M)$, the curvature tensor $R \colon \Gamma(T\M) \times \Gamma(T\M) \times \Gamma(T\M) \sto \Gamma(T\M)$ is defined by
\begin{equation*}
    R(X,Y)Z \defeq \nabla_X \nabla_Y Z-\nabla_Y \nabla_X Z-\nabla_{[X, Y]} Z,
\end{equation*}
where $[X,Y] = XY - YX$ is the Lie bracket. Using the same notation, the Riemannian curvature tensor $R$ is
\begin{equation*}
    R(W,Z,X,Y) \defeq \inn{R(X,Y)Z,W} \in C^\infty(M).
\end{equation*}
For any $x\in \M$ and any linearly independent $v,w \in T_x\M$, the sectional curvature at $x$ associated with $v,w$ is
\begin{equation*}
    K_x(v,w) \defeq \frac{R_x(v,w,v,w)}{\norm{v}^2\norm{w}^2 - \inn{v,w}^2}.
\end{equation*}
For any $x \in \M$ and $v \in T_x\M$ with $\norm{v} = 1$, let $\bb{E_i(x)}_{i=1}^m$ be an orthonormal basis with $E_1(x) = v$. Then the Ricci curvature at $x$ in the direction $v$ is
\begin{equation*}
    \op{Ric}_x(v) \defeq \sum_{i=2}^m K_x(v,E_i(x)).
\end{equation*}

\parhead{Local chart.} Let $(\mathcal{U}_\alpha,\phi_\alpha)$ be a local chart, i.e., $\phi_\alpha \colon \mathcal{U}_\alpha \sto \phi_\alpha(\mathcal{U}_\alpha)$ is a diffeomorphism. For any $x \in \M$, let $\phi_\alpha^{-1}(x) = (x^1,\ldots,x^m)$, called the local coordinate of $x$. The tangent space at $x$ is the vector space spanned by
\begin{equation*}
    T_x\M = \op{span}\bb{\lv{\frac{\partial}{\partial x^1}}_x,\cdots,\lv{\frac{\partial}{\partial x^m}}_x},
\end{equation*}
where $\lv{\frac{\partial}{\partial x^i}}_x$ denotes the vector, for $\iota \circ \phi_{\alpha} \colon \mathcal{U}_\alpha \sto \R^n$,
\begin{equation*}
    \lv{\frac{\partial}{\partial x^i}}_x \iota \circ \phi_{\alpha}(x^1,\ldots,x^m) \in \R^n.
\end{equation*}

For any $x \in \M$, the Riemannian metric $g(x) \colon T_x\M \times T_x\M \sto \R$ is given by
\begin{equation*}
    g(x)(u,v) \defeq \inn{u,v},\quad \forall~u,v \in T_x\M.
\end{equation*}
Using the local coordinates above,
\begin{equation*}
    g_{ij}(x) \defeq \inn{\lv{\frac{\partial}{\partial x^i}}_x,\lv{\frac{\partial}{\partial x^j}}_x}.
\end{equation*}
Thus $g(x) = (g_{ij}(x))$ is represented as a matrix. Its inverse is denoted by $g^{-1}(x) = (g^{ij}(x))$.

Using the local coordinates, the Christoffel coefficients $\Gamma^k_{ij}$ are given by
\begin{equation*}
    \nabla_{\frac{\partial}{\partial x^j}}\frac{\partial}{\partial x^i} = \sum_{k=1}^m \Gamma^k_{ij}\frac{\partial}{\partial x^k}. 
\end{equation*}
Moreover, since $\nabla$ is Levi-Civita,
\begin{equation*}
    \Gamma_{i j}^k=\frac{1}{2} \sum_{\alpha=1}^m g^{k \alpha}\left(\frac{\partial}{\partial x^i} g_{j \alpha}+\frac{\partial}{\partial x^j} g_{i \alpha}-\frac{\partial}{\partial x^\alpha} g_{i j}\right).
\end{equation*}

\subsection{Reach}\label{appen:reach}

Recall the Tubular Neighborhood Theorem \citep[Theorem 5.25]{lee2018introduction}.

\begin{thm}\label{thm:tubular_neighborhood}
    Let $\M \subset \R^n$ be a Riemannian submanifold. Then there exists a smooth function $\varepsilon \colon \M \sto (0,\infty)$ such that, for
    \begin{equation*}
        \mathcal{U}_0 = \bb{(x, \xi) \in N\M \colon \norm{\xi} < \varepsilon(x)} \subset N\M,
    \end{equation*}
    the map $T \colon \mathcal{U}_0 \sto \mathcal{V}_0 = T(\mathcal{U}_0) \subset \R^n$ defined by $T(x,\xi) = x+\xi$ is a diffeomorphism.
\end{thm}

This theorem guarantees that $\M \subset \R^n$ always admits a tubular neighborhood. However, the existence of a uniform tubular neighborhood, i.e., $\inf_{x \in \M} \varepsilon(x) > 0$, is nontrivial. To make this condition more explicit, we recast it in terms of Federer's reach \citep{federer1959curvature,leobacher2021existence}. 

\begin{defn}[Reach]\label{defn:reach}
    Let $\M \subset \R^n$ be a Riemannian submanifold. The reach of $\M$ is
    \begin{equation*}
        \op{reach}(\M) \defeq \sup\bb{\varepsilon > 0 \colon \forall~y ~s.t.~ d(y,\M) < \varepsilon,~\exists~!~x \in \M,~s.t.~ \norm{x - y} = d(y,\M)},
    \end{equation*}
    i.e., the largest radius such that any point whose distance to $\M$ is less than this radius admits a unique projection onto $\M$.
\end{defn}

The connection between the reach and the uniform tubular radius is provided by the following theorem \citep{leobacher2021existence,platt2025lower}.

\begin{thm}\label{thm:reach_to_tubular}
    Let $\M \subset \R^n$ be a closed Riemannian submanifold without boundary. If $r_0 = \op{reach}(\M) > 0$, then for
    \begin{equation*}
        \mathcal{U}_{r_0} = \bb{(x, \xi) \in N\M \colon \norm{\xi} < r_0},\quad \mathcal{V}_{r_0} = \bb{y \in \R^n \colon d(y,\M) < r_0},
    \end{equation*}
    the tubular map $T \colon \mathcal{U}_{r_0} \sto \mathcal{V}_{r_0}$ by $T(x,\xi) = x + \xi$ is a diffeomorphism. In particular, $\mathcal{V}_{r_0}$ is a uniform tubular neighborhood with radius $r_0$.
\end{thm}

\begin{rmk}
    The converse direction is immediate. If $\M$ admits a uniform tubular neighborhood with radius $r$, then for any $y$ with $d(y,\M) < r$, the point $x = \pi(y)$ is the unique nearest point in $\M$. Hence $\op{reach}(\M) \geq r > 0$.
\end{rmk}

Consequently, Assumption \ref{assum:uniform_tubular} is equivalent to $\op{reach}(\M) > 0$.

\subsection{Martingale Theory}

The following material from martingale theory is essential for our analysis; see \citet{le2016brownian} for further details.

Fix a filtered probability space $(\Omega,(\mathcal{F}_t)_{t \geq 0},\mathcal{F},\Pb)$. An $\R^n$-valued stochastic process $(\bm{X}_t)_{t \geq 0}$ is called an $(\mathcal{F}_t)_{t \geq 0}$-martingale if $\bm{X}_t \in L^1$, $\bm{X}_t$ is $\mathcal{F}_t$-measurable, and for any $t \geq s \geq 0$,
\begin{equation*}
    \E\bj{\bm{X}_t \mid \mathcal{F}_s} = \bm{X}_s.
\end{equation*}
The same definition applies in discrete time, i.e., $t = n \in \N$. Moreover, for a continuous-time process $(\bm{X}_t)_{t \geq 0}$, we always assume it has continuous paths, i.e., $t \mapsto \bm{X}_t$ is almost surely continuous.

Let $\bm{\tau} \colon \Omega \sto [0,\infty]$. We call $\bm{\tau}$ an $(\mathcal{F}_t)_{t \geq 0}$-stopping time if $\bb{\bm{\tau} \leq t} \in \mathcal{F}_t$ for all $t \geq 0$. A process $(\bm{X}_t)_{t \geq 0}$ is called a local martingale if $\bm{X}_0 = 0$ and there exists a sequence of stopping times $\bm{\tau}_n \uparrow \infty$ almost surely such that $(\bm{X}_{t \wedge \bm{\tau}_n})_{t \geq 0}$ is a martingale for all $n$.

Note that any martingale is a local martingale, but the converse need not hold. A basic example of a local martingale is given by the It\^o integral. Let $\bm{W}$ be a standard Brownian motion on $\R^n$. If the It\^o integral
\begin{equation*}
    \bm{M}_t = \int_0^t \bm{X}_s ~\dx{\bm{W}_s}
\end{equation*}
is well-defined (i.e., $\bm{X}_t$ is progressively measurable and $\int_0^t \norm{\bm{X}_s}^2~\dx{s} < \infty$ almost surely), then $\bm{M}_t$ is a local martingale. For the local martingale $(\bm{M}_t)_{t \geq 0}$, its (scalar) quadratic variation $([\bm{M}]_t)_{t\geq 0}$ is given by
\begin{equation*}
    [\bm{M}]_t = \int_0^t \norm*{\bm{X}_s}^2~\dx{s}.
\end{equation*}

%% file: appendix_geo.tex
\section{Geometric Boundedness under Assumption \ref{assum:geometric_iota}}\label{appen:geometric_properties}

For the canonical embedding $\iota \colon \M \hookrightarrow \R^n$, Assumption \ref{assum:geometric_iota} yields the desired geometric properties of $\M$. We establish these implications in this section.

We first clarify the definitions of $\nabla d\iota$ and $\nabla^2 d\iota$. Since $\iota \colon \M \sto \R^n$, the differential $d\iota$ is a vector-valued $1$-form on $\M$. Writing $\iota = (\iota^1,\ldots,\iota^n)$, we have
\begin{equation*}
    d\iota = (d\iota^1,\ldots,d\iota^n).
\end{equation*}
Accordingly, the covariant derivative $\nabla d\iota$ is defined canonically by
\begin{equation}\label{eq:def_nabla_d_iota}
    (\nabla_X d\iota)(Y) = D_X(d\iota(Y)) - d\iota (\nabla_XY),\quad \forall~X,Y \in \Gamma\bc{T\M},
\end{equation}
where $D_X = D_{\iota(X)}$, by viewing $X$ as a vector field in $\R^n$ via the embedding. Note that $\nabla d\iota$ is an $\R^n$-valued $2$-form on $\M$, defined by  
\begin{equation*}
    (\nabla d\iota)(Y,X) = (\nabla_X d\iota)(Y),\quad \forall~X,Y \in \Gamma\bc{T\M}.
\end{equation*}
Since both $\nabla$ and $D$ are Levi--Civita connections, it follows that $\nabla d\iota$ is symmetric. 

Similarly, we define the $\R^n$-valued $3$-form $\nabla^2 d\iota = \nabla(\nabla d\iota)$ by, for $X,Y,Z \in \Gamma\bc{T\M}$,
\begin{equation}\label{eq:def_nabla2_d_iota}
    (\nabla (\nabla d\iota))(X,Y,Z) = D_Z\bc{(\nabla d\iota)(X,Y)} - (\nabla d\iota)(\nabla_ZX,Y) - (\nabla d\iota)(X,\nabla_ZY).
\end{equation}
Further details can be found in \citet{tu2017differential} and \citet{taubes2011differential}.

\subsection{Boundedness of Second Fundamental Form}\label{appen:boundedness_of_second_fundamental_form}

Let us first clarify the definition of $\nabla \mathrm{II}$. For $\M \subset \R^n$ a Riemannian submanifold, recall the definition of the second fundamental form $\mathrm{II}$,
\begin{equation*}
    \mathrm{II}(X,Y) = (D_XY)^\perp = D_XY - (D_XY)^\top, \quad \forall~X,Y \in \Gamma\bc{T\M},
\end{equation*}
where $(D_XY)^\top$ denotes the orthogonal projection of $D_XY$ onto $T\M$. Moreover, the covariant derivative of $\mathrm{II}$ can be defined by, for all $X,Y,Z \in \Gamma\bc{T\M}$,
\begin{equation}\label{eq:def_nabla2_second}
    (\nabla \mathrm{II}) (X,Y,Z) = \bc{D^\perp_Z\mathrm{II}}(X,Y) - \mathrm{II}(\nabla_ZX,Y) - \mathrm{II}(X,\nabla_ZY), 
\end{equation}
where $D^\perp_Z\mathrm{II}(X,Y) \defeq \bc{D_Z\mathrm{II}(X,Y)}^\perp$. Note that $(\nabla \mathrm{II})(X,Y,Z) \in \Gamma(N\M)$; see, e.g., \citet{lee2018introduction} for details. For convenience, we write $\nabla \mathrm{II}_x \defeq \nabla \mathrm{II}(x)$ for any $x \in \M$. 

To relate Assumption \ref{assum:geometric_iota} to the boundedness of the second fundamental form, we first establish the following explicit formulas.

\begin{lem}\label{lem:iota_second_form}
    Let $(\M,g)$ be a Riemannian manifold isometrically embedded in $\R^n$ by the map $\iota \colon \M \hookrightarrow \R^n$. Let $\mathrm{II}$ be the second fundamental form of $\M$. Then for $X,Y,Z \in \Gamma(T\M)$,
    \begin{equation*}
        \mathrm{II}(X,Y) = (\nabla d\iota) (X,Y),\quad\bc{\nabla^2 d\iota}(X,Y,Z) = (\nabla \mathrm{II})(X,Y,Z) - S_{\mathrm{II}(X,Y)}(Z),
    \end{equation*}
    where $S_{\mathrm{II}(X,Y)}$ is the shape operator along $\mathrm{II}(X,Y)$.
\end{lem}
\begin{proof}
    By definition (\ref{eq:def_nabla_d_iota}),
    \begin{equation*}
        D_X(Y) = \nabla_XY + (\nabla d\iota)(X,Y),
    \end{equation*}
    where we view $Y = d\iota(Y)$ and $\nabla_X Y = d\iota(\nabla_X Y)$ as vector fields in $\R^n$. On the other hand, by the Gauss formula \citep[Theorem 8.2]{lee2018introduction},
    \begin{equation*}
        D_X(Y) = \nabla_XY + \mathrm{II}(X,Y).
    \end{equation*}
    Comparing the two identities implies $\mathrm{II}(X,Y) = (\nabla d\iota)(X,Y)$.

    For the second identity, let $A \defeq \nabla d\iota$. By definition (\ref{eq:def_nabla2_second}),
    \begin{equation*}
        (\nabla A)(X,Y,Z) = D_Z\bc{A(X,Y)} - A(\nabla_ZX,Y) - A(X,\nabla_ZY), \quad \forall~X,Y,Z \in \Gamma\bc{T\M}.
    \end{equation*}
    Using $A=\mathrm{II}$ and decomposing $D_Z\bc{\mathrm{II}(X,Y)}$ into tangent and normal components, we obtain
    \begin{equation*}
        \bc{\nabla^2 d\iota}(X,Y,Z) = (\nabla \mathrm{II})(X,Y,Z) + \bc{D_Z\bc{\mathrm{II}(X,Y)}}^\top.
    \end{equation*}
    Since $\mathrm{II}(X,Y) \in \Gamma(N\M)$, the Weingarten Equation \citep[Proposition 8.4]{lee2018introduction} implies
    \begin{equation*}
        \bc{D_Z\bc{\mathrm{II}(X,Y)}}^\top = -S_{\mathrm{II}(X,Y)}(Z).
    \end{equation*}
    Therefore,
    \begin{equation*}
        \bc{\nabla^2 d\iota}(X,Y,Z) = (\nabla \mathrm{II})(X,Y,Z) - S_{\mathrm{II}(X,Y)}(Z). 
    \end{equation*}
\end{proof}
\begin{rmk}
    For simplicity, we occasionally use $\mathrm{II} = \nabla d\iota$ and $\nabla \mathrm{II} = \nabla^2 d\iota + S_{\mathrm{II}}$.
\end{rmk}

\begin{lem}\label{lem:bound_shape_operator}
    Let $\M \subset \R^n$ be a Riemannian submanifold.
    \begin{enumerate}[label=(\roman*)]
        \item If $\kappa = \sup_\M\norm*{\mathrm{II}}_{\op{op}} < \infty$, then for any normal vector field $\eta \in \Gamma\bc{N\M}$, the shape operator $S_\eta$ satisfies
        \begin{equation*}
            \norm*{S_\eta}_{\op{op}} \leq \kappa \norm{\eta}.
        \end{equation*}

        \item If
        \begin{equation*}
            r = \sup_{x \in \M}\sup_{\xi \in N_x\M} \frac{\norm*{S_{x,\xi}}_{\op{op}}}{\norm{\xi}} < \infty,
        \end{equation*}
        then $\sup_{x \in \M}\norm*{\mathrm{II}_x}_{\op{op}} \leq r$.
    \end{enumerate}
\end{lem}
\begin{proof}
    Recall that the shape operator (or Weingarten map) $S_{x,\xi} \colon T_x\M \sto T_x\M$ along the normal direction $\xi$ is defined by \citep[Chapter 8]{lee2018introduction}
    \begin{equation*}
        \inn{S_{x,\xi}(u),v} = \inn{\mathrm{II}_x(u,v),\xi},\quad \forall~u,v \in T_x\M.
    \end{equation*}
    \begin{enumerate}[label=(\roman*)]
        \item For any $\eta \in \Gamma(N\M)$,
        \begin{equation*}
            \norm*{S_\eta}_{\op{op}} = \sup_{X,Y \neq 0} \frac{\abs{\inn{S_\eta(X),Y}}}{\norm{X}\norm{Y}} \leq \sup_{X,Y \neq 0} \frac{\norm{\mathrm{II}(X,Y)}}{\norm{X}\norm{Y}} \norm{\eta} \leq \kappa \norm{\eta}.
        \end{equation*}

        \item By definition,
        \begin{equation*}
            \norm*{\mathrm{II}_x}_{\op{op}} = \sup_{u,v \neq 0} \frac{\norm{\mathrm{II}_x(u,v)}}{\norm{u}\norm{v}} = \sup_{u,v \neq 0}\sup_{\xi \neq 0} \frac{\abs{\inn{\mathrm{II}_x(u,v),\xi}}}{\norm{u}\norm{v}\norm{\xi}} = \sup_{\xi \in N_x\M} \frac{\norm*{S_{x,\xi}}_{\op{op}}}{\norm{\xi}} \leq r.
        \end{equation*}
        Taking the supremum over $x \in \M$ implies
        \begin{equation*}
            \sup_{x \in \M}\norm*{\mathrm{II}_x}_{\op{op}} \leq r. 
        \end{equation*}
    \end{enumerate}
\end{proof}

\begin{lem}\label{lem:assum_to_bound_second_form}
    Let $\M \subset \R^n$ be a Riemannian submanifold with embedding map $\iota$. Assumption \ref{assum:geometric_iota} is equivalent to
    \begin{equation*}
        \sup_{x\in\M}\norm*{\mathrm{II}_x}_{\op{op}} <\infty,\quad \sup_{x\in\M}\norm*{\nabla \mathrm{II}_x}_{\op{op}} <\infty,
    \end{equation*}
    where $\mathrm{II}$ denotes the second fundamental form of $\M$.
\end{lem}
\begin{proof}
    Assume that $\iota$ satisfies Assumption \ref{assum:geometric_iota}. By Lemma \ref{lem:iota_second_form},
    \begin{equation*}
        \sup_\M\norm*{\mathrm{II}}_{\op{op}} = \sup_\M\norm*{\nabla d \iota}_{\op{op}} < \infty.
    \end{equation*}
    Moreover, since $\nabla^2 d\iota = \nabla \mathrm{II} - S_{\mathrm{II}}$ and $\nabla \mathrm{II} \perp S_{\mathrm{II}}$,
    \begin{equation*}
        \norm*{\nabla \mathrm{II}}_{\op{op}} \leq \sup_{X,Y,Z} \frac{\norm*{\bc{\nabla^2 d\iota}(X,Y,Z)}}{\norm{X}\norm{Y}\norm{Z}} = \norm*{\nabla^2 d\iota}_{\op{op}}.
    \end{equation*}
    Therefore,
    \begin{equation*}
        \sup_\M\norm*{\nabla \mathrm{II}}_{\op{op}} \leq \sup_\M\norm*{\nabla^2 d\iota}_{\op{op}} < \infty.
    \end{equation*}

    Conversely, assume that $\sup_\M\norm*{\mathrm{II}}_{\op{op}} < \infty$ and $\sup_\M\norm*{\nabla \mathrm{II}}_{\op{op}} < \infty$. Since $\mathrm{II} = \nabla d\iota$, it follows immediately that $\sup_\M\norm*{\nabla d \iota}_{\op{op}} < \infty$. For the second derivative term, let $\kappa = \sup_\M\norm*{\mathrm{II}}_{\op{op}} < \infty$. By Lemma \ref{lem:bound_shape_operator},
    \begin{equation*}
        \norm*{S_{\mathrm{II}(X,Y)}}_{\op{op}} \leq \kappa \norm{\mathrm{II}(X,Y)} \leq \kappa^2 \norm{X}\norm{Y}, \quad \text{i.e.},~\norm*{S_{\mathrm{II}}}_{\op{op}} \leq \kappa^2.
    \end{equation*}
    Hence, using again $\nabla^2 d\iota = \nabla \mathrm{II} - S_{\mathrm{II}}$, we obtain
    \begin{equation*}
        \norm*{\nabla^2 d\iota}_{\op{op}} \leq \norm*{\nabla \mathrm{II}}_{\op{op}} + \norm*{S_{\mathrm{II}}}_{\op{op}} < \infty. 
    \end{equation*}
\end{proof}

\subsection{Taylor Formula of Exponential Map}\label{appen:taylor_formula_of_exponential_map}

Since $\M \subset \R^n$, we can consider a Taylor expansion of the exponential map $\exp_x \colon T_x\M \sto \M$. Moreover, Assumption \ref{assum:geometric_iota} is useful for controlling the remainder term.

\begin{lem}\label{lem:taylor_exponential}
    Let $\M \subset \R^n$ be a Riemannian submanifold such that the second fundamental form satisfies
    \begin{equation*}
        \kappa_1 \defeq \sup_{x \in \M}\norm*{\mathrm{II}_x}_{\op{op}} <\infty,\quad\kappa_2 \defeq \sup_{x \in \M}\norm*{\nabla \mathrm{II}_x}_{\op{op}} <\infty.
    \end{equation*}
    For any $x \in \M$, let $\exp_x \colon T_x\M \sto \M$ be the exponential map. Then for any $v \in T_x\M$,
    \begin{equation*}
        \exp_x(v) = x + v + \frac{1}{2}\mathrm{II}_x(v,v) + R_3(x,v),
    \end{equation*}
    where
    \begin{equation*}
        \norm{R_3(x,v)} \leq \frac{\kappa^2_1 + \kappa_2}{6} \norm{v}^3.
    \end{equation*}
\end{lem}
\begin{proof}
    Fix $x \in \M$ and $v \in T_x\M$. First, by the Hopf--Rinow theorem \citep[Theorem 6.19]{lee2018introduction}, $\exp_x$ is well-defined on $T_x\M$ because of the closedness of $\M$. Let $\gamma \colon [0,1] \sto \M$ be defined by $\gamma(t) = \exp_x(tv)$, i.e., the geodesic satisfying
    \begin{equation*}
        \gamma(0) = x,\quad \dot{\gamma}(0) = v.
    \end{equation*}
    Since $\gamma(t) \in \R^n$, we write $\gamma^{\prime}(t)$ for its Euclidean derivative, and clearly $\gamma^\prime(t) = \dot{\gamma}(t)$. The second derivative, however, should be interpreted extrinsically. By the Gauss's formula \citep[Theorem 8.2]{lee2018introduction},
    \begin{equation*}
        \gamma^{\prime\prime}(t) = D_{\gamma^\prime(t)}\gamma^\prime(t) = \nabla_{\dot{\gamma}(t)}\dot{\gamma}(t) + \mathrm{II}(\dot{\gamma}(t),\dot{\gamma}(t)) = \mathrm{II}(\dot{\gamma}(t),\dot{\gamma}(t)),
    \end{equation*}
    since $\gamma$ is a geodesic and hence $\nabla_{\dot{\gamma}(t)}\dot{\gamma}(t) \equiv 0$. Applying the Euclidean Taylor formula to $\gamma(t)$ gives
    \begin{align*}
        \gamma(1)
        &= \gamma(0) + \gamma^{\prime}(0) + \frac{1}{2}\gamma^{\prime\prime}(0) + \int_0^1\frac{(1-t)^2}{2} \gamma^{\prime\prime\prime}(t) ~\dx{t} \\
        &= x + v + \frac{1}{2}\mathrm{II}(v,v) + R_3(x,v),
    \end{align*}
    where
    \begin{equation*}
        R_3(x,v) = \int_0^1\frac{(1-t)^2}{2} \gamma^{\prime\prime\prime}(t) ~\dx{t}.
    \end{equation*}
    For $\gamma^{\prime\prime\prime}$,
    \begin{equation*}
        \gamma^{\prime\prime\prime}(t) = D_{\gamma^\prime(t)}\gamma^{\prime\prime}(t) = D_{\dot{\gamma}(t)} \mathrm{II}(\dot{\gamma}(t),\dot{\gamma}(t)).
    \end{equation*}
    Let $X = \dot{\gamma}(t)$. By the Weingarten Equation \citep[Proposition 8.4]{lee2018introduction},
    \begin{equation}\label{eq:dire_second_bound_exp}
        \gamma^{\prime\prime\prime}(t)= D_{X} \mathrm{II}(X,X) = -S_{\mathrm{II}(X,X)}(X) + D_X^\perp \mathrm{II}(X,X).
    \end{equation}
    For the first term, since
    \begin{equation*}
        \kappa_1 = \sup_\M\norm*{\mathrm{II}}_{\op{op}} <\infty,
    \end{equation*}
    Lemma \ref{lem:bound_shape_operator} implies
    \begin{equation}\label{eq:shape_bound_exp}
        \norm*{S_{\mathrm{II}(X,X)}(X)} \leq \norm*{S_{\mathrm{II}(X,X)}}_{\op{op}} \norm{X} \leq \kappa_1 \norm*{\mathrm{II}(X,X)} \norm{X} \leq \kappa_1^2\norm{X}^3.
    \end{equation}
    For the second term, by (\ref{eq:def_nabla2_second}),
    \begin{align*}
        D_X^{\perp} \mathrm{II}(X, X)
        &=\left(\nabla_X \mathrm{II}\right)(X, X)+\mathrm{II}\left(\nabla_X X, X\right)+\mathrm{II}\left(X, \nabla_X X\right) \\
        &=\left(\nabla_X \mathrm{II}\right)(X, X),
    \end{align*}
    since $\nabla_XX = 0$ for $X = \dot{\gamma}(t)$. Therefore,
    \begin{equation}\label{eq:perp_dire_bound_exp}
        \norm*{D_X^{\perp} \mathrm{II}(X, X)} = \norm*{\left(\nabla_X \mathrm{II}\right)(X, X)} \leq \kappa_2 \norm{X}^3.
    \end{equation}
    Combining (\ref{eq:shape_bound_exp}), (\ref{eq:perp_dire_bound_exp}), and (\ref{eq:dire_second_bound_exp}), we obtain
    \begin{equation*}
        \norm*{\gamma^{\prime\prime\prime}(t)} \leq \norm*{S_{\mathrm{II}(X,X)}(X)} + \norm*{D_X^{\perp} \mathrm{II}(X, X)} \leq \bc{\kappa^2_1 + \kappa_2}\norm{X}^3.
    \end{equation*}
    Moreover, since $\gamma(t)$ is a geodesic, $\norm{X} = \norm*{\dot{\gamma}(t)} \equiv \norm*{\dot{\gamma}(0)} = \norm{v}$. Consequently,
    \begin{equation*}
        \norm*{R_3(x,v)} \leq \int_0^1\frac{(1-t)^2}{2} \norm*{\gamma^{\prime\prime\prime}(t)} ~\dx{t} \leq \frac{\kappa^2_1 + \kappa_2}{6} \norm{v}^3. 
    \end{equation*}
\end{proof}

\subsection{Boundedness of Ricci Curvature}\label{appen:boundedness_of_ricci_curvature}

By Lemma \ref{lem:assum_to_bound_second_form}, Assumption \ref{assum:geometric_iota} is equivalent to uniform bounds on the second fundamental form and its covariant derivative. Although these conditions are extrinsic, they also encode intrinsic geometric information about $\M$.

\begin{lem}\label{lem:low_bound_ricci}
    Let $\M \subset \R^n$ be an $m$-dimensional Riemannian submanifold such that the second fundamental form is uniformly bounded, i.e., $\kappa = \sup_\M\norm*{\mathrm{II}}_{\op{op}} <\infty$. Then its Ricci curvature satisfies
    \begin{equation*}
        \op{Ric} \geq -2(m-1)\kappa^2g.
    \end{equation*}
\end{lem}
\begin{proof}
    Let $R$ be the Riemannian curvature tensor of $\M$. For any $X,Y,Z,W \in \Gamma(T\M)$, the Gauss's formula \citep[Theorem 8.5]{lee2018introduction} gives
    \begin{equation*}
        \clo{R}(W,X,Y,Z) = R(W,X,Y,Z) - \inn{\mathrm{II}(W,Y),\mathrm{II}(X,Z)} + \inn{\mathrm{II}(X,Y),\mathrm{II}(W,Z)},
    \end{equation*}
    where $\clo{R}$ is the Riemannian curvature tensor of $\R^n$. Since $\clo{R} = 0$, we obtain
    \begin{equation*}
        R(W,X,Y,Z) = \inn{\mathrm{II}(W,Y),\mathrm{II}(X,Z)} - \inn{\mathrm{II}(X,Y),\mathrm{II}(W,Z)}.
    \end{equation*}
    Therefore, for orthonormal $X,Y$, the sectional curvature $K$ of $\M$ satisfies
    \begin{equation*}
        K(X,Y) = R(X,Y,X,Y) = \inn{\mathrm{II}(X,X),\mathrm{II}(Y,Y)} - \norm*{\mathrm{II}(X,Y)}^2.
    \end{equation*}
    Since $\kappa = \sup_\M\norm*{\mathrm{II}}_{\op{op}} <\infty$, we have $\norm*{\mathrm{II}(X,Y)} \leq \kappa \norm{X}\norm{Y}$, and hence
    \begin{equation*}
        K(X,Y) \geq -\norm*{\mathrm{II}(X,X)}\norm*{\mathrm{II}(Y,Y)} - \norm*{\mathrm{II}(X,Y)}^2 \geq -2\kappa^2.
    \end{equation*}
    For a unit vector $X$, let $\bb{E_i}_{i=1}^m$ be an orthonormal frame with $E_1 = X$. Then
    \begin{equation*}
        \op{Ric}(X,X) = \sum_{i=2}^m K(E_i,X) \geq -2(m-1)\kappa^2. 
    \end{equation*}
\end{proof}

\section{Orthogonal Projection and Its Boundedness}\label{appen:projection}

Since we work with the manifold $\M \subset \R^n$ from an extrinsic perspective, it is important to consider the orthogonal projection $P(x) \colon \R^n \sto \R^n$ with $\Img P(x) = T_x\M$ for each $x \in \M$. The existence of such a projection is guaranteed by the Tubular Neighborhood Theorem (Theorem \ref{thm:tubular_neighborhood}); see Section \ref{appen:existence_of_orthogonal_projection}. Furthermore, we study the boundedness of the canonical projection induced by the tubular map (Section \ref{appen:boundedness_of_canonical_projection}) and the boundedness of the corresponding orthogonal projection $P$.

\subsection{Existence of Orthogonal Projection}\label{appen:existence_of_orthogonal_projection}

\begin{lem}\label{lem:projec_of_Dpi}
    Let $\M \subset \R^n$ be a Riemannian submanifold with some tubular neighborhood $\mathcal{V}$. Then for any $x \in \M$,
    \begin{equation*}
        P(x) \defeq D\pi(x),
    \end{equation*}
    is the orthogonal projection onto $T_x\M$, where $\pi \colon \mathcal{V} \sto \M$ is the canonical projection.
\end{lem}
\begin{proof}
    Fix $x \in \M$. First, for any $v \in T_x\M$, let $c \colon (-\varepsilon,\varepsilon) \sto \M$ be a curve with $c(0) = x$ and $\dot{c}(0) = v$. Because $\pi|_\M = \op{id}_\M$, we have $\pi(c(t)) = c(t)$. It follows that
    \begin{equation*}
        D\pi(x)[v] = \lv{\frac{\mathrm{d}}{\mathrm{d}t}}_{t = 0} \pi(c(t)) = \lv{\frac{\mathrm{d}}{\mathrm{d}t}}_{t = 0} c(t) = v.
    \end{equation*}
    Therefore, $D\pi(x) \colon \R^n = T_x\M \oplus N_x\M \sto \R^n$ restricts to the identity on $T_x\M$. In particular, $\Img D\pi(x) \supset T_x\M$. On the other hand, since $\pi \colon \mathcal{V}_0 \sto \M$, we have $\Img D\pi(x) \subset T_x\M$. Hence,
    \begin{equation*}
        \Img D\pi(x) = T_x\M.
    \end{equation*}

    Next, for any $\xi \in N_x\M$, consider the curve $c(t) = x + t\xi$. For $\abs{t} < \varepsilon(x)$, the definition of the canonical projection implies $\pi(c(t)) \equiv x$. Therefore,
    \begin{equation*}
        D\pi(x)[\xi] = \lv{\frac{\mathrm{d}}{\mathrm{d}t}}_{t = 0} \pi(c(t)) = 0,
    \end{equation*}
    so $N_x\M \subset \ker D\pi(x)$. Moreover, by the Rank Theorem,
    \begin{equation*}
        \dim \ker D\pi(x) = n - \dim \Img D\pi(x) = n - \dim T_x\M = \dim N_x\M,
    \end{equation*}
    and hence $\ker D\pi(x) = N_x\M$. Consequently, $P(x) \defeq D\pi(x)$ is the orthogonal projection onto $T_x\M$. 
\end{proof}
\begin{rmk}
    Since $\pi$ is smooth on $\mathcal{V}$, $P(x) = D\pi(x)$ is also smooth on $\mathcal{V}$, and in particular smooth on $\M$. For further discussion of orthogonal projections for general manifolds, see \citet{leobacher2021existence}.
\end{rmk}

\subsection{Boundedness of Canonical Projection}\label{appen:boundedness_of_canonical_projection}

By definition of the canonical projection $\pi$ (Section \ref{sub:pre_riemannian_manifold}), for any $y \in \mathcal{V}_0$, $D\pi(y)$ is a projection onto $T_{\pi(y)}\M$, but it is not necessarily an orthogonal projection. In the isometric embedding setting, however, $D\pi(y)$ admits an explicit formula, which also provides a convenient way to bound $\norm*{D\pi(y)}_{\op{op}}$.

\begin{prop}\label{prop:formula_of_Dpi}
    Let $\M \subset \R^n$ be a Riemannian submanifold and let $\mathcal{V}_0$ be a tubular neighborhood in $\R^n$ with the canonical projection $\pi$. Then for any $y \in \mathcal{V}_0$,
    \begin{equation*}
        (\op{id} - S_{x,\xi}) \circ D\pi(y) = P(x),
    \end{equation*}
    where $y = x + \xi$ is the unique decomposition with $x \in \M$ and $\xi \in N_x\M$, and $S_{x,\xi}$ is the shape operator along $\xi$ at $x$.
\end{prop}
\begin{proof}
    Fix $y \in \mathcal{V}_0$ and write $y = x + \xi$ for the unique decomposition. Recall that $x = \pi(y)$. Define $\xi \colon \mathcal{V}_0 \sto N\M$ by $\xi(y) = \xi$; this map is smooth. Hence for any $X \in \Gamma(T\M)$,
    \begin{equation*}
        \inn{\xi(y),X_{\pi(y)}} = 0,\quad \forall~ y \in \mathcal{V}_0.
    \end{equation*}
    Let $y(t)$ be a curve in $\mathcal{V}_0$ and set $x(t) = \pi(y(t))$. By differentiating,
    \begin{equation}\label{eq:diff_inn_norm_tan}
        0 = \lv{\frac{\mathrm{d}}{\mathrm{d}t}}_{t = 0}\inn{\xi(y(t)),X_{x(t)}} = \inn{D\xi(y)[y^\prime(0)], X_{x(0)}} + \inn{\xi(y(0)), D_{x^\prime(0)}X_{x(0)}}.
    \end{equation}
    Fix $w \in \R^n$ and choose $y(t)$ such that $y(0)=y$ and $y^\prime(0)=w$. Then by the chain rule,
    \begin{equation*}
        x^\prime(0) = \lv{\frac{\mathrm{d}}{\mathrm{d}t}}_{t = 0}\pi(y(t)) = D\pi(y)[w].
    \end{equation*}
    Using this in (\ref{eq:diff_inn_norm_tan}) implies
    \begin{equation}\label{eq:diff_inn_norm_tan_result}
        \inn{D\xi(y)[w], X_{x}} + \inn{\xi, D_{D\pi(y)[w]}X_{x}} = 0.
    \end{equation}
    Since $D\pi(y)[w] \in T_x\M$, the Gauss formula \citep[Theorem 8.2]{lee2018introduction} gives
    \begin{equation*}
        D_{D\pi(y)[w]}X_{x} = \nabla_{D\pi(y)[w]}X_{x} + \mathrm{II}_x\bc{D\pi(y)[w],X_x}.
    \end{equation*}
    Therefore,
    \begin{equation}\label{eq:inn_xi_as_shape}
        \inn{\xi, D_{D\pi(y)[w]}X_{x}} = \inn{\xi,\mathrm{II}_x\bc{D\pi(y)[w],X_x}} = \inn{S_{x,\xi}\bc{D\pi(y)[w]},X_x},
    \end{equation}
    because $\inn{\xi, \nabla_{D\pi(y)[w]}X_{x}} = 0$ for $\xi \in N_x\M$. Combining (\ref{eq:inn_xi_as_shape}) with (\ref{eq:diff_inn_norm_tan_result}), we obtain
    \begin{equation*}
        \inn{D\xi(y)[w], X_{x}} + \inn{S_{x,\xi}\bc{D\pi(y)[w]},X_x} = 0.
    \end{equation*}
    Since this holds for all $X \in \Gamma(T\M)$,
    \begin{equation}\label{eq:tan_of_Dpiy}
        P(x)\bc{D\xi(y)[w]} = -S_{x,\xi}\bc{D\pi(y)[w]}.
    \end{equation}
    On the other hand, from
    \begin{equation*}
        y = \pi(y) + \xi(y),
    \end{equation*}
    taking the directional derivative along $w$ gives
    \begin{equation*}
        w = D\pi(y)[w] + D\xi(y)[w].
    \end{equation*}
    Applying $P(x)$ to both sides, and using $P(x)\bc{D\pi(y)[w]} = D\pi(y)[w]$ since $D\pi(y)[w] \in T_x\M$, we get
    \begin{equation*}
        P(x)(w) = D\pi(y)[w] + P(x)\bc{D\xi(y)[w]}.
    \end{equation*}
    Using (\ref{eq:tan_of_Dpiy}), this becomes
    \begin{equation*}
        P(x)(w) = D\pi(y)[w] - S_{x,\xi}\bc{D\pi(y)[w]} = \bc{\op{id} - S_{x,\xi}}\bc{D\pi(y)[w]}.
    \end{equation*}
    Since $w \in \R^n$ is arbitrary, we conclude that
    \begin{equation*}
        P(x) = \bc{\op{id} - S_{x,\xi}} \circ D\pi(y). 
    \end{equation*}
\end{proof}

\begin{lem}\label{lem:boundeness_of_Dpi}
    Let $\M \subset \R^n$ be a Riemannian submanifold with uniformly bounded second fundamental form, i.e., $\kappa = \sup_\M\norm*{\mathrm{II}}_{\op{op}} <\infty$. Then there exists a tubular neighborhood $\mathcal{V}$ with the canonical projection $\pi$ and a constant $0 < r_0 < 1 / \kappa$ such that
    \begin{equation*}
        D\pi(y) = \bc{\op{id} - S_{x,\xi}}^{-1} \circ P(x),\quad\text{and }\norm*{D\pi(y)}_{\op{op}} < \frac{1}{1 - r_0\kappa},\quad \forall~ y\in \mathcal{V}.
    \end{equation*}
    In particular, when $\M$ admits a uniform tubular neighborhood, $\mathcal{V}$ can be chosen to be a uniform tubular neighborhood.
\end{lem}
\begin{proof}
    First, by Proposition \ref{prop:formula_of_Dpi}, there exists a tubular neighborhood $\mathcal{V}_0$ with the canonical projection $\pi$ such that for any $y \in \mathcal{V}_0$,
    \begin{equation}\label{eq:Pi_second_Dpi}
        P(x) = \bc{\op{id} - S_{x,\xi}} \circ D\pi(y),
    \end{equation}
    where $y = x + \xi$ is the unique decomposition with $x \in \M$ and $\xi \in N_x\M$. Moreover, by Lemma \ref{lem:bound_shape_operator} and the uniform boundedness of the second fundamental form,
    \begin{equation}\label{eq:bound_shape_operator}
        \norm*{S_{x,\xi}}_{\op{op}} \leq \kappa \norm{\xi}.
    \end{equation}
    Shrink $\mathcal{V}_0$ to
    \begin{equation*}
        \mathcal{V} \defeq \mathcal{V}_0 \cap \bb{x+ \xi \colon \norm{\xi} < r_0},
    \end{equation*}
    where $r_0 > 0$ is a constant satisfying $r_0 < 1/ \kappa$. If $\M$ admits a uniform tubular neighborhood, we may take $\mathcal{V}_0 = \bb{x+ \xi \colon \norm{\xi} < \varepsilon_0}$ for a uniform radius $\varepsilon_0$, and then choose
    \begin{equation*}
        \mathcal{V} \defeq \bb{x+ \xi \colon \norm{\xi} < r_0},\quad r_0 = \min\bb{\varepsilon_0,1/\kappa},
    \end{equation*}
    which is also a uniform tubular neighborhood of $\M$.

    For any $y \in \mathcal{V}$, (\ref{eq:bound_shape_operator}) gives
    \begin{equation*}
        \norm*{S_{x,\xi}}_{\op{op}} \leq r_0\kappa < 1.
    \end{equation*}
    Hence $\op{id} - S_{x,\xi}$ is invertible, and by the Neumann series bound \citep{horn2012matrix},
    \begin{equation*}
        \norm*{(\op{id} - S_{x,\xi})^{-1}}_{\op{op}} \leq \frac{1}{1 - \norm*{S_{x,\xi}}_{\op{op}}} < \frac{1}{1 - r_0\kappa}.
    \end{equation*}
    Therefore, by (\ref{eq:Pi_second_Dpi}),
    \begin{equation*}
        D\pi(y) = \bc{\op{id} - S_{x,\xi}}^{-1} \circ P(x),\quad \forall~ y \in \mathcal{V},
    \end{equation*}
    and moreover,
    \begin{equation*}
        \norm*{D\pi(y)}_{\op{op}} \leq \norm*{(\op{id} - S_{x,\xi})^{-1}}_{\op{op}} \norm*{P(x)}_{\op{op}}
        = \norm*{(\op{id} - S_{x,\xi})^{-1}}_{\op{op}} < \frac{1}{1 - r_0\kappa},
    \end{equation*}
    where $\norm*{P(x)}_{\op{op}} = 1$ since $P(x)$ is an orthogonal projection. 
\end{proof}

\subsection{Boundedness of Orthogonal Projection}\label{appen:boundedness_of_orthogonal_projection}

To bound the extrinsic derivative of $P(x)$, we first need the boundedness of the intrinsic derivative under the second fundamental form uniformly bounded.

\begin{lem}\label{lem:bound_intrin_ortho_proj}
    Let $\M \subset \R^n$ be a Riemannian submanifold with uniformly bounded second fundamental form, i.e., $\kappa = \sup_\M\norm*{\mathrm{II}}_{\op{op}} <\infty$. Let $P(x)$ for $x\in \M$ be the orthogonal projection onto $T_x\M$. Then
    \begin{equation*}
        \sup_{x \in \M} \norm*{\nabla P(x)}_{\op{op}} \leq \kappa.
    \end{equation*}
\end{lem}
\begin{proof}
    First, we clarify the definition of the intrinsic covariant derivative $\nabla P(x)$. For any $w \in \R^n$, $P(x)w \in T_x\M$, so $Pw \in \Gamma(T\M)$. Hence $P \in \Gamma(\M, \op{Hom}(\M \times \R^n, T\M))$. Therefore, for any $X \in \Gamma(T\M)$, the covariant derivative $\nabla_X P \in \Gamma(\M, \op{Hom}(\M \times \R^n, T\M))$ is defined by
    \begin{equation*}
        (\nabla_X P)(x) w \defeq \nabla_X (P(x) w),\quad \forall~ w \in \R^n.
    \end{equation*}
    For $w \in \R^n$, write the decomposition
    \begin{equation}\label{eq:decomp_of_w}
            w = w^\top(x) + w^\perp(x),
    \end{equation}
    where
    \begin{equation*}
        w^\top(x) = P(x)w \in T_x\M,\quad w^\perp(x) = w - w^\top(x) \in N_x\M.
    \end{equation*}

    \parhead{Claim:} $(\nabla_X P)w = S_{w^\perp}(X)$ for any $w \in \R^n$, where $S_{w^\perp}$ is the shape operator along $w^\perp$.

    We now prove the claim. Fix $w \in \R^n$ and let $X \in \Gamma(T\M)$. Since $w$ is constant in $\R^n$, taking the directional derivative on both sides of (\ref{eq:decomp_of_w}) gives
    \begin{equation}\label{eq:diff_decomp_of_w}
        0 = D_X w = D_X w^\top + D_X w^\perp.
    \end{equation}
    For the first term, since $w^\top \in \Gamma(T\M)$, the Gauss's formula \citep[Theorem 8.2]{lee2018introduction} implies
    \begin{equation}\label{eq:diff_tan_of_w}
        D_X w^\top = \nabla_X w^\top + \mathrm{II}(X,w^\top).
    \end{equation}
    For the second term, since $w^\perp \in \Gamma(N\M)$, the Weingarten Equation \citep[Proposition 8.4]{lee2018introduction} gives
    \begin{equation*}
        \bc{D_X w^\perp}^\top = -S_{w^\perp}(X),
    \end{equation*}
    and hence
    \begin{equation}\label{eq:diff_perp_of_w}
        D_X w^\perp = \bc{D_X w^\perp}^\top + \bc{D_X w^\perp}^\perp
        = -S_{w^\perp}(X) + \bc{D_X w^\perp}^\perp.
    \end{equation}
    Combining (\ref{eq:diff_tan_of_w}), (\ref{eq:diff_perp_of_w}) with (\ref{eq:diff_decomp_of_w}), we obtain
    \begin{equation*}
        \nabla_X w^\top + \mathrm{II}(X,w^\top) - S_{w^\perp}(X) + \bc{D_X w^\perp}^\perp = 0.
    \end{equation*}
    The left-hand side is tangent, while the remaining terms are normal, so both components must vanish. In particular,
    \begin{equation*}
        \nabla_X w^\top = S_{w^\perp}(X).
    \end{equation*}
    Therefore,
    \begin{equation}\label{eq:intrin_diff_ortho_shape}
        (\nabla_XP)w = \nabla_X\bc{P w} = \nabla_X w^\top = S_{w^\perp}(X),
    \end{equation}
    which proves the claim.

    Since $\kappa = \sup_\M\norm*{\mathrm{II}}_{\op{op}} <\infty$, Lemma \ref{lem:bound_shape_operator} implies
    \begin{equation*}
        \norm*{S_{w^\perp}}_{\op{op}} \leq \kappa \norm*{w^\perp}.
    \end{equation*}
    Hence, for any $X \in T_x\M$,
    \begin{equation*}
        \norm*{S_{w^\perp}(X)} \leq \norm*{S_{w^\perp}}_{\op{op}} \norm{X} \leq \kappa \norm*{w^\perp}\norm{X} \leq \kappa \norm{w}\norm{X}.
    \end{equation*}
    Using (\ref{eq:intrin_diff_ortho_shape}), we obtain
    \begin{equation*}
        \norm*{\nabla P(x)}_{\op{op}}
        = \sup_{X \neq 0,\, w \neq 0} \frac{\norm*{(\nabla_XP)(x)w}}{\norm{X}\norm{w}}
        = \sup_{X \neq 0,\, w \neq 0} \frac{\norm*{S_{w^\perp}(X)}}{\norm{X}\norm{w}}
        \leq \kappa.
    \end{equation*}
    Taking the supremum over $x \in \M$ completes the proof. 
\end{proof}

\section{More Discussions on Geometric Assumptions}\label{appen:more_discussions_on_geometric_assumptions}

Assumption \ref{assum:geometric_iota} and Assumption \ref{assum:uniform_tubular} concern the geometric properties of $\M \subset \R^n$. In this section, we first present three typical examples satisfying these assumptions (Section \ref{appen:examples_of_manifolds}), and then discuss their relationship (Section \ref{appen:relationship_of_assumption_iota_assum_tubular}). We also provide an extrinsic criterion for Assumption \ref{assum:geometric_iota}; see Section \ref{appen:extrinsic_criterion_for_assumption_ref_assum_geometric_iota}.

\subsection{Examples}\label{appen:examples_of_manifolds}

Here we present three examples of Riemannian submanifolds $\M \subset \R^n$ that satisfy Assumption \ref{assum:geometric_iota} and Assumption \ref{assum:uniform_tubular}.

\begin{enumerate}[label=\arabic{*}.]
    \item Compact case: If the Riemannian submanifold $\M \subset \R^n$ is compact, i.e., closed and bounded, then since the canonical embedding $\iota \colon \M \sto \R^n$ is smooth, it follows that
    \begin{equation*}
        \sup_\M\norm*{\nabla d \iota}_{\op{op}}<\infty,\quad\sup_\M \norm*{\nabla^2 d \iota}_{\op{op}}<\infty.
    \end{equation*}
    Hence, $\M$ satisfies Assumption \ref{assum:geometric_iota}. Moreover, since the tubular radius $\varepsilon(x)$ is smooth on $\M$ (Theorem \ref{thm:tubular_neighborhood}), compactness of $\M$ implies
    \begin{equation*}
        \inf_{x \in \M} \varepsilon(x) = \min_{x \in \M} \varepsilon(x) > 0.
    \end{equation*}
    Therefore, $\M$ also satisfies Assumption \ref{assum:uniform_tubular}.

    \item Graphs: Let $f = (f^1,\ldots,f^{n-m}) \colon \R^m \sto \R^{n-m}$ be smooth and satisfy
    \begin{equation}\label{eq:cal_graph_cond}
        \sup_{x \in \R^m} \norm{Df(x)} < \infty,\quad \sup_{x \in \R^m} \norm{D^2f(x)}_{\op{op}} < \infty,\quad \sup_{x \in \R^m} \norm{D^3f(x)}_{\op{op}} < \infty.
    \end{equation}
    Consider the graph manifold of $f$,
    \begin{equation*}
        \M = \bb{(x,f(x)) \colon x \in \R^m} \subset \R^n.
    \end{equation*}
    Let $\iota \colon \M \hookrightarrow \R^n$ be the canonical embedding, and equip $\M$ with the induced Riemannian metric $g$ and Levi--Civita connection $\nabla$. Then $\M$ satisfies Assumption \ref{assum:geometric_iota} (Theorem \ref{thm:cal_graph_iota_thm}) and Assumption \ref{assum:uniform_tubular} (Theorem \ref{thm:cal_graph_tubular}).

    \item Level sets: Let $F \colon \R^n \sto \R^k$ be a smooth function ($k < n$), and let $\M = F^{-1}(0)$. Assume that $\rank DF(x) = k$ for all $x \in \M$. Then, by the constant rank theorem \citep[Corollary 5.14]{lee2012smooth}, $\M$ is an $m$-dimensional smooth submanifold of $\R^n$ with $m = n-k$. Assume further that
    \begin{equation}\label{eq:cal_level_cond}
        \inf_{x \in \M}\sigma_{\min}(DF(x)) > 0,\quad \sup_{x\in \M} \norm*{D^2F(x)}_{\op{op}} < \infty,\quad \sup_{x \in \M}\norm*{D^3F(x)}_{\op{op}} < \infty,
    \end{equation}
    where $\sigma_{\min}$ denotes the smallest singular value. Let $\iota \colon \M \hookrightarrow \R^n$ be the canonical embedding, and equip $\M$ with the induced Riemannian structure and Levi--Civita connection $\nabla$.

    Under (\ref{eq:cal_level_cond}), $\M$ satisfies Assumption \ref{assum:geometric_iota}; see Theorem \ref{thm:cal_level_iota}. However, establishing the existence of a uniform tubular neighborhood (Assumption \ref{assum:uniform_tubular}) requires additional conditions. For $k = 1$, \citet{platt2025lower} showed that $\M$ admits a uniform tubular neighborhood if the $L_1$-norm of $DF$ is bounded below on $\M$ and $\norm{D^2F}$ is bounded on an open, convex neighborhood $\mathcal{V}$ of $\M$, rather than only on $\M$ itself. For general $k$, we extend this type of result; see Theorem \ref{thm:cal_level_tubular_thm}.

\end{enumerate}

\begin{thm}\label{thm:cal_graph_iota_thm}
    Let $f \colon \R^m \sto \R^{n-m}$ be a smooth function, and let $\M = \bb{(x,f(x)) \colon x \in \R^m}$ be its graph manifold. Equip $\M \subset \R^n$ with the induced Riemannian structure given by the canonical embedding $\iota \colon \M \hookrightarrow \R^n$. If $f$ satisfies (\ref{eq:cal_graph_cond}), then
    \begin{equation*}
        \sup_{x \in \M}\norm*{\nabla d \iota(x)}_{\op{op}}<\infty,\quad\sup_{x \in \M} \norm*{\nabla^2 d \iota(x)}_{\op{op}}<\infty.
    \end{equation*}
\end{thm}
\begin{proof}
    We prove the theorem in the following four steps.
    \begin{itemize}
        \item Step $1$. Riemannian metric: Equip $\M$ with the canonical coordinate chart $x = (x^1,\ldots,x^m) \colon \M \sto \R^m$, defined by
        \begin{equation*}
            x^i(x,f(x)) = x^i.
        \end{equation*}
        Let $\bb{\frac{\partial}{\partial x^i}}_{i=1}^m$ be the associated basis of $T_x\M$. Let $y = (y^1,\ldots,y^n)$ be the canonical coordinates on $\R^n$, with coordinate vector fields $\bb{\frac{\partial}{\partial y^k}}_{k=1}^n$. For any $h \in C^\infty(\R^n)$,
        \begin{align*}
            d\iota \bc{\frac{\partial}{\partial x^i}}(h) &= \frac{\partial}{\partial x^i} (h \circ \iota) = \frac{\partial}{\partial x^i} h(x^1,\ldots,x^m,f^1(x),\ldots, f^{n-m}(x)) \\
            &= \frac{\partial}{\partial y^i}h + \sum_{k=1}^{n-m} \frac{\partial f^k}{\partial x^i} \frac{\partial}{\partial y^{k+m}}h,
        \end{align*}
        which implies that
        \begin{equation}\label{eq:dl_embed_vec}
            d\iota \bc{\frac{\partial}{\partial x^i}} = \frac{\partial}{\partial y^i} + \sum_{k=1}^{n-m} \frac{\partial f^k}{\partial x^i} \frac{\partial}{\partial y^{k+m}}.
        \end{equation}
        Let $g$ be the induced Riemannian metric on $\M$. Then
        \begin{equation}\label{eq:metric_formula_graph}
            g_{ij} = \inn{d\iota \bc{\frac{\partial}{\partial x^i}},d\iota \bc{\frac{\partial}{\partial x^j}}} = \delta_{ij} + \sum_{k=1}^{n-m} \frac{\partial f^k}{\partial x^i}\frac{\partial f^k}{\partial x^j},
        \end{equation}
        i.e., in this coordinate system,
        \begin{equation*}
            g(x) = I_m + (Df(x))^T(Df(x)).
        \end{equation*}
        Let $C_1 \defeq \sup_{x} \norm{Df(x)} < \infty$. Then
        \begin{equation}\label{eq:0_g_graph}
            I_m \preceq g(x) \preceq \bc{1+C_1^2}I_m,
        \end{equation}
        and hence
        \begin{equation*}
            1 \leq \norm*{g(x)}_{\op{op}} \leq 1 + C_1^2,\quad \norm*{g(x)^{-1}}_{\op{op}} \leq 1.
        \end{equation*}
        Let $g^{-1} = (g^{ij})$. Since $\norm*{g(x)^{-1}}_{\op{F}}\leq \sqrt{m}\norm*{g(x)^{-1}}_{\op{op}}$ \citep{horn2012matrix},
        \begin{equation}\label{eq:0_g_inv_graph}
            \sup_x \abs{g^{ij}(x)} \leq \sqrt{m}.
        \end{equation}

        Moreover, by (\ref{eq:metric_formula_graph}),
        \begin{equation*}
            \partial_{\alpha} g_{ij} = \sum_{k=1}^{n-m} \bc{\partial_{\alpha i} f^k \partial_j f^k + \partial_i f^k\partial_{\alpha j} f^k},
        \end{equation*}
        where we write $\partial_\alpha = \frac{\partial}{\partial x^\alpha}$ for simplicity. Let $C_2 \defeq \sup_{x} \norm{D^2f(x)}_{\op{op}} < \infty$. Then
        \begin{equation}\label{eq:1_g_graph}
            \sup_x \abs{\partial_{\alpha} g_{ij}(x)} \leq 2C_1C_2.
        \end{equation}
        Furthermore, using
        \begin{equation*}
            \partial_{\alpha \beta} g_{ij} = \sum_{k=1}^{n-m}\bc{\partial_{\alpha \beta i} f^k \partial_{j} f^k + \partial_{\alpha i} f^k \partial_{\beta j} f^k + \partial_{\beta i} f^k \partial_{\alpha j} f^k +\partial_{i} f^k \partial_{\alpha \beta j} f^k},
        \end{equation*}
        and letting $C_3 \defeq \sup_{x} \norm{D^3f(x)}_{\op{op}} < \infty$, we similarly obtain
        \begin{equation}\label{eq:2_g_graph}
            \sup_x \abs{\partial_{\alpha \beta} g_{ij}(x)} \leq 2C_1C_3 + 2C_2^2.
        \end{equation}
        Finally, for $g^{-1} = (g^{ij})$, we have
        \begin{equation*}
            \partial_\ell g^{ij} = - \sum_{\alpha,\beta = 1}^m g^{i\alpha} \bc{\partial_\ell g_{\alpha \beta}} g^{\beta j}.
        \end{equation*}
        Therefore, by (\ref{eq:0_g_inv_graph}) and (\ref{eq:1_g_graph}),
        \begin{equation}\label{eq:1_g_inv_graph}
            \sup_x \abs{\partial_\ell g^{ij}(x)} \leq 2C_1C_2m^3.
        \end{equation}

        \item Step $2$. Christoffel coefficients: By definition,
        \begin{equation*}
            \Gamma^k_{ij} = \frac{1}{2} \sum_{\alpha=1}^m g^{k \alpha}\left(\partial_i g_{j \alpha}+\partial_j g_{i \alpha}-\partial_\alpha g_{i j}\right).
        \end{equation*}
        Using (\ref{eq:0_g_inv_graph}) and (\ref{eq:1_g_graph}), we obtain
        \begin{equation}\label{eq:0_christoffel_graph}
            \sup_x \abs{\Gamma^k_{ij}(x)} \leq 3C_1C_2m^{3/2}.
        \end{equation}
        Differentiating further,
        \begin{equation*}
            \partial_\ell \Gamma^k_{ij} = \frac{1}{2} \sum_{\alpha=1}^m \bc{ \bc{\partial_\ell g^{k \alpha}}\left(\partial_i g_{j \alpha}+\partial_j g_{i \alpha}-\partial_\alpha g_{i j}\right) + g^{k \alpha}\left(\partial_{\ell i} g_{j \alpha}+\partial_{\ell j} g_{i \alpha}-\partial_{\ell \alpha} g_{i j}\right)}.
        \end{equation*}
        Bounding the $g^{-1}$-related terms by (\ref{eq:0_g_inv_graph}) and (\ref{eq:1_g_inv_graph}), and the $g$-related terms by (\ref{eq:1_g_graph}) and (\ref{eq:2_g_graph}), implies
        \begin{equation}\label{eq:1_christoffel_graph}
            \sup_x \abs{\partial_\ell \Gamma^k_{ij}(x)} \leq 6C_1^2C_2^2m^4 + \bc{3C_1C_3 + 3C_2^2}m^{3/2}.
        \end{equation}

        \item Step $3$. Boundedness of $\nabla d \iota$: By (\ref{eq:def_nabla_d_iota}),
        \begin{equation}\label{eq:cal_diota_ij_graph}
            (\nabla d \iota) (\partial_i,\partial_j) = D_{d\iota(\partial_j)}\bc{d \iota (\partial_i)} - d\iota \bc{\nabla_{\partial_j} \partial_i},
        \end{equation}
        where we write $\partial_i = \frac{\partial}{\partial x^i}$. For convenience, denote $\tilde{\partial}_k = \frac{\partial}{\partial y^k}$. Then, by (\ref{eq:dl_embed_vec}),
        \begin{equation}\label{eq:cal_diota_ext_graph}
            D_{d\iota(\partial_j)}\bc{d \iota (\partial_i)} = D_{\tilde{\partial}_j}\bc{d \iota (\partial_i)} + \sum_{k=1}^{n-m}\partial_j f^k \, D_{\tilde{\partial}_{k+m}}\bc{d \iota (\partial_i)} = \sum_{\ell = 1}^{n-m} \partial_{ij}f^\ell \tilde{\partial}_{\ell+m},
        \end{equation}
        where $D_{\tilde{\partial}_j}\tilde{\partial}_k=0$. Here we use the natural extension of $f$ to $\R^n$ by setting $y^i=x^i$ for $i=1,\ldots,m$ and making it independent of $y^{m+1},\ldots,y^n$, so that $\tilde{\partial}_{ij}f^k = \partial_{ij}f^k$ for $i,j \leq m$ and $\tilde{\partial}_{ij}f^k = 0$ if $i>m$ or $j>m$.

        For the second term in (\ref{eq:cal_diota_ij_graph}), note that $\nabla_{\partial_j} \partial_i = \sum_\alpha \Gamma^\alpha_{ij}\partial_\alpha$, and thus
        \begin{equation}\label{eq:cal_diota_intr_graph}
            d\iota \bc{\nabla_{\partial_j} \partial_i} = \sum_{\alpha = 1}^m \Gamma^\alpha_{ij} \tilde{\partial}_\alpha + \sum_{\alpha=1}^m\sum_{\ell = 1}^{n-m} \Gamma^\alpha_{ij} \partial_\alpha f^\ell \tilde{\partial}_{\ell+m}.
        \end{equation}
        Combining (\ref{eq:cal_diota_ext_graph}), (\ref{eq:cal_diota_intr_graph}), and (\ref{eq:cal_diota_ij_graph}) implies
        \begin{equation}\label{eq:cal_diota_ij_graph_result}
            (\nabla d \iota) (\partial_i,\partial_j) = - \sum_{\alpha = 1}^m \Gamma^\alpha_{ij} \tilde{\partial}_\alpha + \sum_{\ell = 1}^{n-m} \bc{\partial_{ij}f^\ell - \sum_{\alpha=1}^m \Gamma^\alpha_{ij} \partial_\alpha f^\ell}\tilde{\partial}_{\ell+m}.
        \end{equation}
        By (\ref{eq:0_christoffel_graph}) and the assumptions on $f$, there exists a $C_{\iota} = C_{\iota}(C_1,C_2,m) > 0$ such that
        \begin{equation}\label{eq:cal_diota_ij_graph_result_bound}
            \norm{(\nabla d \iota(x)) (\partial_i,\partial_j)} \leq C_{\iota} < \infty.
        \end{equation}
        Since $g(x)$ is uniformly bounded below by (\ref{eq:0_g_graph}), Lemma \ref{lem:bound_operator} implies that there exists a $C_{\iota,1} = C_{\iota,1}(C_1,C_2,m) > 0$ such that
        \begin{equation*}
            \norm{(\nabla d \iota(x)) (u,v)} \leq C_{\iota,1} \norm{u}\norm{v},\quad \forall~ u,v \in T_x\M.
        \end{equation*}
        Hence $\nabla d \iota(x)$ is a bounded bilinear map on $T_x\M \times T_x\M$, and
        \begin{equation}\label{eq:cal_diota_graph_result}
            \norm*{\nabla d \iota(x)}_{\op{op}} = \sup_{u,v \neq 0} \frac{\norm{(\nabla d \iota(x)) (u,v)}}{\norm{u}\norm{v}} \leq C_{\iota,1}, \quad \Rightarrow \quad \sup_{x \in \M}\norm*{\nabla d \iota(x)}_{\op{op}} \leq C_{\iota,1} < \infty.
        \end{equation}

        \item Step $4$. Boundedness of $\nabla^2 d\iota$: Let $A = \nabla d\iota$. By (\ref{eq:def_nabla2_d_iota}),
        \begin{equation}\label{eq:cal_2diota_ij_graph}
            (\nabla A) (\partial_i,\partial_j,\partial_\alpha) = D_{d\iota(\partial_\alpha)}\bc{A(\partial_i,\partial_j)} - A(\nabla_{\partial_\alpha}\partial_i,\partial_j) - A(\partial_i,\nabla_{\partial_\alpha}\partial_j).
        \end{equation}
        Using (\ref{eq:cal_diota_ij_graph_result}), write
        \begin{equation*}
            A_{ij} = A(\partial_i,\partial_j) = (\nabla d \iota) (\partial_i,\partial_j) = - \sum_{\beta = 1}^m \Gamma^\beta_{ij} \tilde{\partial}_\beta + \sum_{\ell = 1}^{n-m} \bc{\partial_{ij}f^\ell - \sum_{\beta=1}^m \Gamma^\beta_{ij} \partial_\beta f^\ell}\tilde{\partial}_{\ell+m}.
        \end{equation*}
        For the first term on the right-hand side of (\ref{eq:cal_2diota_ij_graph}), we compute
        \begin{equation}\label{eq:cal_2diota_ext_graph}
            \begin{aligned}
                D_{d\iota(\partial_\alpha)}\bc{A(\partial_i,\partial_j)} &= - \sum_{\beta = 1}^m \partial_\alpha\Gamma^\beta_{ij} \tilde{\partial}_\beta \\
                &\quad+ \sum_{\ell = 1}^{n-m} \bc{\partial_{ij\alpha}f^\ell - \sum_{\beta=1}^m \bc{\partial_\alpha\Gamma^\beta_{ij} \partial_\beta f^\ell + \Gamma^\beta_{ij} \partial_{\alpha\beta} f^\ell}}\tilde{\partial}_{\ell+m},
            \end{aligned}
        \end{equation}
        where we extend $\Gamma^k_{ij}$ to $\R^n$ so that $\tilde{\partial}_\alpha \Gamma^k_{ij} = \partial_\alpha \Gamma^k_{ij}$ for $\alpha = 1,\ldots,m$ and $\tilde{\partial}_\alpha \Gamma^k_{ij} = 0$ for $\alpha = m+1,\ldots,n$. For the remaining two terms,
        \begin{equation}\label{eq:cal_2diota_intr_graph}
            A(\nabla_{\partial_\alpha}\partial_i,\partial_j) = \sum_{\beta = 1}^m \Gamma^\beta_{\alpha i} A_{\beta j}, \quad A(\partial_i,\nabla_{\partial_\alpha}\partial_j) = \sum_{\beta = 1}^m\Gamma^\beta_{\alpha j}A_{i\beta}.
        \end{equation}
        Substituting (\ref{eq:cal_2diota_ext_graph}) and (\ref{eq:cal_2diota_intr_graph}) into (\ref{eq:cal_2diota_ij_graph}), and using (\ref{eq:0_christoffel_graph}), (\ref{eq:1_christoffel_graph}), (\ref{eq:cal_diota_ij_graph_result_bound}), together with (\ref{eq:cal_graph_cond}), we obtain that there exists $C^\prime_{\iota} = C^\prime_{\iota}(C_1,C_2,C_3,m) > 0$, independent of $x$, such that
        \begin{equation*}
            \norm*{(\nabla^2 d\iota(x))(\partial_i,\partial_j,\partial_\alpha)} \leq C^\prime_\iota < \infty.
        \end{equation*}
        As in Step $3$, the uniform lower bound (\ref{eq:0_g_graph}) and Lemma \ref{lem:bound_operator} imply that there exists $C_{\iota,2} = C_{\iota,2}(C_1,C_2,C_3,m) > 0$ such that
        \begin{equation*}
            \norm*{(\nabla^2 d\iota(x))(u,v,w)} \leq C_{\iota,2} \norm{u}\norm{v}\norm{w},\quad \forall~ u,v,w \in T_x\M.
        \end{equation*}
        Therefore,
        \begin{equation*}
            \sup_{x \in \M}\norm*{\nabla^2 d\iota(x)}_{\op{op}} = \sup_{x \in \M}\sup_{u,v,w \neq 0} \frac{\norm*{(\nabla^2 d\iota(x))(u,v,w)}}{\norm{u}\norm{v}\norm{w}}
            \leq C_{\iota,2}. 
        \end{equation*}
    \end{itemize}
\end{proof}

\begin{thm}\label{thm:cal_graph_tubular}
    Let $f \colon \R^m \sto \R^{n-m}$ be a smooth function, and let $\M = \bb{(x,f(x)) \colon x \in \R^m}$ be its graph manifold equipped with the induced Riemannian structure. If $f$ satisfies
    \begin{equation*}
        \sup_{x \in \R^m} \norm{Df(x)} < \infty,\quad
        \sup_{x \in \R^m} \norm{D^2f(x)}_{\op{op}} < \infty,
    \end{equation*}
    then there exists a uniform tubular neighborhood $\mathcal{V}_{r_0}$ of $\M$ in $\R^n$ with radius $r_0$.
\end{thm}

\begin{proof}
    Let $\Phi \colon \R^m \sto \R^n$ be defined by $\Phi(x) = (x,f(x))$. Then $\Img \Phi = \M$ and, by \citet[Proposition 5.37]{lee2012smooth},
    \begin{equation*}
        T_{\Phi(x)}\M \defeq \bb{(v,D\Phi(x)v) \colon v\in \R^m} \subset \R^n.
    \end{equation*}
    Consequently, for any $x \in \R^m$ and $\eta \in \R^k$ ($k = n-m$), a convenient parametrization of the normal space is
    \begin{equation*}
        \xi(x,\eta) = \bc{-Df(x)^T \eta,\eta} \in N_{\Phi(x)}\M.
    \end{equation*}
    In particular,
    \begin{equation}\label{eq:cal_graph_bound_normal_vec}
        \norm{\eta}^2 \leq \norm*{\xi(x,\eta)}^2
        = \norm{\eta}^2 + \eta^T Df(x) Df(x)^T\eta
        \leq (1+C_1^2)\norm{\eta}^2,
    \end{equation}
    where $C_1 \defeq \sup_{x\in \M}\norm{Df(x)}_{\op{op}} < \infty$. It is also straightforward to verify that
    \begin{equation*}
        (\Phi, \xi) \colon \R^n \sto N\M
    \end{equation*}
    defined by $(\Phi, \xi)(x,\eta) = (\Phi(x),\xi(x,\eta))$ is a diffeomorphism; hence $(x,\eta)$ provides a global coordinate system on $N\M$. In these coordinates, the tubular map $T \colon N\M \sto \R^n$ is
    \begin{equation}\label{eq:cal_graph_para_tubular_F}
        T(x,\eta) = \Phi(x) + \xi(x,\eta) = \bc{x - Df(x)^T \eta,f(x)+\eta}.
    \end{equation}
    To construct a uniform tubular neighborhood, we show that $T$ is a local diffeomorphism and then establish injectivity on a uniform neighborhood of the zero section.

    \begin{itemize}
        \item Step $1$. Local diffeomorphism: By the inverse function theorem \citep[Theorem 4.5]{lee2012smooth}, it suffices to show that $DT(x,\eta)$ is invertible for the points under consideration. From (\ref{eq:cal_graph_para_tubular_F}), the differential $DT(x,\eta) \colon \R^m \times \R^k \sto \R^m \times \R^k$ is
        \begin{equation*}
            DT(x,\eta) = \bc{
                \begin{array}{cc}
                    I_m- \sum_{\alpha = 1}^{k}\eta^\alpha D^2 f^\alpha(x) & -Df(x)^T \\
                    Df(x) & I_k
                \end{array}
            }.
        \end{equation*}
        By the Schur complement \citep{horn2012matrix}, $DT(x,\eta)$ is invertible if and only if
        \begin{equation*}
            S(x,\eta) = I_m +Df(x)^T Df(x) - \sum_{\alpha = 1}^{k}\eta^\alpha D^2 f^\alpha(x)
        \end{equation*}
        is invertible. For any $v \in \R^m$,
        \begin{align*}
            v^T S(x,\eta) v
            &= \norm{v}^2 + \norm*{Df(x)v}^2 - \sum_{\alpha = 1}^{k}\eta^\alpha v^T D^2 f^\alpha(x)v \\
            &\geq \norm{v}^2 - \norm{\eta} \norm*{D^2f(x)[v,v]}
            \geq (1- C_2\norm{\eta})\norm{v}^2,
        \end{align*}
        where $C_2 \defeq \sup_{x \in \M}\norm*{D^2f(x)}_{\op{op}} < \infty$. Hence, if $\norm{\eta} < 1/C_2$, then $S(x,\eta)$ is positive definite and therefore invertible. It follows that $DT(x,\eta)$ is invertible for all
        \begin{equation*}
            (x,\eta) \in \mathcal{U}_{1/C_2} \defeq \bb{(x,\eta) \colon \norm{\eta} < 1 / C_2},
        \end{equation*}
        i.e., $T$ is a local diffeomorphism on $\mathcal{U}_{1/C_2}$.

        \item Step $2$. Global injectivity: Let $(u,v) \in \R^n$. There exists $(x,\eta) \in \R^n$ such that $T(x,\eta) = (u,v)$ if and only if
        \begin{equation*}
            u = x - Df(x)^T\eta,\quad v = f(x)+\eta,
        \end{equation*}
        equivalently,
        \begin{equation*}
            x - u + Df(x)^T(f(x) - v) = 0,\quad \eta = v - f(x).
        \end{equation*}
        Fix $(u,v) \in \Img T$ and define
        \begin{equation*}
            \varphi_{u,v}(y) = \frac{1}{2}\bc{\norm{y-u}^2 + \norm{f(y)-v}^2}.
        \end{equation*}
        Since $D\varphi_{u,v}(y) = (y-u) + Df(y)^T(f(y) - v)$, we have $(x,\eta) \in T^{-1}(u,v)$ if and only if
        \begin{equation*}
            D\varphi_{u,v}(x) = 0,\quad v = \eta + f(x).
        \end{equation*}
        Suppose that $T(x,\eta) = T(x^\prime,\eta^\prime) = (u,v)$, so that $x$ and $x^\prime$ are critical points of $\varphi_{u,v}$. We show that this cannot occur provided
        \begin{equation*}
            \norm{\eta},~\norm*{\eta^\prime} < r_0 \defeq \frac{1}{C_2(1+2C_1^2)}.
        \end{equation*}

        First, note that
        \begin{equation*}
            \norm*{x - x^\prime} = \norm*{Df(x)^T \eta - Df(x)^T \eta^\prime} < 2C_1r_0.
        \end{equation*}
        Let $y_t = tx + (1-t)x^\prime$ for $t \in [0,1]$. Then
        \begin{equation}\label{eq:cal_graph_boud_v_yt}
            \begin{aligned}
                \norm*{v - f(y_t)}
                &\leq \norm*{v - f(x)} + \norm*{f(x) - f(y_t)} \\
                &\leq \norm{\eta} + C_1\norm{x - y_t}
                < r_0 + C_1\norm{x - x^\prime}
                < (1+2C_1^2)r_0.
            \end{aligned}
        \end{equation}

        Moreover, since
        \begin{equation*}
            D^2\varphi_{u,v}(y) = I_m + Df(y)^T Df(y) + \sum_{\alpha = 1}^k \bc{f^\alpha(y) - v^\alpha}D^2f^\alpha(y),
        \end{equation*}
        the same estimate as in Step $1$ yields, for any $w \in \R^m$,
        \begin{equation*}
            w^T D^2\varphi_{u,v}(y_t) w
            \geq
            \bc{1 - C_2 \norm{f(y_t) - v}}\norm{w}^2
            >
            \bc{1 - C_2(1+2C_1^2)r_0}\norm{w}^2,
        \end{equation*}
        where we used (\ref{eq:cal_graph_boud_v_yt}). Therefore,
        \begin{align*}
            \bc{D\varphi_{u,v}(x)- D\varphi_{u,v}(x^\prime)}[x - x^\prime]
            &= \int_0^1 \bc{x - x^\prime}^T D^2\varphi_{u,v}(y_t) \bc{x - x^\prime} ~\dx{t} \\
            &>
            \bc{1 - C_2(1+2C_1^2)r_0}\norm*{x - x^\prime}^2
            >
            0,
        \end{align*}
        which contradicts $D\varphi_{u,v}(x) = D\varphi_{u,v}(x^\prime) = 0$.

    \end{itemize}

    Combining Steps $1$ and $2$, define
    \begin{equation*}
        \mathcal{U}_{r_0} = \bb{(x,\eta) \colon \norm{\eta} < r_0},\quad
        r_0 =  \frac{1}{C_2(1+2C_1^2)} < \frac{1}{C_2}.
    \end{equation*}
    Then $T$ is a local diffeomorphism and injective on $\mathcal{U}_{r_0}$. Hence $T \colon \mathcal{U}_{r_0} \sto \mathcal{V}_{r_0} = T(\mathcal{U}_{r_0}) \subset \R^n$ is a diffeomorphism, i.e., $\mathcal{V}_{r_0}$ is a uniform tubular neighborhood of $\M$ with radius $r_0$. 
\end{proof}

\begin{thm}\label{thm:cal_level_iota}
    Let $F \colon \R^n \sto \R^k$ be a smooth function ($k < n$) and let $\M = F^{-1}(0)$. Assume that $\rank DF(x) = k$ for all $x \in \M$ and that $F$ satisfies (\ref{eq:cal_level_cond}). Equip $\M$ with the induced Riemannian structure given by the canonical embedding $\iota \colon \M \hookrightarrow \R^n$. Then
    \begin{equation*}
        \sup_{x \in \M}\norm*{\nabla d \iota(x)}_{\op{op}}<\infty,\quad
        \sup_{x \in \M} \norm*{\nabla^2 d \iota(x)}_{\op{op}}<\infty.
    \end{equation*}
\end{thm}
\begin{proof}
    By Proposition \ref{lem:assum_to_bound_second_form}, it suffices to prove
    \begin{equation*}
        \sup_{x \in \M}\norm*{\mathrm{II}_x}_{\op{op}} <\infty,\quad \sup_{x \in \M}\norm*{\nabla\mathrm{II}_x}_{\op{op}} < \infty.
    \end{equation*}
    We prove the theorem in the following three steps.

    \begin{itemize}
        \item Step $1$. Right inverse: Let $G(x) \defeq DF(x)$ for $x \in \R^n$. By the SVD,
        \begin{equation*}
            G(x) = R(x)\Sigma(x) Q(x),\quad \Sigma(x) = \begin{pmatrix} \sigma_1(x) & \cdots & 0 & \cdots & 0 \\ ~ &\ddots & ~ &~ &~\\ 0 & \cdots & \sigma_k(x) & \cdots & 0\end{pmatrix},
        \end{equation*}
        where $\sigma_1(x) \geq \cdots \geq \sigma_k(x) = \sigma_{\min}(G(x)) > 0$ since $G(x)$ has full rank on $\M$. Consider the Moore--Penrose right inverse \citep{Golan2012Linear}
        \begin{equation*}
            G(x)^\dagger \defeq G(x)^T\bc{G(x)G(x)^T}^{-1} = Q^T \begin{pmatrix} \sigma_1^{-1} & \cdots & 0\\ \vdots & \ddots & \vdots \\ 0 & \cdots & \sigma^{-1}_k \\ \vdots &~ &\vdots \\ 0 & \cdots & 0\end{pmatrix} R^T.
        \end{equation*}
        We record two consequences.
        \begin{enumerate}[label=(\alph{*})]
            \item By assumption, $c_0 \defeq \inf_{x \in \M} \sigma_k(x) > 0$, and hence
            \begin{equation}\label{eq:bound_MP_level}
                \sup_{x \in \M}\norm{G(x)^\dagger}_{\op{op}} \leq \sup_{x \in \M}\sigma_k(x)^{-1} \leq c_0^{-1}.
            \end{equation}

            \item Since $\M = F^{-1}(0)$, \citet[Proposition 5.38]{lee2012smooth} implies $T_x\M = \ker G(x)$ and thus $N_x\M = (\ker G(x))^\perp$. By Lemma \ref{lem:Moore_Penrose},
            \begin{equation*}
                P_x = G(x)^\dagger G(x)
            \end{equation*}
            is the orthogonal projection onto $N_x\M$.
        \end{enumerate}

        \item Step $2$. Boundedness of $\mathrm{II}$: Write $F = (F^1,\ldots,F^k)$. For each $\alpha \in \{1,\ldots,k\}$, let $f^\alpha = F^\alpha|_\M \in C^\infty(\M)$. Since $F|_\M = 0$, we have $f^\alpha \equiv 0$. By Lemma \ref{lem:ext_hess_intr_hess}, for any $X,Y \in \Gamma(T\M)$,
        \begin{equation*}
            0 = \nabla^2 f^\alpha(X,Y) = D^2F^\alpha[X,Y] + DF^\alpha[\mathrm{II}(X,Y)].
        \end{equation*}
        Hence, $DF^\alpha[\mathrm{II}(X,Y)] = -D^2F^\alpha[X,Y]$ for $\alpha = 1,\ldots,k$, and therefore, for any $x \in \M$,
        \begin{equation*}
            DF(x)[\mathrm{II}_x(u,v)] = - D^2F(x)[u,v],\quad \forall~u,v \in T_x\M.
        \end{equation*}
        Since $\mathrm{II}_x(u,v) \in N_x\M$, the projection property of $G(x)^\dagger$ from Step $1$ gives
        \begin{equation*}
            \mathrm{II}_x(u,v) = -G(x)^\dagger D^2F(x)[u,v].
        \end{equation*}
        Using (\ref{eq:bound_MP_level}) and $C_2 \defeq \sup_{x \in \M}\norm*{D^2F(x)}_{\op{op}} < \infty$, we obtain
        \begin{equation*}
            \norm*{\mathrm{II}_x}_{\op{op}} \leq \norm*{G(x)^\dagger}_{\op{op}}\norm*{D^2F(x)}_{\op{op}} \leq \frac{C_2}{c_0},
        \end{equation*}
        and hence
        \begin{equation}\label{eq:bound_second_level}
            \sup_{x \in \M}\norm*{\mathrm{II}_x}_{\op{op}} \leq \frac{C_2}{c_0} < \infty.
        \end{equation}

        \item Step $3$. Boundedness of $\nabla \mathrm{II}$: Fix $x \in \M$ and $u,v,w \in T_x\M$. Choose $U,V,W \in \Gamma(T\M)$ such that
        \begin{equation}\label{eq:cal_setting_level}
            U_x=u,\quad V_x=v,\quad W_x = w,\quad \nabla_W U|_x = \nabla_W V|_x = 0.
        \end{equation}
        Then $(\nabla \mathrm{II})(U,V,W)(x)$ is independent of the choice of $U,V,W$ by Lemma \ref{lem:tensor_nabla_second}; we denote it by $(\nabla \mathrm{II}_x)(u,v,w)$.

        Let $N \defeq \mathrm{II}(U,V) \in \Gamma(N\M)$. By Step $2$, $DF[N] = - D^2F[U,V]$. Taking the directional derivative along $W$ and applying the chain rule gives
        \begin{equation*}
            D^2F[N,W] + DF[D_WN] = - D^3F[U,V,W] - D^2F[D_WU,V] - D^2F[U,D_WV],
        \end{equation*}
        that is,
        \begin{equation}\label{eq:cal_ext_DF_level}
            DF[D_WN] = - D^3F[U,V,W] - D^2F[D_WU,V] - D^2F[U,D_WV] - D^2F[N,W].
        \end{equation}
        For the left-hand side, since $(D_WN|_x)^\top \in T_x\M = \ker DF(x)$, we have
        \begin{equation}\label{eq:cal_left_ext_DF_1}
            DF(x)\bj{D_WN|_x} = DF(x)\bj{(D_WN|_x)^\perp}.
        \end{equation}
        Moreover, by (\ref{eq:cal_setting_level}) and the definition (\ref{eq:def_nabla2_second}),
        \begin{equation}\label{eq:cal_left_ext_DF_2}
            (D_WN|_x)^\perp = (\nabla \mathrm{II}_x)(u,v,w).
        \end{equation}
        On the right-hand side of (\ref{eq:cal_ext_DF_level}), by the Gauss formula \citep{lee2018introduction} and (\ref{eq:cal_setting_level}),
        \begin{equation}\label{eq:cal_right_ext_DF}
            \begin{aligned}
                D_WU|_x &= \nabla_WU|_x+\mathrm{II}_x(u,w) = \mathrm{II}_x(u,w),\\ 
                D_WV|_x &= \nabla_WV|_x+\mathrm{II}_x(v,w) = \mathrm{II}_x(v,w).
            \end{aligned}
        \end{equation}
        Combining (\ref{eq:cal_left_ext_DF_1}), (\ref{eq:cal_left_ext_DF_2}), (\ref{eq:cal_right_ext_DF}), and (\ref{eq:cal_ext_DF_level}) implies
        \begin{align*}
            DF(x)[(\nabla \mathrm{II}_x)(u,v,w)]
            &=  - D^3F(x)[u,v,w] -  D^2F(x)[\mathrm{II}_x(u,w),v] \\
            &\quad - D^2F(x)[u,\mathrm{II}_x(v,w)] - D^2F(x)[\mathrm{II}_x(u,v),w].
        \end{align*}
        Since $(\nabla \mathrm{II}_x)(u,v,w) \in N_x\M$, the projection property of $G(x)^\dagger$ from Step $1$ yields
        \begin{align*}
            (\nabla \mathrm{II}_x)(u,v,w)
            &= -G(x)^\dagger \Big(D^3F(x)[u,v,w] +  D^2F(x)[\mathrm{II}_x(u,w),v] \\
            &\quad +D^2F(x)[u,\mathrm{II}_x(v,w)] + D^2F(x)[\mathrm{II}_x(u,v),w]\Big).
        \end{align*}
        Using (\ref{eq:bound_MP_level}), (\ref{eq:bound_second_level}), and $C_2 = \sup_{\M}\norm*{D^2F}_{\op{op}} < \infty$ and $C_3 \defeq \sup_{\M}\norm*{D^3F}_{\op{op}} < \infty$, we conclude that
        \begin{equation*}
            \sup_{x \in \M}\norm*{\nabla \mathrm{II}_x}_{\op{op}} \leq \frac{C_3}{c_0} + \frac{3C_2^2}{c_0^2} < \infty. 
        \end{equation*}
    \end{itemize}
\end{proof}

\begin{thm}\label{thm:cal_level_tubular_thm}
    Let $F \colon \R^n \sto \R^k$ be a smooth function ($k < n$) and $\M = F^{-1}(0)$. Assume that for any $x \in \M$, $\rank DF(x) = k$ and that $F$ satisfies
    \begin{equation*}
        \inf_{x \in \M}\sigma_{\min}(DF(x)) > 0,\quad
        \sup_{y\in \mathcal{V}} \norm*{D^2F(y)}_{\op{op}} < \infty,
    \end{equation*}
    where $\mathcal{V}$ is an open and convex neighborhood of $\M$. Then $\M$ admits a uniform tubular neighborhood.
\end{thm}

\begin{proof}
    Since $\M$ is closed and without boundary by the implicit function theorem, it suffices, by Theorem \ref{thm:reach_to_tubular}, to show that
    \begin{equation}\label{eq:cal_level_def_r0_reach}
        r_0 \defeq \sup\bb{\varepsilon > 0 \colon \forall~y ~\text{s.t.}~ d(y,\M) < \varepsilon,~\exists~!~x \in \M~\text{s.t.}~ \norm{x - y} = d(y,\M)} > 0.
    \end{equation}
    Fix $y \in \R^n$. Suppose there exist two distinct points $p \neq q$ in $\M$ such that
    \begin{equation*}
        \norm{y - p} = \norm{y - q} = d(y,\M) \eqdef \rho.
    \end{equation*}
    We bound $\rho$ from below.

    \begin{itemize}
        \item Step $1$. $y - p \in N_{p}\M$: Define $g(x) = \norm{y-x}^2 \colon \M \sto \R$. For any $u \in T_{p}\M$, choose a smooth curve $\gamma \colon (-\varepsilon,\varepsilon) \sto \M$ such that $\gamma(0) = p$ and $\dot{\gamma}(0) = u$. Since $p$ minimizes $g$, $t=0$ is a critical point of $g(\gamma(t))$, and hence
        \begin{equation*}
            0 = \lv{\frac{\mathrm{d}}{\mathrm{d}t}}_{t = 0}g(\gamma(t)) = 2 \inn{u, y - p}.
        \end{equation*}
        Because $u$ is arbitrary, this implies $y - p \in N_{p}\M$.

        \item Step $2$. A Taylor bound: Let $v \defeq q-p \neq 0$. Since $p,q \in \M \subset \mathcal{V}$ and $\mathcal{V}$ is convex, the line segment
        \begin{equation*}
            [p,q] \defeq \bb{tp + (1-t)q \colon t \in [0,1]}
        \end{equation*}
        is contained in $\mathcal{V}$. Using Taylor's formula and the fact that $F(p)=F(q)=0$, we obtain
        \begin{equation*}
            -DF(p)v
            = F(q) - F(p) - DF(p)v
            = \int_0^1 (1-t)D^2F(p+tv)[v,v]~\dx{t}.
        \end{equation*}
        Let $C_2 \defeq \sup_{y\in \mathcal{V}} \norm*{D^2F(y)}_{\op{op}} < \infty$. Then
        \begin{equation}\label{eq:cal_level_bound_DF_tubular}
            \norm*{DF(p)v}
            \leq \int_0^1 (1-t) C_2 \norm{v}^2 ~\dx{t}
            = \frac{C_2}{2}\norm{v}^2.
        \end{equation}

        \item Step $3$. Lower bound on $\rho$: Let $w \defeq y - p \in N_p\M$. Since $y-q = y-p-(q-p) = w-v$ and $\norm{y-q}=\norm{y-p}=\norm{w}=\rho$, we have
        \begin{equation*}
            \norm{w} = \norm{w - v}
            \quad \Rightarrow \quad
            \frac{1}{2}\norm{v}^2 = \inn{w,v}.
        \end{equation*}
        Decompose $v = v_T + v_N$ with $v_T \in T_p\M$ and $v_N \in N_p\M$. Since $w \in N_p\M$, we obtain $\inn{w,v}=\inn{w,v_N}$, and thus, by Cauchy--Schwarz,
        \begin{equation}\label{eq:cal_level_bound_DF_tubular_0}
            \frac{1}{2}\norm{v}^2
            = \inn{w,v_N}
            \leq \norm{w}\norm{v_N}
            = \rho \norm{v_N}.
        \end{equation}
        In particular, $v_N \neq 0$ since $v\neq 0$. Moreover, because $v_T \in T_p\M = \ker DF(p)$, we have $DF(p)v = DF(p)v_N$. Combining this with (\ref{eq:cal_level_bound_DF_tubular}) and (\ref{eq:cal_level_bound_DF_tubular_0}) gives
        \begin{equation}\label{eq:cal_level_bound_DF_tubular_result}
            \norm*{DF(p)v_N}
            = \norm*{DF(p)v}
            \leq \frac{C_2}{2}\norm{v}^2
            \leq C_2\rho \norm{v_N}.
        \end{equation}
        On the other hand, if $c_0 \defeq \inf_{x \in \M}\sigma_{\min}(DF(x)) > 0$, then
        \begin{equation}\label{eq:cal_level_low_DF_tubular}
            \norm*{DF(p)v_N} \geq c_0\norm{v_N}.
        \end{equation}
        Combining (\ref{eq:cal_level_bound_DF_tubular_result}) and (\ref{eq:cal_level_low_DF_tubular}) gives
        \begin{equation*}
            c_0\norm{v_N} \leq C_2\rho \norm{v_N}.
        \end{equation*}
        Since $v_N \neq 0$, we conclude that
        \begin{equation*}
            \rho \geq \frac{c_0}{C_2}.
        \end{equation*}
    \end{itemize}

    Therefore, if $y \in \R^n$ admits more than one point $x \in \M$ such that $\norm{y - x} = d(y,\M)$, then $d(y,\M) \geq c_0 / C_2$. Equivalently, if $d(y,\M) < c_0 / C_2$, then $y$ has a unique nearest point in $\M$. By the definition of reach (\ref{eq:cal_level_def_r0_reach}), this implies
    \begin{equation*}
        r_0 \geq c_0 / C_2 > 0. 
    \end{equation*}
\end{proof}

\subsection{Relationship between Assumption \ref{assum:geometric_iota} and Assumption \ref{assum:uniform_tubular}}\label{appen:relationship_of_assumption_iota_assum_tubular}

\parhead{Assumption \ref{assum:geometric_iota} to Assumption \ref{assum:uniform_tubular}.} By Proposition \ref{lem:assum_to_bound_second_form}, Assumption \ref{assum:geometric_iota} is equivalent to uniform boundedness of the second fundamental form and its covariant derivative:
\begin{equation*}
    \sup_{x \in \M} \norm*{\mathrm{II}_x}_{\op{op}} < \infty,\quad
    \sup_{x \in \M} \norm*{\nabla\mathrm{II}_x}_{\op{op}} < \infty.
\end{equation*}
These bounds capture local geometric regularity of $\M$. In particular, boundedness of the second fundamental form implies that the tubular map $T(x,\xi) = x + \xi$ is a local diffeomorphism on a neighborhood of the zero section in $N\M$, as stated in Proposition \ref{prop:bound_second_local_tubular}.

However, in general, control of the second fundamental form alone does not ensure global injectivity of $T$; see Example \ref{exam:second_cannot_uniform_tubular}. If, in addition, the embedding map $\iota \colon \M \hookrightarrow \R^n$ is globally bi-Lipschitz continuous, then boundedness of the second fundamental form does imply the existence of a uniform tubular neighborhood; see Appendix \ref{appen:assumption_iota_imply_assum_tubular}.

\parhead{Assumption \ref{assum:uniform_tubular} to Assumption \ref{assum:geometric_iota}.} Since the existence of a uniform tubular neighborhood guarantees global diffeomorphism of the tubular map $T$, it is not hard to see that, if $\M$ admits a uniform tubular neighborhood, then the second fundamental form is uniformly bounded; see Proposition \ref{prop:uniform_tubular_bound_second_form}.

However, the existence of a uniform tubular neighborhood does not, in general, control the covariant derivative of the second fundamental form. For instance, consider a graph manifold $\M = \bb{(x,f(x)) \colon x \in \R^m} \subset \R^n$ associated with a smooth map $f \colon \R^m \sto \R^{n-m}$. By Theorem \ref{thm:cal_graph_tubular}, if $C_1 = \norm*{Df}_{\infty} < \infty$ and $C_2 = \norm*{D^2f}_{\infty} < \infty$, then $\M$ admits a uniform tubular neighborhood with radius
\begin{equation*}
    r_0 \defeq \frac{1}{C_2(1+2C_1^2)} < \frac{1}{C_2}.
\end{equation*}
In contrast, bounding $\nabla \mathrm{II}$ requires additional regularity. As shown in Theorem \ref{thm:cal_graph_iota_thm}, control of $\nabla \mathrm{II}$ (equivalently, of $\nabla^2 d\iota$ by Proposition \ref{lem:assum_to_bound_second_form}) involves third-order derivatives of $f$. Indeed, the representations of $\nabla^2 d\iota$ in (\ref{eq:cal_2diota_ext_graph}) and (\ref{eq:cal_2diota_ij_graph}) contain terms depending on $D^3f$. Consequently, a uniform tubular neighborhood may exist even when $\nabla \mathrm{II}$ is unbounded, since boundedness of $D^3f$ is not required in Theorem \ref{thm:cal_graph_tubular}. We provide Example \ref{exam:uniform_tubular_cannot_deriv_second} for such case.

\begin{prop}\label{prop:bound_second_local_tubular}
    Let $\M \subset \R^n$ be a Riemannian submanifold. If
    \begin{equation*}
        \kappa = \sup_{x \in \M} \norm*{\mathrm{II}_x}_{\op{op}} < \infty,
    \end{equation*}
    then, on the open set
    \begin{equation*}
        \mathcal{U}_{1 /\kappa} \defeq \bb{(x,\xi) \in N\M \colon \norm{\xi} < 1 / \kappa} \subset N\M,
    \end{equation*}
    the tubular map $T \colon \mathcal{U}_{1 / \kappa} \sto \R^n$ given by $T(x,\xi) = x+\xi$ is a local diffeomorphism.   
\end{prop}
\begin{proof}
    We first identify $T_{(x,\xi)}(N\M)$ with $T_x\M \oplus N_x\M$. Locally, choose an orthonormal frame $\bb{E_i(x)}_{i=1}^m$ for $T\M$ in a neighborhood of $x$. Then $(x,\xi) \in N\M$ if and only if
    \begin{equation*}
        g_i(x,\xi) \defeq \inn{\xi,E_i(x)} = 0,\quad \forall~ i=1,\ldots,m.
    \end{equation*}
    Therefore, for $(v,\zeta) \in T_x\M \times \R^n$, we have $(v,\zeta) \in T_{(x,\xi)}(N\M)$ if and only if $dg_i(x,\xi)[v,\zeta]=0$ for $i=1,\ldots,m$. Computing,
    \begin{align*}
        0 =dg_i(x,\xi)[v,\zeta]
        &= \lv{\frac{\mathrm{d}}{\mathrm{d}t}}_{t=0} g_i(x+tv,\xi + t\zeta)
        = \lv{\frac{\mathrm{d}}{\mathrm{d}t}}_{t=0} \inn{\xi+t\zeta,E_i(x+tv)} \\
        &= \inn{\zeta,E_i(x)} + \inn{\xi,D_vE_i(x)}
        = \inn{\zeta,E_i(x)} + \inn{\xi,\nabla_vE_i(x) + \mathrm{II}_x(E_i(x),v)} \\
        &= \inn{\zeta,E_i(x)} + \inn{S_{x,\xi}v,E_i(x)},
    \end{align*}
    where $S_{x,\xi}$ is the shape operator along $\xi$ at $x$. Hence $\zeta^\top = -S_{x,\xi}v$, and therefore
    \begin{equation*}
        \zeta = -S_{x,\xi}v+w
    \end{equation*}
    for some $w \in N_x\M$. It follows that
    \begin{equation*}
        T_{(x,\xi)}(N\M) = \bb{(v,-S_{x,\xi}v+w) \colon v \in T_x\M,~w \in N_x\M},
    \end{equation*}
    so the map $(v,w) \mapsto (v,-S_{x,\xi}v+w)$ gives an isomorphism $T_{(x,\xi)}(N\M) \cong T_x\M \oplus N_x\M$.

    Under this identification, we derive an explicit expression for
    \begin{equation*}
        dT(x,\xi) \colon T_{(x,\xi)}(N\M) \cong T_x\M \oplus N_x\M \sto \R^n.
    \end{equation*}
    Let $(x(t),\xi(t))$ be a smooth curve in $N\M$ such that $(x(0),\xi(0))=(x,\xi)$ and $(\dot{x}(0),\dot{\xi}(0))=(v,\zeta)$, where $\zeta=-S_{x,\xi}v+w$. Then
    \begin{equation*}
        dT(x,\xi)(v,w) = \lv{\frac{\mathrm{d}}{\mathrm{d}t}}_{t=0} (x(t) + \xi(t)) = (\op{id}-S_{x,\xi})v + w.
    \end{equation*}
    Since $\dim T_{(x,\xi)}(N\M)=n$, the map $dT(x,\xi)$ is invertible if and only if $\ker dT(x,\xi)=\bb{0}$. Moreover,
    \begin{equation*}
        dT(x,\xi)(v,w) = (\op{id}-S_{x,\xi})v + w = 0
    \end{equation*}
    holds if and only if $w=0$ and $(\op{id}-S_{x,\xi})v=0$, because $w \in N_x\M$ while $(\op{id}-S_{x,\xi})v \in T_x\M$. Hence $\ker dT(x,\xi)=\bb{0}$ if and only if $\op{id}-S_{x,\xi}$ is invertible.

    By Lemma \ref{lem:bound_shape_operator}, the bound $\kappa = \sup_{x \in \M} \norm*{\mathrm{II}_x}_{\op{op}} < \infty$ implies
    \begin{equation*}
        \norm*{S_{x,\xi}}_{\op{op}} \leq \kappa \norm{\xi}.
    \end{equation*}
    Therefore, for any
    \begin{equation*}
        (x,\xi) \in \mathcal{U}_{r_0} \defeq \bb{(x,\xi) \in N\M \colon \norm{\xi} < r_0},\quad r_0 \defeq \frac{1}{\kappa},
    \end{equation*}
    we have $\norm*{S_{x,\xi}}_{\op{op}} < 1$. In particular, $\op{id}-S_{x,\xi}$ is invertible by the Neumann series \citep{horn2012matrix}. Consequently, $dT(x,\xi)$ is invertible for all $(x,\xi) \in \mathcal{U}_{r_0}$, and hence $T \colon \mathcal{U}_{r_0} \sto \R^n$ is a local diffeomorphism. 
\end{proof}

\begin{exam}\label{exam:second_cannot_uniform_tubular}
    Let $\M = \bb{(x,e^{-x}) \colon x\in \R} \cup \bb{(x,-e^{-x}) \colon x\in \R} \subset \R^2$. For a plane curve, the second fundamental form is essentially the curvature, i.e.,
    \begin{equation*}
        \mathrm{II}(\vec{T},\vec{T}) = \kappa \vec{n},
    \end{equation*}
    where $\vec{T}$ is the unit tangent vector, $\vec{n}$ is a choice of unit normal vector, and $\kappa$ is the curvature; see \citet{do2016differential} for details. Consequently,
    \begin{equation*}
        \norm*{\mathrm{II}_{(x,\pm e^{-x})}}_{\op{op}} = \abs{\kappa(x)} = \frac{e^{-x}}{\left(1+e^{-2 x}\right)^{3 / 2}} \quad \Rightarrow \quad \sup \norm*{\mathrm{II}}_{\op{op}} < \infty,
    \end{equation*}
    and
    \begin{equation*}
        \norm*{\nabla \mathrm{II}_{(x,\pm e^{-x})}}_{\op{op}} = \abs{\frac{\mathrm{d} \kappa}{\mathrm{d} s}} = \frac{e^{-x}\left|1-2 e^{-2 x}\right|}{\left(1+e^{-2 x}\right)^3} \quad \Rightarrow \quad \sup \norm*{\nabla \mathrm{II}}_{\op{op}} < \infty.
    \end{equation*}
    However, for $p_x = (x,0) \in \R^2$, there are two distinct nearest points in $\M$, namely $q_{x,1} = (x,e^{-x})$ and $q_{x,2} = (x,-e^{-x})$. Moreover,
    \begin{equation*}
        d(p_x,\M) = \norm*{p_x - q_{x,1}} = \norm*{p_x - q_{x,2}} = e^{-x} \to 0,
    \end{equation*}
    which implies that $\op{reach}(\M)=0$ (see Appendix \ref{appen:reach}). Therefore, $\M$ does not admit a uniform tubular neighborhood by Theorem \ref{thm:reach_to_tubular}.
\end{exam}

\begin{prop}\label{prop:uniform_tubular_bound_second_form}
    Let $\M \subset \R^n$ be a Riemannian submanifold. If $\M$ admits a uniform tubular neighborhood of radius $r_0$, then
    \begin{equation*}
        \sup_{x \in \M} \norm*{\mathrm{II}_x} \leq \frac{1}{r_0} < \infty.
    \end{equation*}
\end{prop}
\begin{proof}
    Let
    \begin{equation*}
        \mathcal{U}_{r_0} \defeq \bb{(x,\xi) \in N\M \colon \norm{\xi} < r_0}.
    \end{equation*}
    Since $T \colon \mathcal{U}_{r_0} \sto \mathcal{V}_{r_0} = T(\mathcal{U}_{r_0})$ is a diffeomorphism, the differential $dT_{(x,\xi)}$ is invertible for every $(x,\xi) \in \mathcal{U}_{r_0}$. Fix $x \in \M$ and $\eta \in N_x\M$ with $\norm{\eta} = 1$. Then $dT_{(x,t\eta)}$ is invertible for all $t \in (-r_0,r_0)$. By the argument in the proof of Proposition \ref{prop:bound_second_local_tubular}, this is equivalent to invertibility of
    \begin{equation*}
        \op{id} - S_{x,t\eta} = \op{id} - t S_{x,\eta},\quad \forall~t \in (-r_0,r_0)
    \end{equation*}
    Since $S_{x,\eta}$ is symmetric, if $\lambda$ is an eigenvalue of $S_{x,\eta}$, then $1-t\lambda$ is an eigenvalue of $\op{id}-tS_{x,\eta}$. Thus $1-t\lambda \neq 0$ for all $t \in (-r_0,r_0)$, which implies
    \begin{equation*}
        \abs{\lambda} \leq \frac{1}{r_0},
    \end{equation*}
    since otherwise taking $t = 1/\lambda \in (-r_0,r_0)$ would give $1-t\lambda=0$. Consequently,
    \begin{equation*}
        \norm{S_{x,\eta}} \leq \frac{1}{r_0}.
    \end{equation*}
    The conclusion follows from Lemma \ref{lem:bound_shape_operator}:
    \begin{equation*}
        \sup_{x \in \M} \norm*{\mathrm{II}_x} \leq \frac{1}{r_0} < \infty. 
    \end{equation*}
\end{proof}

\begin{exam}\label{exam:uniform_tubular_cannot_deriv_second}
    Define $f \colon \R \sto \R$ by
    \begin{equation*}
        f(x) \defeq \int_0^x I(t) ~\dx{t},\quad I(x) \defeq \int_0^x \sin (u^2) ~\dx{u}.
    \end{equation*}
    and consider a graph manifold $\M = \bb{(x,f(x)) \colon x \in \R} \subset \R^2$. Then
    \begin{equation*}
        f^\prime(x) = I(x),\quad f^{\prime\prime}(x) = \sin(x^2),\quad f^{\prime\prime\prime}(x) = 2x\cos(x^2).
    \end{equation*}
    In particular,
    \begin{equation*}
        \norm*{f^{\prime\prime}}_\infty < \infty,\quad \norm*{f^{\prime\prime\prime}}_\infty = \infty.
    \end{equation*}

    We next verify that $f^\prime$ is bounded. Since $f^\prime(x)=I(x)$, it is immediate that $\abs{I(x)} \leq I(1) <\infty$ for $\abs{x} \leq 1$. For $\abs{x} > 1$, applying the change of variables $t = u^2$ gives
    \begin{equation*}
        I(x) = \int_0^1 \sin (u^2) ~\dx{u} + \int_1^{x^2}\frac{\sin t}{t^{1/2}}~\dx{t}.
    \end{equation*}
    The first term is bounded by $1$. For the second term, integration by parts implies
    \begin{align*}
        \abs{\int_1^{x^2}\frac{\sin t}{t^{1/2}}~\dx{t}}
        &= \abs{\bc{-t^{-1 / 2}\cos t }_1^{x^2}-\frac{1}{2}\int_1^{x^2}t^{-3/2}\cos t ~\dx{t}}\\
        &\leq \abs{x^{-1}\cos x^2}+ \abs{\cos 1} +\frac{1}{2}\int_1^{x^2}\abs{t^{-3/2}\cos t} ~\dx{t} \\
        &\leq 2 + \frac{1}{2}\int_1^\infty t^{-3/2} ~\dx{t} = 3.
    \end{align*}
    Therefore, $\abs{I(x)} \leq I(1) + 3$ for all $x$, and hence
    \begin{equation*}
        \norm*{f^{\prime}}_{\infty} < \infty.
    \end{equation*}

    By Theorem \ref{thm:cal_graph_tubular}, $\M$ admits a uniform tubular neighborhood. However, since $\norm*{f^{\prime\prime\prime}}_\infty = \infty$, we have that $\nabla \mathrm{II}$ is unbounded by Theorem \ref{thm:cal_graph_iota_thm}. Alternatively, one may directly check that $\mathrm{d}\kappa / \mathrm{d}s$ is unbounded, as in Example \ref{exam:second_cannot_uniform_tubular}.
\end{exam}

\subsection{Extrinsic Criterion for Assumption \ref{assum:geometric_iota}}\label{appen:extrinsic_criterion_for_assumption_ref_assum_geometric_iota} 

Assumption \ref{assum:geometric_iota} requires uniform bounds on intrinsic derivatives, which can be difficult to verify in practice. Therefore, we provide a sufficient criterion formulated purely in terms of extrinsic derivatives, which additionally involves the orthogonal projection.

\begin{prop}\label{prop:ext_criter_assum_iota}
    Let $\M \subset \R^n$ be a Riemannian submanifold with the canonical embedding $\iota$. Let $\mathcal{V}$ be an open neighborhood of $\M$ in $\R^n$, and let $P \colon \mathcal{V} \sto \R^n$ be a smooth map such that for any $x \in \M$, $P(x)$ is the orthogonal projection onto $T_x\M$. If
    \begin{equation*}
        \sup_{x \in \M} \norm*{DP(x)}_{\op{op}} < \infty,\quad \sup_{x \in \M} \norm*{D^2P(x)}_{\op{op}} < \infty,
    \end{equation*}
    then $\iota$ satisfies Assumption \ref{assum:geometric_iota}.
\end{prop}
\begin{proof}
    First, by Proposition \ref{lem:assum_to_bound_second_form}, it suffices to prove that
    \begin{equation*}
        \sup_{x \in \M} \norm*{\mathrm{II}_x}_{\op{op}} < \infty,\quad \sup_{x \in \M} \norm*{\nabla \mathrm{II}_x}_{\op{op}} < \infty.
    \end{equation*}

    For $X,Y \in \Gamma(T\M)$, by the Gauss formula \citep{lee2018introduction},
    \begin{equation*}
        D_XY = \nabla_XY + \mathrm{II}(X,Y),\quad \nabla_XY \in \Gamma(T\M),~\mathrm{II}(X,Y) \in \Gamma(N\M),
    \end{equation*}
    and hence
    \begin{equation*}
        \mathrm{II}(X,Y) = (\op{id} - P)(D_XY).
    \end{equation*}
    Moreover, since $Y \in \Gamma(T\M)$, we have $PY = Y$, and thus
    \begin{equation*}
        D_XY = D_X(PY) = (D_XP)Y + P(D_XY).
    \end{equation*}
    It follows that
    \begin{equation}\label{eq:second_proj_ext_deriv}
        \mathrm{II}(X,Y) = (\op{id} - P)(D_XP)Y.
    \end{equation}
    Therefore, for any $x \in \M$ and any $u, v \in T_x\M$,
    \begin{equation*}
        \mathrm{II}_x(u,v) = (\op{id} - P(x))D_uP(x)v,
    \end{equation*}
    which implies
    \begin{equation}\label{eq:bound_second_ext_crite}
        \sup_{x \in \M} \norm*{\mathrm{II}_x}_{\op{op}} \leq \sup_{x \in \M} \norm*{DP(x)}_{\op{op}} < \infty.
    \end{equation}

    We next bound $\nabla \mathrm{II}$. Fix $x \in \M$ and $u,v,w \in T_x\M$. Choose $U,V,W \in \Gamma(T\M)$ such that
    \begin{equation*}
        U_x=u,~V_x=v,~W_x = w,\quad \nabla_W U|_x = \nabla_W V|_x = 0.
    \end{equation*}
    Since $(\nabla \mathrm{II})(U,V,W)(x)$ is independent of the choice of $U,V,W$ by Lemma \ref{lem:tensor_nabla_second}, we denote it by $(\nabla \mathrm{II}_x)(u,v,w)$. Then, by the definition (\ref{eq:def_nabla2_second}),
    \begin{equation}\label{eq:nabla2_scond_ext_deriv_1}
        (\nabla \mathrm{II}_x)(u,v,w) = \bc{D_w \mathrm{II}(U,V)|_x}^\perp = (\op{id} - P(x))D_w \mathrm{II}(U,V)|_x.
    \end{equation}
    Using (\ref{eq:second_proj_ext_deriv}), we compute
    \begin{equation*}
        D_w\mathrm{II}(U,V)|_x
        =
        D_w(\op{id} - P)|_x\bc{D_uP(x)v}
        + (\op{id} - P(x))D_w \bc{(D_UP)V}|_x.
    \end{equation*}
    Substituting into (\ref{eq:nabla2_scond_ext_deriv_1}) gives
    \begin{equation}\label{eq:nabla2_scond_ext_deriv_2}
        (\nabla \mathrm{II}_x)(u,v,w)
        =
        (\op{id} - P(x))D_w(\op{id} - P)|_x\bc{D_uP(x)v}
        + (\op{id} - P(x))D_w \bc{(D_UP)V}|_x.
    \end{equation}
    Note that $P^2 = P$, and hence
    \begin{equation*}
        D_\xi(P^2) = (D_\xi P)P + P(D_\xi P) = D_\xi P,\quad \forall~ \xi \in \R^n.
    \end{equation*}
    Multiplying by $P$ and $I-P$ on the left and right, respectively, yields
    \begin{equation}\label{eq:relation_P_DP_ext_crite}
        P(D_\xi P)P = 0,\quad (I-P)(D_\xi P)(I-P) = 0.
    \end{equation}
    The first identity in (\ref{eq:relation_P_DP_ext_crite}) implies that $D_uP(x)v \in N_x\M$, since $P(x)v = v$ for $v \in T_x\M$. Consequently, the second identity in (\ref{eq:relation_P_DP_ext_crite}) gives
    \begin{equation*}
        (\op{id} - P(x))D_w(\op{id} - P)|_x\bc{D_uP(x)v}
        =
        -(\op{id} - P(x))(D_wP|_x)(\op{id} - P(x))\bc{D_uP(x)v}
        = 0.
    \end{equation*}
    Combining this with (\ref{eq:nabla2_scond_ext_deriv_2}), we obtain
    \begin{equation}\label{eq:nabla2_scond_ext_deriv_3}
        (\nabla \mathrm{II}_x)(u,v,w) = (\op{id} - P(x))D_w \bc{(D_UP)V}|_x.
    \end{equation}
    For the right-hand side, the chain rule yields
    \begin{equation}\label{eq:nabla2_scond_ext_deriv_3_1}
        \begin{aligned}
            D_w \bc{(D_UP)V}|_x
            &= D_w(D_UP)|_xv+ D_uP(x)(D_wV)|_x  \\
            &= \bc{D^2P(x)[u,w] + D_{D_wU|_x}P(x)}v + D_uP(x)(D_wV)|_x.
        \end{aligned}
    \end{equation}
    Moreover, since $\nabla_W U|_x = \nabla_W V|_x = 0$, we have
    \begin{equation}\label{eq:nabla2_scond_ext_deriv_3_2}
        \begin{aligned}
            D_wU|_x &= \nabla_W U|_x + \mathrm{II}_x(u,w) = \mathrm{II}_x(u,w), \\
            D_wV|_x &= \nabla_W V|_x + \mathrm{II}_x(v,w) = \mathrm{II}_x(v,w).
        \end{aligned}
    \end{equation}
    Combining (\ref{eq:nabla2_scond_ext_deriv_3_1}), (\ref{eq:nabla2_scond_ext_deriv_3_2}), and (\ref{eq:nabla2_scond_ext_deriv_3}), we obtain
    \begin{equation}\label{eq:nabla2_scond_ext_deriv_result}
        \begin{aligned}
            (\nabla \mathrm{II}_x)(u,v,w)
            &=
            (\op{id} - P(x))\bc{D^2P(x)[u,w]v + D_{\mathrm{II}_x(u,w)}P(x)v + D_uP(x)\mathrm{II}_x(v,w)} \\
            &=
            (\op{id} - P(x))\bc{D^2P(x)[u,w]v + D_{\mathrm{II}_x(u,w)}P(x)v},
        \end{aligned}
    \end{equation}
    where the second equality follows from
    \begin{equation*}
        (\op{id} - P(x))D_uP(x)\mathrm{II}_x(v,w)
        =
        (\op{id} - P(x))D_uP(x)(\op{id} - P(x))\mathrm{II}_x(v,w)
        = 0,
    \end{equation*}
    by (\ref{eq:relation_P_DP_ext_crite}) and the fact that $\mathrm{II}_x(v,w) \in N_x\M$. Hence, (\ref{eq:nabla2_scond_ext_deriv_result}) implies
    \begin{equation*}
        \norm*{\nabla\mathrm{II}_x}_{\op{op}} \leq \norm*{D^2P(x)}_{\op{op}} + \norm*{DP(x)}_{\op{op}}\norm*{\mathrm{II}_x}_{\op{op}}.
    \end{equation*}
    Together with the bound on $\mathrm{II}$ in (\ref{eq:bound_second_ext_crite}), this implies
    \begin{equation*}
        \sup_{x \in \M} \norm*{\nabla\mathrm{II}_x}_{\op{op}} < \infty. 
    \end{equation*}
\end{proof}

\section{More Discussions on Embedding Map}\label{appen:discuss_embedding_map}

To derive strong convergence measured by the intrinsic distance $d_\M$ on $\M$ from Theorem \ref{thm:ext_p_strong_conv_GEM}, we require that $\M$ is well-embedded in $\R^n$; see Section \ref{appen:bi_lipschitz_continuity_of_embedding_map}. Furthermore, this embedding property makes Assumption \ref{assum:uniform_tubular} unnecessary under Assumption \ref{assum:geometric_iota}; see Section \ref{appen:assumption_iota_imply_assum_tubular}.

\subsection{Bi-Lipschitz Continuity}\label{appen:bi_lipschitz_continuity_of_embedding_map}

Let $\M \subset \R^n$ be a Riemannian submanifold with canonical embedding $\iota \colon \M \hookrightarrow \R^n$. For any $x,y \in \M$, let $\gamma \colon [0,1] \sto \M$ be a piecewise smooth curve with $\gamma(0)=x$ and $\gamma(1)=y$. Then
\begin{equation*}
    L(\gamma) \geq \norm{\iota(x)-\iota(y)},
\end{equation*}
where $L(\gamma)$ denotes the length of $\gamma$. Taking the infimum over all such curves yields
\begin{equation}\label{eq:natural_lip_iota}
    \norm{x-y} = \norm{\iota(x)-\iota(y)} \leq \inf_\gamma L(\gamma) = d_\M(x,y),
\end{equation}
where $d_\M$ is the Riemannian distance induced on $\M$. Hence, $\iota$ is automatically $1$-Lipschitz. In general, however, the reverse inequality in (\ref{eq:natural_lip_iota}) need not hold.

\begin{defn}[Bi-Lipschitz]\label{defn:bi-lipschitz}
    Let $\M \subset \R^n$ be a Riemannian submanifold with the embedding map $\iota$. If
    \begin{equation*}
        L_{\iota} \defeq \sup_{x \neq y \in \M} \frac{d_\M(x,y)}{\norm{x - y}} < \infty,
    \end{equation*}
    then $\iota$ is called globally bi-Lipschitz continuous, or $L_\iota$-bi-Lipschitz continuous.
\end{defn}
\begin{rmk}
    The terminology \emph{bi-Lipschitz} reflects the fact that, in this case, the inverse map $\iota^{-1} \colon \iota(\M) \sto \M$ is also Lipschitz. Equivalently,
    \begin{equation*}
        d_\M(x,y) \leq L_{\iota}\norm{x - y},\quad \forall~x,y \in \M.
    \end{equation*}
\end{rmk}

Regarding the bi-Lipschitz property of the embedding $\iota \colon \M \hookrightarrow \R^n$, we record the following two facts:
\begin{enumerate}[label=(\roman*)]
    \item If $\M \subset \R^n$ is compact, then $\iota$ is globally bi-Lipschitz continuous; see Theorem \ref{thm:cpt_bi-lipschitz}.
    \item There exist non-compact submanifolds $\M \subset \R^n$ for which $\iota$ is globally bi-Lipschitz continuous; see Example \ref{exam:noncpt_bilip_embedding}.
\end{enumerate}

\begin{exam}\label{exam:noncpt_bilip_embedding}
    Let $f \colon \R^m \sto \R^{n-m}$ be a smooth map such that
    \begin{equation*}
        C_1 \defeq \sup_{x \in \R^m} \norm{Df(x)}_{\op{op}} < \infty.
    \end{equation*}
    Consider the graph manifold $\M = \bb{(x,f(x)) \colon x \in \R^m}$ equipped with the induced Riemannian metric. For any $p=(x,f(x))$ and $q=(y,f(y))$ in $\M$, define
    \begin{equation*}
        \gamma(t) = \bc{x + t(y-x),\, f(x+t(y-x))} \in \M,\quad t \in [0,1].
    \end{equation*}
    Then $\gamma \colon [0,1] \sto \M$ satisfies $\gamma(0)=p$ and $\gamma(1)=q$. Moreover,
    \begin{equation*}
        \dot{\gamma}(t) = \bc{y-x,\, Df(x+t(y-x))(y-x)},
    \end{equation*}
    and hence
    \begin{equation*}
        \norm*{\dot{\gamma}(t)}
        \leq \sqrt{\norm{y-x}^2 + \norm{Df(x+t(y-x))}_{\op{op}}^2 \norm{y-x}^2}
        \leq \sqrt{1+C_1^2}\,\norm{x-y}.
    \end{equation*}
    Therefore,
    \begin{equation*}
        d_\M(p,q) \leq L(\gamma)
        = \int_0^1 \norm*{\dot{\gamma}(t)} ~\dx{t}
        \leq \sqrt{1+C_1^2}\,\norm{x-y}.
    \end{equation*}
\end{exam}

\begin{thm}\label{thm:cpt_bi-lipschitz}
    Let $\M \subset \R^n$ be an $m$-dimensional, connected, and compact Riemannian manifold with canonical embedding $\iota \colon \M \hookrightarrow \R^n$. Then $\iota$ is globally bi-Lipschitz continuous, i.e., there exists $L_{\iota} > 0$ such that
    \begin{equation*}
        d_\M(x,y) \leq L_{\iota}\norm{x - y},\quad \forall~x,y \in \M.
    \end{equation*}
\end{thm}
\begin{proof}
    Since $\M$ is compact and connected,
    \begin{equation*}
        \op{diam}(\M) \defeq \sup_{x,y \in \M} d_\M(x,y) < \infty.
    \end{equation*}
    The proof consists in the following two steps.

    \begin{itemize}
        \item Step $1$. Local bi-Lipschitz continuity: Fix $p \in \M$. Since $\M \subset \R^n$ is an embedded submanifold, after a rigid motion of $\R^n$ we may assume that $p=0$ and
        \begin{equation*}
            T_p\M = \R^{m} \times \bb{0} \subset \R^m \times \R^{n-m}.
        \end{equation*}
        By the Implicit Function Theorem, there exist $r>0$, an open neighborhood $\mathcal{V}$ of $p$, and a smooth map $f \colon B_r(0) \subset \R^m \sto \R^{n-m}$ such that
        \begin{equation*}
            \M \cap \mathcal{V} = \bb{(x,f(x)) \colon x \in B_r(0)},
        \end{equation*}
        i.e., $\M \cap \mathcal{V}$ is the graph of $f$. By possibly shrinking $r$, we may also assume
        \begin{equation*}
            K_p \defeq \sup_{x \in B_r(0)} \norm*{Df(x)}_{\op{op}} < \infty.
        \end{equation*}
        Then Example \ref{exam:noncpt_bilip_embedding} implies that, for all $x,y \in \M \cap \mathcal{V}$,
        \begin{equation*}
            d_\M(x,y) \leq \sqrt{1+K_p^2}\,\norm{x-y}.
        \end{equation*}

        \item Step $2$. Global boundedness: Suppose, toward a contradiction, that
        \begin{equation*}
            L_{\iota} \defeq \sup_{x \neq y \in \M} \frac{d_\M(x,y)}{\norm{x - y}} = \infty.
        \end{equation*}
        Then there exist sequences $x_k \neq y_k$ in $\M$ such that
        \begin{equation*}
            \frac{d_\M(x_k,y_k)}{\norm*{x_k - y_k}} > k.
        \end{equation*}
        Since $d_\M(x_k,y_k) \leq \op{diam}(\M)$, it follows that
        \begin{equation*}
            \norm*{x_k - y_k} < \frac{d_\M(x_k,y_k)}{k} \leq \frac{\op{diam}(\M)}{k} \sto 0.
        \end{equation*}
        By compactness (hence closedness) of $\M$, we may pass to a subsequence (not relabeled) such that
        \begin{equation*}
            x_k \to p,\quad y_k \to p
        \end{equation*}
        for some $p \in \M$. By Step $1$, there exists a neighborhood $\mathcal{V}$ of $p$ and a constant $C_p<\infty$ such that
        \begin{equation*}
            d_\M(x,y) \leq C_p \norm{x-y},\quad \forall~x,y \in \M \cap \mathcal{V}.
        \end{equation*}
        Therefore, for $k$ sufficiently large we have $x_k,y_k \in \M \cap \mathcal{V}$, and hence
        \begin{equation*}
            k < \frac{d_\M(x_k,y_k)}{\norm*{x_k - y_k}} \leq C_p,
        \end{equation*}
        which is a contradiction. This proves that $L_{\iota}<\infty$. 
    \end{itemize}
\end{proof}

\subsection{Deducing Assumption \ref{assum:uniform_tubular} from Assumption \ref{assum:geometric_iota}}\label{appen:assumption_iota_imply_assum_tubular}

As discussed in Appendix \ref{appen:relationship_of_assumption_iota_assum_tubular}, Assumption \ref{assum:geometric_iota} does not, in general, imply Assumption \ref{assum:uniform_tubular}: the existence of a uniform tubular neighborhood is a global property of $\M$, whereas Assumption \ref{assum:geometric_iota} only enforces local geometric regularity. However, when the canonical embedding $\iota \colon \M \hookrightarrow \R^n$ is globally bi-Lipschitz continuous, Assumption \ref{assum:geometric_iota}, more precisely, the uniform bound
\begin{equation*}
    \sup_{x \in \M} \norm*{\nabla d\iota(x)} = \sup_{x \in \M} \norm*{\mathrm{II}_x} < \infty,
\end{equation*}
is sufficient to guarantee the existence of a uniform tubular neighborhood, i.e., Assumption \ref{assum:uniform_tubular}; as shown in Theorem \ref{thm:bound_second_to_uniform_tubular}.

\begin{thm}\label{thm:bound_second_to_uniform_tubular}
    Let $\M \subset \R^n$ be a Riemannian submanifold. Assume that the canonical embedding $\iota \colon \M \sto \R^n$ is $L_\iota$-bi-Lipschitz continuous and that
    \begin{equation*}
        \kappa \defeq \sup_{x \in \M} \norm*{\mathrm{II}_x} < \infty.
    \end{equation*}
    Then $\M$ admits a uniform tubular neighborhood with radius
    \begin{equation*}
        r_0 = \frac{1}{\kappa(1+2L_\iota)}.
    \end{equation*}
\end{thm}
\begin{proof}
    Consider the tubular map $T \colon N\M \sto \R^n$ given by $T(x,\xi)=x+\xi$. Since $\kappa=\sup_\M \norm{\mathrm{II}}_{\op{op}}<\infty$ and $r_0<1/\kappa$, Proposition \ref{prop:bound_second_local_tubular} shows that $T$ is a local diffeomorphism on
    \begin{equation*}
        \mathcal{U}_{r_0} \defeq \bb{(x,\xi)\in N\M \colon \norm{\xi}<r_0}.
    \end{equation*}
    Hence, it remains to show that $T$ is injective on $\mathcal{U}_{r_0}$.

    Suppose that there exist $(x,\xi)\neq (y,\eta)$ in $\mathcal{U}_{r_0}$ such that $T(x,\xi)=T(y,\eta)$. Writing
    \begin{equation*}
        z \defeq x+\xi = y+\eta,
    \end{equation*}
    we necessarily have $x\neq y$. Moreover,
    \begin{equation*}
        \norm{x-y} \leq \norm{\xi}+\norm{\eta} < 2r_0,
    \end{equation*}
    and the $L_\iota$-bi-Lipschitz property gives
    \begin{equation}\label{eq:second_to_uniform_rho}
        0< \rho \defeq d_\M(x,y) \leq L_\iota \norm{x-y} < 2L_\iota r_0.
    \end{equation}

    Let $\gamma \colon [0,\rho] \sto \M$ be a unit-speed minimizing geodesic from $x$ to $y$, so that $\gamma(0)=x$, $\gamma(\rho)=y$, $\norm{\dot{\gamma}(t)}\equiv 1$, and $L(\gamma)=\rho$ (Existence follows from the Hopf--Rinow theorem, since $\M$ is assumed closed). Define
    \begin{equation*}
        f(t) \defeq \norm{\gamma(t)-z}^2,\quad t\in[0,\rho].
    \end{equation*}
    Then
    \begin{equation*}
        f^\prime(t) = 2\inn{\dot{\gamma}(t),\, \gamma(t)-z}.
    \end{equation*}
    Since $z-x=\xi \in N_x\M$ and $z-y=\eta \in N_y\M$, we have
    \begin{equation}\label{eq:second_to_uniform_f_prime}
        f^\prime(0)=0,\qquad f^\prime(\rho)=0,
    \end{equation}
    because $\dot{\gamma}(0)\in T_x\M$ and $\dot{\gamma}(\rho)\in T_y\M$.

    Next, using the Gauss formula along $\gamma$ and that $\gamma$ is a geodesic on $\M$,
    \begin{equation*}
        \gamma^{\prime\prime}(t) = \nabla_{\dot{\gamma}(t)}\dot{\gamma}(t) + \mathrm{II}(\dot{\gamma}(t),\dot{\gamma}(t)) = \mathrm{II}(\dot{\gamma}(t),\dot{\gamma}(t)).
    \end{equation*}
    Hence,
    \begin{align*}
        f^{\prime\prime}(t)
        &= 2\norm{\dot{\gamma}(t)}^2 + 2\inn{\gamma^{\prime\prime}(t),\, \gamma(t)-z} \\
        &= 2 + 2\inn{\mathrm{II}(\dot{\gamma}(t),\dot{\gamma}(t)),\, \gamma(t)-z}.
    \end{align*}
    Using $\norm{\dot{\gamma}(t)}\equiv 1$ and the bound $\norm{\mathrm{II}}_{\op{op}}\leq \kappa$, we obtain
    \begin{equation*}
        f^{\prime\prime}(t) \geq 2 - 2\kappa \norm{\gamma(t)-z}.
    \end{equation*}
    Moreover,
    \begin{equation*}
        \norm{\gamma(t)-z} \leq \norm{\gamma(t)-x} + \norm{x-z} \leq d_\M(\gamma(t),\gamma(0)) + \norm{\xi} < t + r_0,
    \end{equation*}
    where we used that $\iota$ is $1$-Lipschitz, (\ref{eq:natural_lip_iota}). Therefore, by (\ref{eq:second_to_uniform_rho}),
    \begin{equation*}
        f^{\prime\prime}(t) > 2 - 2\kappa(t+r_0) \geq 2 - 2\kappa(\rho+r_0) > 2 - 2\kappa(2L_\iota r_0 + r_0) = 0,\quad \forall~t\in[0,\rho],
    \end{equation*}
    where the last equality uses the definition $r_0 = 1/(\kappa(1+2L_\iota))$. Thus $f^\prime$ is strictly increasing on $[0,\rho]$, so $f^\prime(0) < f^\prime(\rho)$ since $\rho>0$, contradicting (\ref{eq:second_to_uniform_f_prime}). This contradiction shows that $T$ is injective on $\mathcal{U}_{r_0}$, and hence $T \colon \mathcal{U}_{r_0} \sto T(\mathcal{U}_{r_0})$ is a diffeomorphism. 
\end{proof}

%% file: appendix_p_conv_gem.tex
\section{\texorpdfstring{Geometry of $A(x)$}{Geometry of A(x)}}\label{appen:geometry_of_a_x_}

The extra drift $A(x)$ in (\ref{eq:mfd_sde_eucl_w}) plays an important role in controlling the strong pathwise discrepancy between the extrinsic EM $\bm{Y}^h_k$ and the intrinsic GEM $\bm{X}^h_k$. In this section, we first discuss its geometric meaning (Section \ref{appen:proof_of_proposition_prop_geometry_of_a}) and then introduce a useful formula for the subsequent discrepancy estimates.

\subsection{Proof of Proposition \ref{prop:geometry_of_A}}\label{appen:proof_of_proposition_prop_geometry_of_a}

\begin{proof}[Proof of Proposition \ref{prop:geometry_of_A}]
    By Theorem 3.1.4 in \citet{hsu2002stochastic}, for any $f \in C^\infty(\M)$,
    \begin{equation*}
        \Delta_\M f = \sum_{\alpha = 1}^n P_\alpha (P_\alpha (f)),
    \end{equation*}
    where $\Delta_\M$ is the Laplace--Beltrami operator, and $P_\alpha(x) = P(x)e_\alpha$ for $\bb{e_\alpha}_{\alpha=1}^n$, the canonical basis of $\R^n$.

    Let $\iota = (\iota^1,\ldots,\iota^n) \colon \M \hookrightarrow \R^n$ be the canonical embedding, where $\iota^j \colon \M \sto \R$ is given by $\iota^j(x) = x^j$ for $x = (x^1,\ldots,x^n) \in \M \subset \R^n$. Then
    \begin{equation*}
        \Delta_\M \iota^j = \sum_{\alpha = 1}^n P_\alpha (P_\alpha (\iota^j)).
    \end{equation*}
    Moreover, by extending $\iota^j(x^1,\ldots,x^n) = x^j$ to $\R^n$,
    \begin{equation*}
        P_\alpha (\iota^j) = \sum_{\beta = 1}^n P_{\beta\alpha}\frac{\partial \iota^j}{\partial x^{\beta}} = P_{j\alpha}.
    \end{equation*}
    Therefore,
    \begin{equation*}
        A^j(x) \defeq \frac{1}{2}\sum_{\alpha = 1}^n P_\alpha (P_{j\alpha})(x) = \frac{1}{2}\sum_{\alpha = 1}^n P_\alpha\bc{P_\alpha (\iota^j)} = \frac{1}{2}\Delta_\M \iota^j(x).
    \end{equation*}
    On the other hand, by Lemma \ref{lem:laplac_of_embedding},
    \begin{equation*}
        \Delta_\M \iota^j(x) = \sum_{i=1}^m \mathrm{II}^j_x\bc{E_i(x),E_i(x)},
    \end{equation*}
    where $\mathrm{II} = \bc{\mathrm{II}^1,\cdots,\mathrm{II}^n} \in \R^n$ is the second fundamental form. It follows that
    \begin{equation*}
        A^j(x) = \frac{1}{2}\sum_{i=1}^m \mathrm{II}^j_x\bc{E_i(x),E_i(x)},
    \end{equation*}
    for $j = 1,\ldots,n$, and hence
    \begin{equation*}
        A(x) = \frac{1}{2}\sum_{i=1}^m \mathrm{II}_x\bc{E_i(x),E_i(x)}.
    \end{equation*}
\end{proof}

\begin{cor}\label{cor:boundedness_of_A}
    Assume that $\M \subset \R^n$ is an $m$-dimensional Riemannian submanifold with $\kappa = \sup_{\M}\norm{\mathrm{II}_x} < \infty$. For $A(x)$ defined in Proposition \ref{prop:geometry_of_A},
    \begin{equation*}
        \sup_{x \in \M} \norm*{A(x)} < \infty.
    \end{equation*}
\end{cor}
\begin{proof}
    Fix $x \in \M$ and choose an orthonormal basis $\bb{E_i(x)}_{i=1}^m$ of $T_x\M$. By Proposition \ref{prop:geometry_of_A},
    \begin{equation*}
        \norm*{A(x)} \leq \frac{1}{2}\sum_{i=1}^m \norm*{\mathrm{II}_x\bc{E_i(x),E_i(x)}} \leq \frac{1}{2}m\kappa.
    \end{equation*}
    Taking the supremum over $x \in \M$ gives
    \begin{equation*}
        \sup_{x \in \M} \norm*{A(x)} \leq \frac{1}{2}m\kappa < \infty.
    \end{equation*}
\end{proof}

\begin{lem}\label{lem:laplac_of_embedding}
    Let $\M \subset \R^n$ be an $m$-dimensional Riemannian submanifold with the canonical embedding $\iota = (\iota^1,\ldots,\iota^n) \colon \M \hookrightarrow \R^n$, where $\iota^j \colon \M \sto \R$ is given by $\iota^j(x) = x^j$ for $x = (x^1,\ldots,x^n) \in \M \subset \R^n$. For any $x \in \M$, let $\bb{E_i(x)}_{i=1}^m$ be an orthonormal basis of $T_x\M$. Then
    \begin{equation*}
        \Delta_\M \iota^j(x) = \sum_{i=1}^m \mathrm{II}^j_x\bc{E_i(x),E_i(x)},\quad \forall ~ x\in \M,~\forall~j = 1,\ldots,n,
    \end{equation*}
    where $\Delta_\M$ is the Laplace--Beltrami operator, and $\mathrm{II} = \bc{\mathrm{II}^1,\cdots,\mathrm{II}^n} \in \R^n$ is the second fundamental form.
\end{lem}
\begin{proof}
    In the following, for any $u \in \R^n$, we denote by $u^\perp(x)$ the orthogonal projection of $u$ onto $N_x\M$.

    By the definition of $\Delta_\M$,
    \begin{equation*}
        \Delta_\M \iota^j = \sum_{i=1}^m \nabla^2 \iota^j \bc{E_i,E_i} = \sum_{i=1}^m \nabla_{E_i}\bc{\nabla \iota}(E_i).
    \end{equation*}
    Let $\nabla \iota^j$ denote the gradient of $\iota^j$ along $\M$. Then
    \begin{align*}
        \Delta_\M \iota^j
        &= \sum_{i=1}^m \inn{\nabla_{E_i}\nabla\iota^j,E_i} \\
        &= \sum_{i=1}^m \bc{\inn{D_{E_i}\nabla\iota^j,E_i} - \inn{\bc{D_{E_i}\nabla\iota^j}^\perp,E_i}} \\
        &= \sum_{i=1}^m \inn{D_{E_i}\nabla\iota^j,E_i},
    \end{align*}
    since $\bc{D_{E_i}\nabla\iota^j}^\perp \in N\M$ and thus $\inn{\bc{D_{E_i}\nabla\iota^j}^\perp,E_i} = 0$. On the right-hand side,
    \begin{equation*}
        \nabla\iota^j = \nabla^{\R}\iota^j - \bc{\nabla^{\R}\iota^j}^\perp,
    \end{equation*}
    where $\iota^j$ is extended to $\R^n$ by $\iota^j(x^1,\ldots,x^n) = x^j$, and $\nabla^{\R}\iota^j$ denotes the Euclidean gradient on $\R^n$. Note that $\nabla^{\R}\iota^j = e_j$, where $\bb{e_\alpha}_{\alpha=1}^n$ is the canonical basis of $\R^n$, and hence $D_{E_i}(\nabla^{\R}\iota^j) = 0$. Therefore,
    \begin{equation*}
        \Delta_\M \iota^j = \sum_{i=1}^m \inn{D_{E_i}\nabla\iota^j,E_i} = - \sum_{i=1}^m \inn{D_{E_i}(e_j^\perp),E_i}.
    \end{equation*}
    Moreover, since $\inn{e_j^\perp,E_i} = 0$, we have
    \begin{equation*}
       0 = D_{E_i} \bc{\inn{e_j^\perp,E_i}} = \inn{D_{E_i}(e_j^\perp),E_i} + \inn{e_j^\perp,D_{E_i}E_i}.
    \end{equation*}
    Consequently,
    \begin{equation*}
        \Delta_\M \iota^j = \sum_{i=1}^m \inn{e_j^\perp,D_{E_i}E_i} = \sum_{i=1}^m \inn{e_j,\bc{D_{E_i}E_i}^\perp} = \sum_{i=1}^m \inn{e_j,\mathrm{II}(E_i,E_i)} = \sum_{i=1}^m\mathrm{II}^j\bc{E_i,E_i}.
    \end{equation*}
\end{proof}

\subsection{Induced Formula}\label{appen:induced_formula}

\begin{lem}\label{lem:expectation_second_A}
    Let $\M \subset \R^n$ be an $m$-dimensional Riemannian manifold with a smoothly defined orthogonal projection $P(x)$. Let $A(x)$ be defined as (\ref{eq:def_U_A}). If $\bm{\xi} \sim \mathcal{N}(0,I_n)$, then for any $x \in \M$,
    \begin{equation*}
        A(x) = \frac{1}{2}\E\bj{\mathrm{II}_x(P(x)\bm{\xi},P(x)\bm{\xi})}.
    \end{equation*} 
\end{lem}
\begin{proof}
    Fix $x \in \M$ and choose an orthonormal basis $\bb{E_i(x)}_{i=1}^m$ of $T_x\M$. Since $P(x)\bm{\xi} \in T_x\M$, we can write
    \begin{equation*}
        P(x)\bm{\xi} = \sum_{i=1}^m\bm{\zeta}_i E_i(x),
    \end{equation*}
    where
    \begin{equation*}
        \bm{\zeta}_i = \inn{\bm{\xi},E_i(x)},\quad i = 1,\ldots,m.
    \end{equation*}
    Note that $\bm{\zeta} = (\bm{\zeta}_i) \sim \mathcal{N}(0,I_m)$ \citep{muirhead2009aspects}. Since $\mathrm{II}_x$ is bilinear,
    \begin{equation*}
        \mathrm{II}_x(P(x)\bm{\xi},P(x)\bm{\xi}) = \sum_{i,j = 1}^m \bm{\zeta}_i\bm{\zeta}_j\mathrm{II}_x\bc{E_i(x),E_j(x)}.
    \end{equation*}
    Taking expectations and using Proposition \ref{prop:geometry_of_A}, we obtain
    \begin{equation*}
        \E\bj{\mathrm{II}_x(P(x)\bm{\xi},P(x)\bm{\xi})} =  \sum_{i,j = 1}^m \E\bj{\bm{\zeta}_i\bm{\zeta}_j}\mathrm{II}_x\bc{E_i(x),E_j(x)} = \sum_{i = 1}^m \mathrm{II}_x\bc{E_i(x),E_i(x)} = 2 A(x).
    \end{equation*}
\end{proof}

\section{Global Lipschitz Continuity of Extension}\label{appen:global_lipschitz_continuity_of_extension}

Let $\M \subset \R^n$ be a Riemannian submanifold and let $F \colon \M \sto \R^n$ (or $\R^{n \times n}$). By the Tubular Neighborhood Theorem (Theorem \ref{thm:tubular_neighborhood}) and the closedness of $\M$, there exist two tubular neighborhoods $\mathcal{V}_0,\mathcal{V}$ such that $\M \subset \clo{\mathcal{V}_0} \subset \mathcal{V}$. To extend $F$ to $\R^n$, we apply the Urysohn Lemma \citep{munkres2013topology} and define
\begin{equation*}
    \widetilde{F}(y) \defeq \chi(y)F(\pi(y)),\quad \forall~ y \in \R^n,
\end{equation*}
where $\pi$ is the canonical projection of $\mathcal{V}$, and $\chi \colon \R^n \sto [0,1]$ is a bump function such that $\chi|_{\M} \equiv 1$ and $\chi|_{\R^n \backslash \clo{\mathcal{V}_0}} \equiv 0$. To make $\widetilde{F}$ globally Lipschitz continuous, it suffices to obtain a uniform bound on
\begin{equation}\label{eq:ext_chain_rule}
    D\widetilde{F}(y)[w] = D\chi(y)[w]F(\pi(y)) + \chi(y)DF(\pi(y))\bj{D\pi(y)[w]},\quad \forall~y\in \R^n,~\forall~w \in \R^n.
\end{equation}
We apply this construction to the coefficients $F(x) = V(x)$, $A(x)$, and $P(x)$ in SDE (\ref{eq:mfd_sde_eucl_w}). The following results provide globally Lipschitz extensions under the geometric boundedness of $\M$ and the regularity of $V$.

\subsection{Boundedness of Bump Function}

We first need a uniform bound on $\norm*{D\chi(y)}$. In general, such a bound cannot be obtained if $\M$ does not admit a uniform tubular neighborhood, which is why we need Assumption \ref{assum:uniform_tubular}.

\begin{lem}\label{lem:bound_chi}
    Let $\M \subset \R^n$ be a Riemannian submanifold admitting a uniform tubular neighborhood $\mathcal{V}$. Then there exists a uniform tubular neighborhood $\mathcal{V}_0$ with $\M \subset \clo{\mathcal{V}_0} \subset \mathcal{V}$, and a bump function $\chi \in C^\infty(\R^n)$ such that,
    \begin{equation*}
        \chi\in[0,1],\quad \chi|_\M \equiv 1,\quad \chi|_{\R^n \backslash \clo{\mathcal{V}_0}} \equiv 0,\quad\text{ and }\sup_{y \in \R^n} \norm*{D\chi(y)} < \infty.
    \end{equation*}
\end{lem}
\begin{proof}
    Let $r_0 > 0$ be the uniform radius of $\mathcal{V}$. Then $\mathcal{V} = \mathcal{V}_{r_0}$ (Theorem \ref{thm:tubular_neighborhood}), where $\mathcal{V}_{r} = \bb{y \in \R^n \colon d(y,\M) < r}$ for any $r>0$.

    By Lemma \ref{lem:cut_func}, we can choose $\psi \in C^\infty([0,\infty))$ such that
    \begin{equation*}
        \psi|_{[0,r_0/4]} \equiv 1,\quad \psi|_{[r_0/2,\infty)} \equiv 0,\quad \text{ and } C_\psi \defeq \sup_{t \geq 0} \abs{\psi^\prime(t)} < \infty.
    \end{equation*}
    Let $g(y) \defeq d(y,\M) \colon \mathcal{V} \sto \R$. Define
    \begin{equation*}
        \chi(y) \defeq \psi(g(y)),\quad \forall~ y \in \R^n.
    \end{equation*}
    Then $\chi \in [0,1]$. Moreover, $\chi(y) = 1$ on $\mathcal{V}_{r_0/4}$ and $\chi(y) = 0$ on $\R^n \backslash \clo{\mathcal{V}_{r_0/2}}$. Setting $\mathcal{V}_0 \defeq \mathcal{V}_{r_0/2} \subset \mathcal{V}$, we obtain
    \begin{equation*}
        \chi|_\M \equiv 1,\quad \chi|_{\R^n \backslash \clo{\mathcal{V}_0}} \equiv 0.
    \end{equation*}

    Moreover, $D\chi(y)=0$ for $y \in \mathcal{V}_{r_0/4}$ and for $y \in \R^n \backslash \clo{\mathcal{V}_{r_0/2}}$. For
    \begin{equation*}
        y \in \clo{\mathcal{V}_{r_0/2}} \backslash \mathcal{V}_{r_0/4} \subset \mathcal{V} \backslash \M,
    \end{equation*}
    the distance function $g$ is smooth on $\mathcal{V} \backslash \M$, and $\norm*{Dg(y)} = 1$ there \citep{lee2018introduction}. Hence,
    \begin{equation*}
        \norm*{D\chi(y)} = \abs{\psi^\prime(g(y))}\norm*{Dg(y)} \leq C_\psi.
    \end{equation*}
    Therefore,
    \begin{equation*}
        \sup_{y \in \R^n}\norm*{D\chi(y)} \leq C_\psi < \infty.
    \end{equation*}
\end{proof}

\subsection{Proof of Lemma \ref{lem:global_lip_ext_of_Vec_and_Proj}}\label{appen:proof_of_lemma_global_lip_vec_and_proj}

\begin{proof}[Proof of Lemma \ref{lem:global_lip_ext_of_Vec_and_Proj}]
    We treat $V(x)$ and $P(x)$ separately, following the steps below.
    \begin{itemize}
        \item Intrinsic derivatives: First, $\kappa \defeq \sup_\M \norm{\mathrm{II}_x} < \infty$ by Assumption \ref{assum:geometric_iota} and Lemma \ref{lem:assum_to_bound_second_form}. For $P(x)$, Lemma \ref{lem:bound_intrin_ortho_proj} shows
        \begin{equation*}
            \sup_{x \in \M}\norm*{\nabla P(x)}_{\op{op}} \leq \kappa < \infty.
        \end{equation*}
        For $V(x)$, Assumption \ref{assum:vector_field} ensures
        \begin{equation*}
            C_{2,V} \defeq \sup_{x \in \M} \norm*{\nabla V(x)}_{\op{op}} < \infty.
        \end{equation*}

        \item Extrinsic derivatives: Fix $x \in \M$ and let $v \in T_x\M$ with $\norm{v} = 1$. For $P(x)$, by the Gauss's formula \citep[Theorem 8.2]{lee2018introduction},
        \begin{equation*}
            D_v(P(x)w) = \nabla_v (P(x)w) + \mathrm{II}(P(x)w,v),\quad \forall~w \in \R^n.
        \end{equation*}
        Using the intrinsic bound above, we obtain
        \begin{align*}
            \norm{D_v(P(x)w)}
            &\leq \norm*{\nabla_v (P(x)w)} + \norm*{\mathrm{II}(P(x)w,v)} \\
            &\leq \norm{\nabla P(x)}_{\op{op}}\norm{v}\norm{w} + \norm{\mathrm{II}}_{\op{op}} \norm{P(x)w}\norm{v} \\
            &\leq 2\kappa \norm{w},
        \end{align*}
        since $P(x)$ is an orthogonal projection. Viewing $w \mapsto D_v(P(x)w)$ as a linear map, it follows that
        \begin{equation*}
            \norm*{D_vP(x)}_{\op{op}} = \sup_w \frac{\norm{D_v(P(x)w)}}{\norm{w}} \leq 2\kappa.
        \end{equation*}
        Furthermore, since $\norm*{D_vP(x)}_{\op{F}} \leq \sqrt{n}\norm*{D_vP(x)}_{\op{op}}$ \citep{horn2012matrix},
        \begin{equation}\label{eq:proof_ext_bound_P}
            C_P \defeq \sup_{x \in \M}\sup_{v \in T_x\M,\norm{v}=1} \norm*{D_vP(x)}_{\op{F}} \leq 2\sqrt{n}\kappa < \infty. 
        \end{equation}
        For $V(x)$, Assumption \ref{assum:vector_field} also gives $C_{1,V} \defeq \sup_{x \in \M} \norm*{V(x)} < \infty$. Hence,
        \begin{align*}
            \norm{D_v(V(x))}
            &\leq \norm*{\nabla_v (V(x))} + \norm*{\mathrm{II}(V(x),v)} \\
            &\leq \norm{\nabla V(x)}_{\op{op}}\norm{v} + \norm{\mathrm{II}}_{\op{op}} \norm{V(x)}\norm{v} \\
            &\leq \bc{C_{2,V} + \kappa C_{1,V}}\norm{v},
        \end{align*}
        which implies
        \begin{equation}\label{eq:proof_ext_bound_V}
            C_V \defeq \sup_{x \in \M}\sup_{v \in T_xM,\norm{v}=1} \norm*{D_vV(x)} \leq C_{2,V} + \kappa C_{1,V}  <\infty.
        \end{equation}

        \item Global extension: By Lemma \ref{lem:boundeness_of_Dpi}, since $\M$ admits a uniform tubular neighborhood and $\kappa = \sup_\M \norm{\mathrm{II}_x} < \infty$, there exists a uniform tubular neighborhood $\mathcal{V}$ with canonical projection $\pi$ such that
        \begin{equation}\label{eq:proof_ext_bound_Dpi}
            C_\pi \defeq \sup_{y \in \mathcal{V}}\norm*{D\pi(y)}_{\op{op}} < \infty.
        \end{equation}
        Moreover, by Lemma \ref{lem:bound_chi}, there exists a uniform tubular neighborhood $\mathcal{V}_0$ such that $\clo{\mathcal{V}_0} \subset \mathcal{V}$ and a bump function $\chi \in C^\infty(\R^n)$ satisfying
        \begin{equation}\label{eq:proof_ext_bound_chi_0}
            \chi \in [0,1],\quad \chi|_\M \equiv 1,\quad \chi|_{\R^n \backslash \clo{\mathcal{V}_0}} \equiv 0,
        \end{equation}
        and
        \begin{equation}\label{eq:proof_ext_bound_chi}
            C_\chi \defeq \sup_{y \in \R^n} \norm*{D\chi(y)} < \infty.
        \end{equation}
        Define
        \begin{equation*}
            \widetilde{V}(y) \defeq \chi(y)V(\pi(y)),\quad \widetilde{P}(y) \defeq \chi(y)P(\pi(y)),\qquad \forall~y\in \R^n,
        \end{equation*}
        which are well-defined since $\pi(y) \in \M$ for all $y \in \clo{\mathcal{V}_0} \subset \mathcal{V}$, and $\chi(y)=0$ for all $y \in \R^n \backslash \clo{\mathcal{V}_0}$.

        \item Global Lipschitz of $\widetilde{P}$: For any $y \in \R^n$ and any $u \in \R^n$ with $\norm{u} = 1$, the chain rule implies
        \begin{align*}
            D_u\widetilde{P}(y)
            &= (D_u\chi(y))P(\pi(y)) + \chi(y) DP(\pi(y))\bj{D\pi(y)[u]} \\
            &= (D_u\chi(y))P(\pi(y)) + \chi(y)D_vP(\pi(y)),
        \end{align*}
        where $v = D\pi(y)[u] \in T_{\pi(y)}\M$; see Appendix \ref{appen:projection}. By (\ref{eq:proof_ext_bound_P}), (\ref{eq:proof_ext_bound_Dpi}), and (\ref{eq:proof_ext_bound_chi}),
        \begin{align*}
            \norm*{D_u\widetilde{P}(y)}_{\op{F}}
            &\leq \norm*{(D_u\chi(y))P(\pi(y))}_{\op{F}} + \norm*{\chi(y)D_vP(\pi(y))}_{\op{F}} \\
            &\leq \abs{D_u\chi(y)}\norm*{P(\pi(y))}_{\op{F}} + \norm*{D_vP(\pi(y))}_{\op{F}} \\
            &\leq C_\chi \norm*{P(\pi(y))}_{\op{F}} \norm{u} + C_P \norm*{v} \\
            &=  C_\chi \norm*{P(\pi(y))}_{\op{F}} \norm{u} + C_P \norm*{D\pi(y)[u]}\\
            &\leq \sqrt{n} C_\chi + C_PC_\pi,
        \end{align*}
        where the last inequality uses $\norm*{P(\pi(y))}_{\op{F}} \leq \sqrt{n}$ since $P(\pi(y))$ is an orthogonal projection. Consequently,
        \begin{equation*}
            L_P \defeq \sup_{y \in \R^n} \norm*{D\widetilde{P}(y)}_{\op{op}} \leq C_PC_\pi + \sqrt{n}C_\chi < \infty,
        \end{equation*}
        and hence $\widetilde{P} \colon \R^n \sto \R^{n\times n}$ is globally Lipschitz continuous, i.e.,
        \begin{equation*}
            \norm*{\widetilde{P}(y_1) - \widetilde{P}(y_2)}_{\op{F}} \leq L_P\norm*{y_1 - y_2}, \quad \forall~y_1,~y_2 \in \R^n.
        \end{equation*}

        \item Global Lipschitz of $\widetilde{V}$: For any $y \in \R^n$ and any $u \in \R^n$ with $\norm{u} = 1$, similarly,
        \begin{equation*}
            D_u\widetilde{V}(y) = (D_u\chi(y))V(\pi(y)) + \chi(y)D_vV(\pi(y)),
        \end{equation*}
        where $v = D\pi(y)[u] \in T_{\pi(y)}\M$. Therefore, by $C_{1,V}= \sup_{x \in \M} \norm*{V(x)} < \infty$ and (\ref{eq:proof_ext_bound_Dpi}), (\ref{eq:proof_ext_bound_V}), and (\ref{eq:proof_ext_bound_chi}),
        \begin{align*}
            \norm*{D_u\widetilde{V}(y)} &\leq \norm*{(D_u\chi(y))V(\pi(y))} + \norm*{\chi(y)D_vV(\pi(y))} \\
            &\leq C_\chi\norm{u}\norm*{V(\pi(y))} + C_V\norm{v} \\
            &\leq C_\chi C_{1,V} + C_VC_\pi.
        \end{align*}
        Consequently,
        \begin{equation*}
            L_V \defeq \sup_{ y \in \R^n}\norm*{D\widetilde{V}(y)}_{\op{op}} \leq C_\chi C_{1,V} + C_VC_\pi < \infty,
        \end{equation*}
        and hence $\widetilde{V} \colon \R^n \sto \R^n$ is globally Lipschitz continuous, i.e.,
        \begin{equation*}
            \norm*{\widetilde{V}(y_1) - \widetilde{V}(y_2)} \leq L_V\norm*{y_1 - y_2}, \quad \forall~y_1,~y_2 \in \R^n.
        \end{equation*}
    \end{itemize}
\end{proof}

\subsection{Proof of Lemma \ref{lem:global_lip_ext_of_A}}\label{appen:proof_of_lemma_global_lip_ext_of_a}

\begin{proof}[Proof of Lemma \ref{lem:global_lip_ext_of_A}]
    Assume $\dim \M = m$. In the following, for any $u \in \R^n$, we denote by $u^\top(x),u^\perp(x)$ the tangential and orthogonal components of $u$ with respect to $T_x\M$. The proof consists of the following steps.
    \begin{itemize}
        \item Boundedness of $A$: First, by Assumption \ref{assum:geometric_iota}, Lemma \ref{lem:assum_to_bound_second_form} gives
        \begin{equation*}
            \kappa_1 \defeq \sup_{x\in \M} \norm{\mathrm{II}_x}_{\op{op}} < \infty.
        \end{equation*}
        Hence, by Corollary \ref{cor:boundedness_of_A},
        \begin{equation}\label{eq:proof_ext_a_bound_0}
            C_A^0 \defeq \sup_{x \in \M}\norm*{A(x)} < \infty.
        \end{equation}

        \item Extrinsic derivative: Fix $x \in \M$. We may choose a local orthonormal frame $\bb{E_i}_{i = 1}^m$ around $x$ such that $\nabla_{E_j}E_i(x) = 0$ for any $i,j = 1,\ldots,m$ \citep{lee2018introduction}. Let $v \in T_xM$ with $\norm{v} = 1$. Then $\nabla_v E_i(x) = 0$ for all $i$. Moreover,
        \begin{equation}\label{eq:proof_ext_a_Dv_second}
            \begin{aligned}
                D_v\mathrm{II}_x\bc{E_i(x),E_i(x)} &= \bc{D_v\mathrm{II}_x\bc{E_i(x),E_i(x)}}^\top + (D_v\mathrm{II}_x\bc{E_i(x),E_i(x)})^\perp \\
                &= -S_{\mathrm{II}\bc{E_i,E_i}}(v) + D_v^\perp \mathrm{II}_x\bc{E_i(x),E_i(x)},
            \end{aligned}
        \end{equation}
        where the second equality follows from the Weingarten Equation \citep[Proposition 8.4]{lee2018introduction}. For the first term on the right-hand side of (\ref{eq:proof_ext_a_Dv_second}), since $\kappa_1 = \sup_\M \norm{\mathrm{II}}_{\op{op}} < \infty$, Lemma \ref{lem:bound_shape_operator} implies
        \begin{equation}\label{eq:proof_ext_a_Dv_second_1}
            \norm*{S_{\mathrm{II}\bc{E_i,E_i}}(v)} \leq \kappa_1 \norm{\mathrm{II}_x\bc{E_i(x),E_i(x)}} \leq \kappa_1^2.
        \end{equation}
        For the second term on the right-hand side of (\ref{eq:proof_ext_a_Dv_second}), by (\ref{eq:def_nabla2_second}) and $\nabla_v E_i(x) = 0$,
        \begin{equation*}
            D_v^\perp \mathrm{II}_x\bc{E_i(x),E_i(x)} = \bc{\nabla_v \mathrm{II}_x}\bc{E_i(x),E_i(x)}.
        \end{equation*}
        Since $\kappa_2 \defeq \sup_\M \norm{\nabla \mathrm{II}}_{\op{op}} < \infty$ (Assumption \ref{assum:geometric_iota} and Lemma \ref{lem:assum_to_bound_second_form}),
        \begin{equation}\label{eq:proof_ext_a_Dv_second_2}
            \norm*{D_v^\perp \mathrm{II}_x\bc{E_i(x),E_i(x)}} \leq \kappa_2.
        \end{equation}
        Combining (\ref{eq:proof_ext_a_Dv_second_1}), (\ref{eq:proof_ext_a_Dv_second_2}), and (\ref{eq:proof_ext_a_Dv_second}), we obtain
        \begin{equation*}
            \norm*{D_v\mathrm{II}_x\bc{E_i(x),E_i(x)}} \leq \kappa_1^2 + \kappa_2,
        \end{equation*}
        and thus, by (\ref{eq:formula_A_second_form}),
        \begin{equation*}
            \norm*{D_vA(x)} \leq \frac{1}{2}\sum_{i=1}^m \norm*{D_v\mathrm{II}_x\bc{E_i(x),E_i(x)}} \leq \frac{1}{2}m\bc{\kappa_1^2 + \kappa_2}.
        \end{equation*}
        Therefore,
        \begin{equation}\label{eq:proof_ext_a_bound}
            C_A \defeq \sup_{x \in \M}\sup_{v\in T_xM,\norm{v}=1} \norm*{D_vA(x)} \leq \frac{1}{2}m\bc{\kappa_1^2 + \kappa_2} < \infty.
        \end{equation}

        \item Global extension: As in the proof of Lemma \ref{lem:global_lip_ext_of_Vec_and_Proj} (Appendix \ref{appen:proof_of_lemma_global_lip_vec_and_proj}), choose a uniform tubular neighborhood $\mathcal{V}$ with canonical projection $\pi$ satisfying (\ref{eq:proof_ext_bound_Dpi}), and a uniform tubular neighborhood $\clo{\mathcal{V}_0} \subset \mathcal{V}$ with a bump function $\chi$ satisfying (\ref{eq:proof_ext_bound_chi_0}) and (\ref{eq:proof_ext_bound_chi}). Define
        \begin{equation*}
            \widetilde{A}(y) \defeq \chi(y) A(\pi(y)),\quad \forall~ y\in \R^n,
        \end{equation*}
        which is well-defined.

        \item Global Lipschitz: For any $y \in \R^n$ and any $u \in \R^n$ with $\norm{u} = 1$, let $v = D\pi(y)[u] \in T_{\pi(y)}\M$. Arguing as in the proof of Lemma \ref{lem:global_lip_ext_of_Vec_and_Proj}, and using (\ref{eq:proof_ext_bound_Dpi}), (\ref{eq:proof_ext_bound_chi}), (\ref{eq:proof_ext_a_bound_0}), and (\ref{eq:proof_ext_a_bound}),
        \begin{align*}
            \norm*{D_u\widetilde{A}(y)}
            &\leq \norm*{(D_u\chi(y))A(\pi(y))} + \norm*{\chi(y)D_vA(\pi(y))} \\
            &\leq C_\chi\norm{u}\norm*{A(\pi(y))} + C_A\norm{v} \\
            &\leq C_\chi C_A^0 + C_AC_\pi.
        \end{align*}
        Therefore,
        \begin{equation*}
            L_A \defeq \sup_{y \in \R^n} \norm*{D\widetilde{A}(y)}_{\op{op}} \leq  C_\chi C_A^0 + C_AC_\pi < \infty,
        \end{equation*}
        and $\widetilde{A}$ is $L_A$-Lipschitz continuous.
    \end{itemize}
\end{proof}

\section{Euclidean EM Scheme}\label{appen:euler_maruyama_algorithm_on_euclidean_space}

Consider the Euclidean SDE
\begin{equation}\label{eq:general_eucl_SDE}
    \dx{\bm{X}_t} = b(t,\bm{X}_t)~\dx{t} + \sigma(t,\bm{X}_t)~\dx{\bm{W}_t},\quad t \in [0,T],
\end{equation}
where $b \colon [0,T]\times\R^n \to \R^n$, $\sigma \colon [0,T]\times\R^n \to \R^{n\times n}$, and $(\bm{W}_t)_{t\in[0,T]}$ is a standard Brownian motion in $\R^n$. To approximate (\ref{eq:general_eucl_SDE}) by a discrete-time scheme, fix a stepsize $h=T/N$ with $N\in\N$ and set $t_k = kh$ for $k=0,1,\ldots,N$. The Euler--Maruyama (EM) method defines $\{\bm{X}^h_k\}_{k=0}^N$ by $\bm{X}^h_0=\bm{X}_0$ and
\begin{equation}\label{eq:general_eucl_EM}
    \bm{X}^h_{k+1} = \bm{X}^h_k + b(t_k,\bm{X}^h_k)h + \sigma(t_k,\bm{X}^h_k)\Delta \bm{W}_k,
\end{equation}
where $\Delta \bm{W}_k = \bm{W}_{t_{k+1}}-\bm{W}_{t_k}\sim \mathcal{N}(0,hI_n)$ (equivalently, one samples $\bm{\xi}\sim\mathcal{N}(0,I_n)$ and sets $\Delta \bm{W}_k=\sqrt{h}\,\bm{\xi}$).

It is classical (see, e.g., \citet{kloeden1992numerical,milstein2013stochastic}) that EM converges strongly with order $1/2$. More precisely, if $\bm{X}_0\in L^2$ and $b,\sigma$ are globally Lipschitz in space and $1/2$-H\"older in time, i.e.,
\begin{equation}\label{eq:lip_hoder_for_EM}
    \begin{aligned}
        &\norm*{b(t,x) - b(t,y)} + \norm*{\sigma(t,x)-\sigma(t,y)}_{\op{F}} \leq L\norm*{x - y},\\
        &\norm*{b(t,x) - b(s,x)} + \norm*{\sigma(t,x)-\sigma(s,x)}_{\op{F}} \leq H(1+\norm{x})\abs{t-s}^{1/2},
    \end{aligned}
\end{equation}
then there exists $C>0$ such that
\begin{equation*}
    \E\bj{\max_{0 \leq k \leq N}\norm*{\bm{X}_{t_k}-\bm{X}_k^h}^2} \leq Ch.
\end{equation*}

Using the Burkholder--Davis--Gundy (BDG) inequality, this $L^2$ strong error estimate can be extended to an $L^p$ estimate for any $p\ge1$ under the same assumptions. Since this extension is not always stated explicitly in standard references, we provide a self-contained proof below; the argument also motivates the proof strategy for our main results.

\begin{thm}\label{thm:p_conver_eucl_EM}
    Let $1 \leq p < \infty$. For the SDE (\ref{eq:general_eucl_SDE}), assume that $\E\bj{\norm{\bm{X}_0}^p} < \infty$ and that $b,\sigma$ satisfy (\ref{eq:lip_hoder_for_EM}). Let $h = T/N$ for $N \in \N$ and $t_k = kh$ for $k = 0,1,\ldots,N$. Let $\{\bm{X}^h_k\}_{k=0}^N$ be constructed by the EM scheme (\ref{eq:general_eucl_EM}) with $\bm{X}^h_0 = \bm{X}_0$. Then there exists a constant $C_p>0$ such that
    \begin{equation*}
        \E\bj{\max_{0 \leq k \leq N}\norm*{\bm{X}_{t_k}-\bm{X}_k^h}^p} \leq C_p h^{p/2}.
    \end{equation*}
\end{thm}
\begin{proof}
    We first assume $p \geq 2$. Note that (\ref{eq:lip_hoder_for_EM}) implies that there exists a constant $K>0$ such that
    \begin{equation}\label{eq:linear_grow_eucl_EM}
        \norm*{b(t,x)} + \norm*{\sigma(t,x)}_{\op{F}} \leq K(1+\norm{x}).
    \end{equation}
    The proof consists of the following steps.
    \begin{itemize}
        \item Step $1$. Moment boundedness of $\bm{X}^h_k$: By (\ref{eq:general_eucl_EM}),
        \begin{align*}
            \bm{X}^h_{k}
            &= \bm{X}^h_{k-1} + b(t_{k-1},\bm{X}^h_{k-1})h + \sigma(t_{k-1},\bm{X}^h_{k-1})\Delta \bm{W}_{k-1} \\
            &= \bm{X}_0 + h\sum_{s = 0}^{k-1}b(t_s,\bm{X}^h_s) + \sum_{s = 0}^{k-1}\sigma(t_s,\bm{X}^h_s)\Delta \bm{W}_s \\
            &= \bm{X}_0 + \bm{A}_{k} + \bm{M}_{k},
        \end{align*}
        where
        \begin{equation*}
            \bm{A}_k \defeq h\sum_{s = 0}^{k-1}b(t_s,\bm{X}^h_s),\quad
            \bm{M}_k \defeq \sum_{s = 0}^{k-1}\sigma(t_s,\bm{X}^h_s)\Delta \bm{W}_s.
        \end{equation*}
        For any $\ell \in \N$, define the stopping time
        \begin{equation*}
            \bm{N}_\ell \defeq \min\bb{k > 0 \colon \norm*{\bm{X}^h_k} > \ell},
        \end{equation*}
        and consider the stopped process $\bb{\bm{X}^h_{k \wedge \bm{N}_\ell}}_{k=0}^N$, i.e.,
        \begin{equation*}
            \bm{X}^h_{k \wedge \bm{N}_\ell}
            = \bm{X}_0 + \bm{A}_{k \wedge \bm{N}_\ell} + \bm{M}_{k \wedge \bm{N}_\ell}.
        \end{equation*}
        Therefore, by H\"older's inequality,
        \begin{equation*}
            \norm*{\bm{X}^h_{k \wedge \bm{N}_\ell}}^p
            \leq 3^{p-1}\bc{\norm*{\bm{X}_0}^p + \norm*{\bm{A}_{k \wedge \bm{N}_\ell}}^p + \norm*{\bm{M}_{k \wedge \bm{N}_\ell}}^p}.
        \end{equation*}
        Consequently,
        \begin{equation}\label{eq:eucl_EM_moment_bound_0}
            \begin{aligned}
                \E\bj{\max_{0 \leq m \leq k \wedge \bm{N}_\ell} \norm*{\bm{X}^h_{m}}^p}
            &\leq 3^{p-1}\Big(\E\bj{\norm*{\bm{X}_0}^p}+ \E\bj{\max_{0 \leq m \leq k \wedge \bm{N}_\ell} \norm*{\bm{A}_{m}}^p} \\
            &\quad+ \E\bj{\max_{0 \leq m \leq k \wedge \bm{N}_\ell} \norm*{\bm{M}_{m}}^p}\Big).
            \end{aligned}
        \end{equation}
        By the linear growth bound for $b$ in (\ref{eq:linear_grow_eucl_EM}),
        \begin{align*}
            \max_{0 \leq m \leq k \wedge \bm{N}_\ell} \norm*{\bm{A}_m}^p
            &\leq h^p\max_{0 \leq m \leq k \wedge \bm{N}_\ell}\norm*{\sum_{s = 0}^{m-1}b(t_s,\bm{X}^h_s)}^p \\ 
            &\leq h^p\max_{0 \leq m \leq k \wedge \bm{N}_\ell}\bc{\sum_{s = 0}^{m-1}\norm*{b(t_s,\bm{X}^h_s)}}^p \\
            &\leq h^p\bc{\sum_{s = 0}^{(k \wedge \bm{N}_\ell)-1}\norm*{b(t_s,\bm{X}^h_s)}}^p \\
            &\leq h^p\bc{\sum_{s = 0}^{k-1}\norm*{b(t_{s \wedge \bm{N}_\ell},\bm{X}^h_{s \wedge \bm{N}_\ell})}}^p \\
            &\leq h^p k^{p-1}\sum_{s = 0}^{k-1}\norm*{b(t_{s \wedge \bm{N}_\ell},\bm{X}^h_{s \wedge \bm{N}_\ell})}^p \\
            &\leq h t_k^{p-1}\sum_{s = 0}^{k-1}K^p\bc{1+\norm*{\bm{X}^h_{s \wedge \bm{N}_\ell}}}^p \\
            &\leq h T^{p-1}K^p2^{p-1}\sum_{s = 0}^{k-1}\bc{1+\norm*{\bm{X}^h_{s \wedge \bm{N}_\ell}}^p} \\
            &\leq h T^{p-1}K^p2^{p-1}\sum_{s = 0}^{k-1}\bc{1+\max_{0 \leq m \leq s \wedge \bm{N}_\ell}\norm*{\bm{X}^h_{m}}^p}.
        \end{align*}
        Let
        \begin{equation*}
            a_s \defeq \E\bj{\max_{0 \leq m \leq s \wedge \bm{N}_\ell}\norm*{\bm{X}^h_{m}}^p}.
        \end{equation*}
        Then
        \begin{equation}\label{eq:eucl_EM_moment_bound_1}
            \E\bj{\max_{0 \leq m \leq k \wedge \bm{N}_\ell} \norm*{\bm{A}_{m}}^p}
            \leq h T^{p-1}K^p2^{p-1}\sum_{s = 0}^{k-1}\bc{1 + a_s}.
        \end{equation}
        For the last term in (\ref{eq:eucl_EM_moment_bound_0}), define
        \begin{equation*}
            \tilde{\bm{M}}_t \defeq \int_0^t \bm{H}_u~\dx{\bm{W}_u},\quad \bm{H}_u \defeq \sum_{s = 0}^{N-1} \sigma(t_s,\bm{X}^h_s)\mathbb{I}_{(t_s,t_{s+1}]}(u).
        \end{equation*}
        Then $\tilde{\bm{M}}$ is a continuous local martingale and $\tilde{\bm{M}}_{t_k}=\bm{M}_k$. By the BDG inequality \citep{marinelli2016maximal},
        \begin{align*}
            \E\bj{\max_{0 \leq m \leq k \wedge \bm{N}_\ell} \norm*{\bm{M}_{m}}^p}
            &\leq \E\bj{\sup_{0 \leq t \leq t_{k \wedge \bm{N}_\ell}} \norm{\tilde{\bm{M}}_{t}}^p}
            \leq B_p \E\bj{[\tilde{\bm{M}}]^{p/2}_{t_{k \wedge \bm{N}_\ell}}} \\
            &= B_p \E\bj{\bc{\sum_{s = 0}^{(k \wedge \bm{N}_\ell) - 1} h\norm*{\sigma(t_s,\bm{X}^h_s)}_{\op{F}}^2}^{p/2}}.
        \end{align*}
        Using (\ref{eq:linear_grow_eucl_EM}), we further have
        \begin{align*}
            \E\bj{\max_{0 \leq m \leq k \wedge \bm{N}_\ell} \norm*{\bm{M}_{m}}^p}
            &\leq B_pK^p \E\bj{\bc{\sum_{s = 0}^{(k \wedge \bm{N}_\ell) - 1} h\bc{1 + \norm*{\bm{X}^h_s}}^2}^{p/2}} \\
            &\leq B_pK^p \E\bj{\bc{\sum_{s = 0}^{k - 1} h\bc{1 + \norm*{\bm{X}^h_{s \wedge \bm{N}_\ell}}}^2}^{p/2}}.
        \end{align*}
        Since $p \geq 2$, Lemma \ref{lem:holder_p_geq_2} implies
        \begin{align*}
            \E\bj{\max_{0 \leq m \leq k \wedge \bm{N}_\ell} \norm*{\bm{M}_{m}}^p}
            &\leq B_pK^p T^{\frac{p}{2}-1}\E\bj{\sum_{s = 0}^{k - 1} h\bc{1 + \norm*{\bm{X}^h_{s \wedge \bm{N}_\ell}}}^p} \\
            &\leq hB_pK^p T^{\frac{p}{2}-1}2^{p-1}\E\bj{\sum_{s = 0}^{k - 1}\bc{1 + \max_{0 \leq m \leq s \wedge \bm{N}_\ell}\norm*{\bm{X}^h_{m}}^p}} \\
            &\leq hB_pK^p T^{\frac{p}{2}-1}2^{p-1}\sum_{s = 0}^{k - 1}\bc{1 + a_s}.
        \end{align*}
        Combining the above inequality and (\ref{eq:eucl_EM_moment_bound_1}) with (\ref{eq:eucl_EM_moment_bound_0}), we obtain
        \begin{equation*}
            a_k \leq C_{1,p} + C_{2,p}h\sum_{s = 0}^{k - 1}\bc{1 + a_s},
        \end{equation*}
        where $C_{1,p} \defeq 3^{p-1}\E\bj{\norm*{\bm{X}_0}^p}$ and
        \begin{equation*}
            C_{2,p} \defeq T^{p-1}K^p6^{p-1} + B_p T^{\frac{p}{2}-1}K^p6^{p-1}.
        \end{equation*}
        Then by the discrete Gr\"onwall inequality \citep[Lemma 10.5]{thomee2007galerkin},
        \begin{equation*}
            a_N \leq (1+C_{1,p})e^{C_{2,p}T} - 1.
        \end{equation*}
        Consequently,
        \begin{equation*}
            \E\bj{\max_{0 \leq k \leq N \wedge \bm{N}_\ell} \norm*{\bm{X}^h_{k}}^p}
            \leq (1+C_{1,p})e^{C_{2,p}T} - 1.
        \end{equation*}
        As $\ell \sto \infty$, Fatou's lemma \citep{folland2013real} implies
        \begin{equation}\label{eq:bound_disc_eucl_EM}
            \E\bj{\max_{0 \leq k \leq N} \norm*{\bm{X}^h_{k}}^p} \leq M_p,
        \end{equation}
        for some constant $M_p$ depending only on $p$.

        \item Step $2$. Boundedness of $\bm{X}_t$: First, by using a similar argument as in Step $1$ (or see the proof of Theorem 6.1 in \citet{baudoin2014diffusion} and replacing the It\^o's isometry by the BDG inequality), we have
        \begin{equation}\label{eq:bound_conti_eucl_EM}
            \E\bj{\sup_{0 \leq t \leq T} \norm*{\bm{X}_t}^p} \leq M^\prime_p.
        \end{equation}
        Moreover, by (\ref{eq:general_eucl_SDE}), for $t \geq s$,
        \begin{equation*}
            \bm{X}_t - \bm{X}_s = \int_s^t b(u,\bm{X}_u)~\dx{u} + \int_s^t\sigma(u,\bm{X}_u)~\dx{\bm{W}_u}.
        \end{equation*}
        Therefore,
        \begin{equation}\label{eq:boud_diff_conti_eucl_EM_0}
            \norm*{\bm{X}_t - \bm{X}_s}^p \leq 2^{p-1}\bc{\norm*{\int_s^t b(u,\bm{X}_u)~\dx{u}}^p + \norm*{\int_s^t\sigma(u,\bm{X}_u)~\dx{\bm{W}_u}}^p}.
        \end{equation}
        For the first term on the right-hand side, by H\"older's inequality and (\ref{eq:linear_grow_eucl_EM}),
        \begin{align*}
            \E\bj{\norm*{\int_s^t b(u,\bm{X}_u)~\dx{u}}^p}
            &\leq (t-s)^{p-1}\int_s^t \E\bj{\norm*{b(u,\bm{X}_u)}^p}~\dx{u} \\
            &\leq (t-s)^{p-1}K^p\int_s^t \E\bj{\bc{1+\norm*{\bm{X}_u}}^p}~\dx{u} \\
            &\leq (t-s)^{p-1}K^p2^{p-1}\int_s^t \bc{1+\E\bj{\norm*{\bm{X}_u}^p}}~\dx{u}.
        \end{align*}
        Using (\ref{eq:bound_conti_eucl_EM}), we obtain
        \begin{equation}\label{eq:boud_diff_conti_eucl_EM_1}
            \E\bj{\norm*{\int_s^t b(u,\bm{X}_u)~\dx{u}}^p} \leq 2^{p-1}K^p(1+M_p^\prime)(t-s)^p.
        \end{equation}
        For the second term in (\ref{eq:boud_diff_conti_eucl_EM_0}), by the BDG inequality,
        \begin{align*}
            \E\bj{\norm*{\int_s^t\sigma(u,\bm{X}_u)~\dx{\bm{W}_u}}^p}
            &\leq B_p^\prime \E\bj{\bc{\int_s^t\norm*{\sigma(u,\bm{X}_u)}_{\op{F}}^2~\dx{u}}^{p/2}} \\
            &\leq B_p^\prime (t-s)^{\frac{p}{2}-1}\E\bj{\int_s^t \norm*{\sigma(u,\bm{X}_u)}_{\op{F}}^p~\dx{u}},
        \end{align*}
        where the last inequality follows from Lemma \ref{lem:holder_p_geq_2} for $p \geq 2$. Then (\ref{eq:linear_grow_eucl_EM}) and (\ref{eq:bound_conti_eucl_EM}) imply
        \begin{equation}\label{eq:boud_diff_conti_eucl_EM_2}
            \E\bj{\norm*{\int_s^t\sigma(u,\bm{X}_u)~\dx{\bm{W}_u}}^p} \leq 2^{p-1}K^pB_p^\prime(1+M_p^\prime)(t-s)^{p/2}.
        \end{equation}
        Combining (\ref{eq:boud_diff_conti_eucl_EM_2}), (\ref{eq:boud_diff_conti_eucl_EM_1}), and (\ref{eq:boud_diff_conti_eucl_EM_0}), when $t-s<1$ there exists $L_p$ depending only on $p$ such that
        \begin{equation}\label{eq:boud_diff_conti_eucl_EM}
            \E\bj{\norm*{\bm{X}_t - \bm{X}_s}^p } \leq L_p(t-s)^{p/2}.
        \end{equation}

        \item Step $3$. Comparison of $\bm{X}_{t_k}$ and $\bm{X}^h_k$: By (\ref{eq:general_eucl_SDE}),
        \begin{equation*}
            \bm{X}_{t_{k+1}} = \bm{X}_{t_k} + \int_{t_k}^{t_{k+1}} b(u,\bm{X}_u)~\dx{u} + \int_{t_k}^{t_{k+1}}\sigma(u,\bm{X}_u)~\dx{\bm{W}_u}.
        \end{equation*}
        Let $\bm{e}_k \defeq \bm{X}_{t_k} - \bm{X}_k^h$. Then
        \begin{equation*}
            \bm{e}_{k+1} = \bm{e}_k + \int_{t_k}^{t_{k+1}} \bc{b(u,\bm{X}_u) - b\bc{t_k,\bm{X}^h_k}} ~\dx{u} + \int_{t_k}^{t_{k+1}}\bc{\sigma(u,\bm{X}_u) - \sigma(t_k,\bm{X}_k^h)}~\dx{\bm{W}_u}.
        \end{equation*}
        Since $\bm{e}_0=0$, we can write
        \begin{align*}
            \bm{e}_k &= \underbrace{\sum_{s=0}^{k-1} \int_{t_s}^{t_{s+1}} \bc{b(u,\bm{X}_u) - b\bc{t_s,\bm{X}^h_s}} ~\dx{u}}_{\eqdef \bm{B}_k} \\
            &\quad+ \underbrace{\sum_{s=0}^{k-1}\int_{t_s}^{t_{s+1}}\bc{\sigma(u,\bm{X}_u) - \sigma(t_s,\bm{X}_s^h)}~\dx{\bm{W}_u}}_{\eqdef \bm{G}_k}.
        \end{align*}
        Therefore,
        \begin{equation*}
            \max_{0 \leq m \leq k} \norm*{\bm{e}_m}^p \leq 2^{p-1}\bc{\max_{0 \leq m \leq k}\norm*{\bm{B}_m}^p + \max_{0 \leq m \leq k}\norm*{\bm{G}_m}^p}.
        \end{equation*}
        Taking expectations gives
        \begin{equation}\label{eq:bound_ek_eucl_EM_0}
            \E\bj{\max_{0 \leq m \leq k} \norm*{\bm{e}_m}^p} \leq 2^{p-1}\bc{\E\bj{\max_{0 \leq m \leq k}\norm*{\bm{B}_m}^p}     + \E\bj{\max_{0 \leq m \leq k}\norm*{\bm{G}_m}^p} }.
        \end{equation}
        where the existence on the right-hand side is given by (\ref{eq:bound_disc_eucl_EM}) and (\ref{eq:bound_conti_eucl_EM}). For the first term on the right-hand side of (\ref{eq:bound_ek_eucl_EM_0}),
        \begin{equation}\label{eq:Bk_bound_eucl_EM_0}
            \begin{aligned}
                \norm*{\bm{B}_m}^p
                &\leq m^{p-1}\sum_{s = 0}^{m-1} \norm*{\int_{t_s}^{t_{s+1}} \bc{b(u,\bm{X}_u) - b\bc{t_s,\bm{X}^h_s}} ~\dx{u}}^p \\
                &\leq t_m^{p-1}\sum_{s = 0}^{m-1}\int_{t_s}^{t_{s+1}} \norm*{b(u,\bm{X}_u) - b\bc{t_s,\bm{X}^h_s}}^p ~\dx{u}.
            \end{aligned}
        \end{equation}
        By (\ref{eq:lip_hoder_for_EM}),
        \begin{equation}\label{eq:Bk_bound_eucl_EM_1}
            \begin{aligned}
                \norm*{b(u,\bm{X}_u) - b\bc{t_s,\bm{X}^h_s}}^p
                &\leq 3^{p-1}\Big(\norm*{b(u,\bm{X}_u) - b\bc{u,\bm{X}_{t_s}}}^p \\
                &\quad+ \norm*{b\bc{u,\bm{X}_{t_s}} - b\bc{u,\bm{X}^h_s}}^p \\
                &\quad + \norm*{b(u,\bm{X}^h_s) - b\bc{t_s,\bm{X}^h_s}}^p\Big) \\
                &\leq 3^{p-1}\Big( L^p\norm*{\bm{X}_u - \bm{X}_{t_s}}^p + L^p\norm*{\bm{e}_s}^p \\
                &\quad+ H^p\bc{1+\norm*{\bm{X}^h_s}}^p h^{p/2} \Big).
            \end{aligned}
        \end{equation}
        Substituting (\ref{eq:Bk_bound_eucl_EM_1}) into (\ref{eq:Bk_bound_eucl_EM_0}) and taking expectations, we obtain
        \begin{align*}
            \E\bj{\max_{0 \leq m \leq k}\norm*{\bm{B}_m}^p}
            &\leq T^{p-1} \sum_{s = 0}^{k-1} \int_{t_s}^{t_{s+1}} \E\bj{\norm*{b(u,\bm{X}_u) - b\bc{t_s,\bm{X}^h_s}}^p} ~\dx{u} \\
            &\leq \bc{3T}^{p-1} \sum_{s = 0}^{k-1} \int_{t_s}^{t_{s+1}} \Big( L^p \E\bj{\norm*{\bm{X}_u - \bm{X}_{t_s}}^p} + L^p\E\bj{\norm*{\bm{e}_s}^p}  \\
            &\qquad\qquad\qquad\qquad\qquad  + H^p2^{p-1}\bc{1 + \E\bj{\norm*{\bm{X}^h_s}^p}}h^{p/2}\Big)~\dx{u} \\
            &\leq 3^{p-1}T^{p}\bc{L^pL_p+ H^p2^{p-1}(1+M_p)}h^{p/2} \\
            &\quad+ \bc{3T}^{p-1} L^p \sum_{s = 0}^{k-1} h \E\bj{\norm*{\bm{e}_s}^p},
        \end{align*}
        where the last inequality uses (\ref{eq:bound_disc_eucl_EM}) and (\ref{eq:boud_diff_conti_eucl_EM}). Hence there exist constants $C_{p,1},C_{p,2}>0$ such that
        \begin{equation}\label{eq:Bk_bound_eucl_EM_result}
            \E\bj{\max_{0 \leq m \leq k}\norm*{\bm{B}_m}^p} \leq C_{p,1}h^{p/2} + C_{p,2}\sum_{s = 0}^{k-1} h y_s,
        \end{equation}
        where
        \begin{equation*}
            y_s \defeq \E\bj{\max_{0\leq m \leq s} \norm*{\bm{e}_m}^p}.
        \end{equation*}
        For the second term on the right-hand side of (\ref{eq:bound_ek_eucl_EM_0}), by the BDG inequality there exists $B_p^\prime>0$ such that
        \begin{equation}\label{eq:Gk_bound_eucl_EM_0}
            \begin{aligned}
                \E\bj{\max_{0 \leq m \leq k}\norm*{\bm{G}_m}^p}
                &\leq B_p^\prime \E\bj{\bc{\sum_{s = 0}^{k-1} \int_{t_s}^{t_{s+1}}\norm*{\sigma(u,\bm{X}_u) - \sigma(t_s,\bm{X}_s^h)}_{\op{F}}^2 ~\dx{u}}^{p/2}} \\
                &\leq B_p^\prime \E\bj{ k^{\frac{p}{2} - 1} \sum_{s = 0}^{k-1} \bc{\int_{t_s}^{t_{s+1}}\norm*{\sigma(u,\bm{X}_u) - \sigma(t_s,\bm{X}_s^h)}_{\op{F}}^2 ~\dx{u}}^{p/2}} \\
                &\leq B_p^\prime T^{\frac{p}{2} - 1} \sum_{s = 0}^{k-1} \int_{t_s}^{t_{s+1}}\E\bj{\norm*{\sigma(u,\bm{X}_u) - \sigma(t_s,\bm{X}_s^h)}_{\op{F}}^p} ~\dx{u},
            \end{aligned}
        \end{equation}
        where the last two inequalities follow from H\"older's inequality and Lemma \ref{lem:holder_p_geq_2}. By the same decomposition as in (\ref{eq:Bk_bound_eucl_EM_1}) and using (\ref{eq:lip_hoder_for_EM}), we have
        \begin{align*}
            \E\bj{\norm*{\sigma(u,\bm{X}_u) - \sigma(t_s,\bm{X}_s^h)}_{\op{F}}^p}
            &\leq 3^{p-1}\Big(L^pL_ph^{p/2} + L^p\E\bj{\norm*{\bm{e}_s}^p} \\
            &\quad + 2^{p-1}H^p(1+M_p)h^{p/2}\Big).
        \end{align*}
        Substituting this bound into (\ref{eq:Gk_bound_eucl_EM_0}) gives
        \begin{equation}\label{eq:Gk_bound_eucl_EM_result}
            \E\bj{\max_{0 \leq m \leq k}\norm*{\bm{G}_m}^p} \leq C^\prime_{p,1}h^{p/2} + C^\prime_{p,2}\sum_{s = 0}^{k-1} h y_s,
        \end{equation}
        for some constants $C^\prime_{p,1},C^\prime_{p,2}>0$. Combining (\ref{eq:Bk_bound_eucl_EM_result}), (\ref{eq:Gk_bound_eucl_EM_result}), and (\ref{eq:bound_ek_eucl_EM_0}), we obtain
        \begin{equation*}
            y_k \leq C^{\prime\prime}_{p,1}h^{p/2} + C^{\prime\prime}_{p,2}\sum_{s = 0}^{k-1} h y_s,
        \end{equation*}
        for some constants $C^{\prime\prime}_{p,1},C^{\prime\prime}_{p,2}>0$ depending only on $p$. By the discrete Gr\"onwall inequality,
        \begin{equation*}
            y_N \leq C^{\prime\prime}_{p,1}e^{C^{\prime\prime}_{p,2}T}h^{p/2},
        \end{equation*}
        that is,
        \begin{equation*}
            \E\bj{\max_{0 \leq k \leq N} \norm*{\bm{X}_{t_k} - \bm{X}_k^h}^p} \leq C_ph^{p/2}.
        \end{equation*}
    \end{itemize}
    Finally, for $1 \leq p < 2$, by the Lyapunov's inequality \citep[Corollary 2.12.21]{bogachev2007measure},
    \begin{equation*}
        \E\bj{\max_{0 \leq k \leq N} \norm*{\bm{X}_{t_k} - \bm{X}_k^h}^p} \leq \E\bj{\max_{0 \leq k \leq N} \norm*{\bm{X}_{t_k} - \bm{X}_k^h}^2}^{p/2} \leq C_2^{p/2} h^{p/2}.
    \end{equation*}
\end{proof}
\begin{rmk}\label{rmk:p_conver_eucl_EM_time}
    If we allow $T$ to vary in (\ref{eq:general_eucl_SDE}), it is not hard to see that $M_p$ and $M^\prime_p$ in (\ref{eq:bound_disc_eucl_EM}) and (\ref{eq:bound_conti_eucl_EM}), which bound the moments of $\bm{X}^h_k$ and $\bm{X}_t$, satisfy $M_p,M^\prime_p = \mathcal{O}(\exp(T^p))$. This also implies that $L_p = \mathcal{O}(\exp(T^p))$ in (\ref{eq:boud_diff_conti_eucl_EM}). Therefore, the constant $C_p$ in the final estimate satisfies $C_p = \mathcal{O}(\exp(T^p))$, where $\mathcal{O}$ hides $L,H$ (\ref{eq:lip_hoder_for_EM}), $p$, and $\E\bj{\norm{\bm{X}_0}^p}$.
\end{rmk}

\section{Omitted Proofs in Section \ref{sec:p_strong_convergence_of_gem}}

In this section, we first prove Lemma \ref{lem:comparison_EM_GEM}, which controls the strong pathwise discrepancy between the extrinsic and intrinsic discretizations; see Section \ref{appen:proof_of_lemma_comparison_em_gem}. We also provide the omitted proofs of Theorem \ref{thm:intr_p_strong_conv_GEM} and Corollary \ref{cor:p_strong_conv_GEM_cpt}; see Section \ref{appen:proof_of_theorem_ref_thm_intr_p_strong_conv_gem} and Section \ref{appen:proof_of_corollary_ref_cor_p_strong_conv_gem_cpt}.

\subsection{Proof of Lemma \ref{lem:comparison_EM_GEM}}\label{appen:proof_of_lemma_comparison_em_gem}

Before starting the proof, we introduce some notation.

For any $y, u \in \R^n$, define
\begin{equation*}
    \Phi^h_E(y,u) \defeq y + hU(y) + P(y)u.
\end{equation*}
For the $\R^n$-Brownian motion $\bm{W}$ in SDE (\ref{eq:mfd_sde_eucl_w}), set
\begin{equation*}
    \Delta \bm{W}_k \defeq \bm{W}_{t_{k+1}} - \bm{W}_{t_k} \sim \mathcal{N}(0,hI_n).
\end{equation*}
Thus, in (\ref{eq:em_to_gem}) we may realize $\sqrt{h}\bm{\xi} \sim \mathcal{N}(0,hI_n)$ by
\begin{equation*}
    \sqrt{h}\bm{\xi} = \Delta \bm{W}_k.
\end{equation*}
Then (\ref{eq:em_to_gem}) can be written as
\begin{equation}\label{eq:pf_comp_def_Y}
    \bm{Y}^h_{k+1} = \Phi^h_E\bc{\bm{Y}^h_k,\Delta \bm{W}_k}.
\end{equation}

Similarly, for any $x \in \M$ and any $u \in \R^n$, define
\begin{equation*}
    \Phi^h_G(x,u) \defeq \exp_x\bc{hV(x)+ P(x)u},
\end{equation*}
where $\exp_x \colon T_x\M \sto \M$ is the exponential map. Since $P(x)$ is the orthogonal projection onto $T_x\M$ and $\Delta \bm{W}_k \sim \mathcal{N}(0,hI_n)$, the random variable $P(x)\Delta \bm{W}_k$ is a Gaussian on $T_x\M$ with mean $0$ and variance $h$ with respect to the metric $g(x)$ \citep{muirhead2009aspects}. Accordingly, we may realize $\sqrt{h}\bm{\xi}_x$ in (\ref{eq:gem}) by
\begin{equation*}
    \sqrt{h}\bm{\xi}_x = P(x)\Delta \bm{W}_k.
\end{equation*} 
Then (\ref{eq:gem}) can be written as
\begin{equation}\label{eq:pf_comp_def_X}
    \bm{X}^h_{k+1} = \Phi^h_G\bc{\bm{X}^h_k,\Delta \bm{W}_k}.
\end{equation}

Without loss of generality, assume that $\bm{X}_0$ is independent of the Brownian motion $\bm{W}$.
\begin{equation*}
    \mathcal{F}_t \defeq \sigma\bc{\bm{X}_0,\bm{W}_{s} \colon s \leq t},
\end{equation*}
which is also a $\bm{W}$-filtration \citep{karatzas1991brownian}. Let $\mathcal{F}_k \defeq \mathcal{F}_{t_k}$ for $k = 0,\ldots,N$. Then $\bm{X}^h_k$ and $\bm{Y}^h_k$ are $\mathcal{F}_k$-measurable. Moreover, since $\Delta \bm{W}_k$ is independent of $\mathcal{F}_k$, it follows that \citep{durrett2019probability}
\begin{equation}\label{eq:pf_comp_indep}
    \E\bj{f(\bm{X}^h_k,\Delta \bm{W}_k) \mid \mathcal{F}_k} = \E\bj{ f(\bm{X}^h_k,\Delta \bm{W}_k) \mid \sigma(\bm{X}^h_k)},
\end{equation}
for any integrable function $f$. The same identity holds with $\bm{X}^h_k$ replaced by $\bm{Y}^h_k$.

The main proof strategy for Lemma \ref{lem:comparison_EM_GEM} follows the similar outline as that for the EM algorithm (Theorem \ref{thm:p_conver_eucl_EM}). We first prove the case $p \geq 2$, and then apply Lyapunov's inequality to extend the result to all $1 \leq p < \infty$. 

\begin{proof}[Proof of Lemma \ref{lem:comparison_EM_GEM}]
    First, assume $p \geq 2$ and let $\bm{e}_{k} \defeq \bm{X}^h_k - \bm{Y}^h_k$.
    \begin{itemize}
        \item Step $1$. Martingale decomposition: Using the notation above and Lemma \ref{lem:pf_comp_one_step},
        \begin{align*}
            \bm{e}_{k+1} &= \Phi^h_G\bc{\bm{X}^h_k,\Delta \bm{W}_k} - \Phi^h_E\bc{\bm{Y}^h_k,\Delta \bm{W}_k} \\
            &= \Phi^h_E\bc{\bm{X}^h_k,\Delta \bm{W}_k} - \Phi^h_E\bc{\bm{Y}^h_k,\Delta \bm{W}_k} + \underbrace{\Phi^h_G\bc{\bm{X}^h_k,\Delta \bm{W}_k} - \Phi^h_E\bc{\bm{X}^h_k,\Delta \bm{W}_k}}_{= \bm{\Delta}_h(\bm{X}^h_k)} \\
            &= \bm{e}_k + h\bc{U(\bm{X}^h_k) - U(\bm{Y}^h_k)} + \bc{P(\bm{X}^h_k) - P(\bm{Y}^h_k)}\Delta \bm{W}_k + \bm{\Delta}_h(\bm{X}^h_k).
        \end{align*}
        Furthermore, by the uniform boundedness of $\bm{\Delta}_h$ (Lemma \ref{lem:pf_comp_one_step}), $\bm{\Delta}_h(\bm{X}^h_k) \in L^1$. So it can define
        \begin{equation*}
            \bm{\mu}_k \defeq \E\bj{\bm{\Delta}_h(\bm{X}^h_k) \mid \mathcal{F}_k},\quad \bm{m}^{(2)}_k \defeq \bm{\Delta}_h(\bm{X}^h_k) - \bm{\mu}_k.
        \end{equation*}
        Therefore,
        \begin{equation*}
            \bm{e}_{k+1} = \bm{e}_k + \underbrace{h\bc{U(\bm{X}^h_k) - U(\bm{Y}^h_k)} + \bm{\mu}_k}_{\eqdef \bm{a}_k} + \underbrace{\bc{P(\bm{X}^h_k) - P(\bm{Y}^h_k)}\Delta \bm{W}_k}_{\eqdef \bm{m}^{(1)}_k}+ \bm{m}^{(2)}_k.
        \end{equation*}
        Let
        \begin{equation*}
            \bm{A}_k \defeq \sum_{j=0}^{k-1}\bm{a}_j,\quad \bm{M}_k^{(1)} \defeq \sum_{j=0}^{k-1}\bm{m}^{(1)}_j,\quad \bm{M}_k^{(2)} \defeq \sum_{j=0}^{k-1}\bm{m}^{(2)}_j.
        \end{equation*}
        Note that both $\bm{M}_k^{(1)}$ and $\bm{M}_k^{(2)}$ are $\mathcal{F}_k$-martingales by definition. Since $\bm{e}_0 = 0$, we have
        \begin{equation*}
            \bm{e}_k = \bm{A}_k + \bm{M}_k,\quad \bm{M}_k \defeq \bm{M}_k^{(1)} + \bm{M}_k^{(2)}.
        \end{equation*}
        It follows that
        \begin{equation}\label{pf_comp_total_bound}
            \E\bj{ \max_{0 \leq m \leq k} \norm*{\bm{e}_m}^p} \leq 2^{p-1}\bc{ \E\bj{ \max_{0 \leq m \leq k} \norm*{\bm{A}_m}^p} + \E\bj{ \max_{0 \leq m \leq k} \norm*{\bm{M}_m}^p}},
        \end{equation}
        where
        \begin{equation*}
            S_k \defeq \E\bj{ \max_{0 \leq m \leq k} \norm*{\bm{e}_m}^p} < \infty
        \end{equation*}
        because $\E\bj{ \max_{m} \norm*{\bm{Y}^h_m}^p} < \infty$ by Theorem \ref{thm:p_conver_eucl_EM} and $\E\bj{ \max_{m} \norm*{\bm{X}^h_m}^p} < \infty$ by Lemma \ref{lem:pf_comp_moment}.

        \item Step $2$. Boundedness of $\bm{A}_k$: We treat the two terms in $\bm{a}_k$ separately. For the $U$ term, by Section \ref{sub:well_posed_extension_on_euclidean_space}, $U$ is $L_U$-Lipschitz continuous on $\R^n$, and hence
        \begin{equation}\label{pf_comp_lip_U}
            \norm*{U(\bm{X}^h_k) - U(\bm{Y}^h_k)}^p \leq L_U^p\norm*{\bm{e}_k}^p.
        \end{equation}
        For the other term, by Lemma \ref{lem:pf_comp_one_step} and (\ref{eq:pf_comp_indep}), there exists a uniform constant $C_\mu > 0$ such that
        \begin{equation}\label{pf_comp_bound_mu}
            \norm*{\bm{\mu}_k} = \norm*{\E\bj{\bm{\Delta}_h(\bm{X}^h_k) \mid \sigma(\bm{X}^h_k)}} = \norm*{\mu_h(\bm{X}^h_k)} \leq C_\mu h^{3/2}. 
        \end{equation}
        Combining (\ref{pf_comp_lip_U}) and (\ref{pf_comp_bound_mu}),
        \begin{align*}
            \max_{0 \leq m \leq k} \norm*{\bm{A}_m}^p
            &\leq \max_{0 \leq m \leq k}\bc{\sum_{j=0}^{m-1}h\norm*{U(\bm{X}^h_j) - U(\bm{Y}^h_j)}  + \sum_{j=0}^{m-1} \norm*{\bm{\mu}_j}}^p \\
            &\leq 2^{p-1}\bc{\sum_{j=0}^{k-1}h\norm*{U(\bm{X}^h_j) - U(\bm{Y}^h_j)}}^p+ 2^{p-1}\bc{\sum_{j=0}^{k-1}\norm*{\bm{\mu}_j}}^p \\
            &\leq 2^{p-1} k^{p-1}h^p \sum_{j=0}^{k-1}\norm*{U(\bm{X}^h_j) - U(\bm{Y}^h_j)}^p + 2^{p-1}\bc{\sum_{j=0}^{N-1}\norm*{\bm{\mu}_j}}^p \\
            &\leq (2T)^{p-1}L_U^p\sum_{j=0}^{k-1}h \norm*{\bm{e}_j}^p + 2^{p-1} C_\mu^pT^p h^{p/2}.
        \end{align*}
        It follows that
        \begin{equation}\label{pf_comp_bound_Ak}
            \E\bj{ \max_{0 \leq m \leq k} \norm*{\bm{A}_m}^p} \leq (2T)^{p-1}L_U^p\sum_{j=0}^{k-1}h S_j + 2^{p-1} C_\mu^pT^p h^{p/2},
        \end{equation}
        where we used $\E\bj{\norm*{\bm{e}_j}^p} \leq \E\bj{ \max_{0 \leq m \leq j} \norm*{\bm{e}_m}^p}$. 

        \item Step $3$. Boundedness of $\bm{M}_k$: By definition,
        \begin{equation}\label{pf_comp_bound_Mk_0}
            \E\bj{ \max_{0 \leq m \leq k} \norm*{\bm{M}_m}^p} \leq 2^{p-1} \bc{\E\bj{ \max_{0 \leq m \leq k} \norm*{\bm{M}^{(1)}_m}^p}+\E\bj{ \max_{0 \leq m \leq k} \norm*{\bm{M}^{(2)}_m}^p}}.
        \end{equation}
        For $\bm{M}^{(1)}$, similarly to the proof for the Euclidean EM algorithm (Theorem \ref{thm:p_conver_eucl_EM}), the BDG inequality implies that
        \begin{equation*}
            \E\bj{ \max_{0 \leq m \leq k} \norm*{\bm{M}^{(1)}_m}^p} \leq C^{(1)}_p \E\bj{\bc{\sum_{j=0}^{k-1}h\norm*{P(\bm{X}^h_k) - P(\bm{Y}^h_k)}^2_{\op{F}}}^{p/2}}.
        \end{equation*}
        As shown in Section \ref{sub:well_posed_extension_on_euclidean_space}, $P$ is extended on $\R^n$ as an $L_P$-Lipschitz map. Combining this with Lemma \ref{lem:holder_p_geq_2} for $p \geq 2$,
        \begin{align*}
            \bc{\sum_{j=0}^{k-1}h\norm*{P(\bm{X}^h_k) - P(\bm{Y}^h_k)}^2_{\op{F}}}^{p/2} &\leq t_k^{\frac{p}{2} - 1}\sum_{j=0}^{k-1}h \norm*{P(\bm{X}^h_k) - P(\bm{Y}^h_k)}^p_{\op{F}} \\ 
            &\leq T^{\frac{p}{2} - 1}L_P^p \sum_{j=0}^{k-1}h\norm*{\bm{e}_j}^p.
        \end{align*}
        Therefore,
        \begin{equation}\label{pf_comp_bound_M1}
            \E\bj{ \max_{0 \leq m \leq k} \norm*{\bm{M}^{(1)}_m}^p} \leq  C^{(1)}_p T^{\frac{p}{2} - 1}L_P^p \sum_{j=0}^{k-1}hS_j.
        \end{equation}

        Next, for $\bm{M}^{(2)}$, by the discrete-time (vector-valued) BDG inequality \citep[Corollary 3.2]{zorin2022martingale},
        \begin{equation*}
            \E\bj{\max_{0 \leq m \leq k}\norm*{\bm{M}^{(2)}_m}^p} \leq C^{(2)}_p\E\bj{[\bm{M}^{(2)}]_k^{p/2}},
        \end{equation*}
        where
        \begin{equation*}
            [\bm{M}^{(2)}]_k = \sum_{j=0}^{k-1} \norm*{\bm{m}^{(2)}_j}^2.
        \end{equation*}
        Therefore, by Lemma \ref{lem:holder_p_geq_2} for $p \geq 2$,
        \begin{equation*}
            \E\bj{\max_{0 \leq m \leq k}\norm*{\bm{M}^{(2)}_m}^p} \leq C^{(2)}_p k^{\frac{p}{2} - 1}\sum_{j=0}^{k-1} \E\bj{\norm*{\bm{m}^{(2)}_j}^p}.
        \end{equation*}
        Moreover, by (\ref{eq:pf_comp_indep}), $\bm{m}^{(2)}_k = \widetilde{\bm{\Delta}}_h(\bm{X}^h_k)$ as in Lemma \ref{lem:pf_comp_one_step}, which also implies that
        \begin{equation*}
            \E\bj{\norm*{\bm{m}^{(2)}_k}^p} = \E\bj{\norm*{\widetilde{\bm{\Delta}}_h(\bm{X}^h_k)}^p} \leq C_\delta h^p.
        \end{equation*}
        Thus,
        \begin{equation}\label{pf_comp_bound_M2}
            \E\bj{\max_{0 \leq m \leq k}\norm*{\bm{M}^{(2)}_m}^p} \leq C^{(2)}_p k^{\frac{p}{2} - 1}\sum_{j=0}^{k-1}C_\delta h^p \leq C^{(2)}_pC_\delta  T^{p/2}h^{p/2}.
        \end{equation}

        Finally, combining (\ref{pf_comp_bound_M1}), (\ref{pf_comp_bound_M2}), and (\ref{pf_comp_bound_Mk_0}), we have
        \begin{equation}\label{pf_comp_bound_Mk}
            \E\bj{ \max_{0 \leq m \leq k} \norm*{\bm{M}_m}^p} \leq C^{(1)}_p T^{\frac{p}{2} - 1}L_P^p \sum_{j=0}^{k-1}hS_j + C^{(2)}_pC_\delta  T^{p/2}h^{p/2}.
        \end{equation}

        \item Step $4$. Final bound: Combining (\ref{pf_comp_bound_Ak}), (\ref{pf_comp_bound_Mk}), and (\ref{pf_comp_total_bound}), there exist constants $C_{1,p}$ and $C_{2,p}$ such that
        \begin{equation*}
            S_k \leq C_{1,p}T^{p}h^{p/2}+ C_{2,p}T^{p-1}\sum_{j=0}^{k-1}hS_j,
        \end{equation*}
        where we assume $T > 1$ so that $T^{p/2} \leq T^{p}$. Then, by the discrete Gr\"onwall inequality \citep[Lemma 10.5]{thomee2007galerkin},
        \begin{equation*}
            S_N \leq C_{1,p}T^{p}\exp\bc{C_{2,p}T^{p}} h^{p/2}.
        \end{equation*}
        Therefore, for $G_p(T) = C_{1,p}T^{p}\exp\bc{C_{2,p}T^{p}} = \mathcal{O}(\exp(T^p))$, we have
        \begin{equation*}
            \E\bj{ \max_{0 \leq k \leq N} \norm*{\bm{X}^h_k - \bm{Y}^h_k}^p} \leq G_p(T)h^{p/2}. 
        \end{equation*}
    \end{itemize}
    Finally, for $1 \leq p < \infty$, by Lyapunov's inequality \citep[Corollary 2.12.21]{bogachev2007measure},
    \begin{equation*}
        \E\bj{ \max_{0 \leq k \leq N} \norm*{\bm{X}^h_k - \bm{Y}^h_k}^p} \leq \E\bj{ \max_{0 \leq k \leq N} \norm*{\bm{X}^h_k - \bm{Y}^h_k}^2}^{p/2} \leq \mathcal{O}(\exp(T^p))h^{p/2}.
    \end{equation*}
\end{proof}

\begin{lem}\label{lem:pf_comp_one_step}
    Assume that $\M$ satisfies Assumption \ref{assum:geometric_iota} and Assumption \ref{assum:uniform_tubular}, and that $V \in \Gamma(T\M)$ satisfies Assumption \ref{assum:vector_field}. Let $1 \leq p < \infty$. Using the notation above, for any $x \in \M$, define
    \begin{equation*}
        \bm{\Delta}_h(x) \defeq \Phi^h_G\bc{x,\Delta \bm{W}_k} - \Phi^h_E\bc{x,\Delta \bm{W}_k},\quad \mu_h(x) \defeq \E\bj{\bm{\Delta}_h(x)}.
    \end{equation*}
    If $h \leq 1$, the following statements hold.
    \begin{enumerate}[label=(\roman*)]
        \item There exists a constant $C_{\mu} > 0$ such that
        \begin{equation*}
            \sup_{x\in\M}\norm{\mu_h(x)} \leq C_{\mu}h^{3/2}.
        \end{equation*}

        \item Define the centered error
        \begin{equation*}
            \widetilde{\bm{\Delta}}_h(x) \defeq \bm{\Delta}_h(x) - \mu_h(x).
        \end{equation*}
        Then there exists a constant $C_{\delta} > 0$ such that
        \begin{equation*}
            \sup_{x\in\M}\E\bj{\norm*{\widetilde{\bm{\Delta}}_h(x)}^p} \leq C_{\delta}h^p.
        \end{equation*}
    \end{enumerate}
\end{lem}
\begin{proof}
    Fix $x \in \M$. First, let
    \begin{equation*}
        \bm{v} = hV(x) + P(x)\Delta \bm{W}_k = hV(x)+\sqrt{h}\bm{\xi}_x  \in T_x\M,
    \end{equation*}
    where $\bm{\xi}_x$ is a standard Gaussian on $T_x\M$ with $P(x)\Delta \bm{W}_k = \sqrt{h}\bm{\xi}_x$. By Assumption \ref{assum:vector_field},
    \begin{equation}\label{eq:pf_comp_one_step_bound_CV}
        C_{V} \defeq \sup_{x \in\M}\norm{V(x)} < \infty.
    \end{equation}
    Hence,
    \begin{equation*}
        \norm{\bm{v}} \leq h\norm{V(x)} + \norm*{P(x)\Delta \bm{W}_k} \leq C_Vh + \sqrt{h}\norm*{\bm{\xi}_x}.
    \end{equation*}
    Therefore, for any $1 \leq q < \infty$, 
    \begin{equation*}
        \E\bj{\norm{\bm{v}}^q} \leq 2^{q-1}\bc{C_V^qh^q + h^{q/2}\E\bj{\norm*{\bm{\xi}_x}^q}}.
    \end{equation*}
    Since $\bm{\xi}_x$ is a standard Gaussian, Lemma \ref{lem:bound_normal_norm} gives $\E[\norm{\bm{\xi}_x}^q] \leq \mathcal{O}(m^{q/2})$ with $m=\dim\M$. Combining this bound with $h \leq 1$, there exists a constant $C_v > 0$, independent of $x$, such that
    \begin{equation}\label{eq:pf_comp_one_step_bound_vq}
        \E\bj{\norm{\bm{v}}^q} \leq C_vh^{q/2}.
    \end{equation}

    For this $\bm{v}$, Lemma \ref{lem:taylor_exponential} gives
    \begin{equation*}
        \Phi^h_G\bc{x,\Delta \bm{W}_k} = x + \bm{v} + \frac{1}{2}\mathrm{II}_x(\bm{v},\bm{v}) + R_3(x,\bm{v}),
    \end{equation*}
    and hence
    \begin{equation}\label{eq:pf_comp_one_step_formula_Delta}
        \bm{\Delta}_h(x) = \Phi^h_G\bc{x,\Delta \bm{W}_k} - \Phi^h_E\bc{x,\Delta \bm{W}_k} = \frac{1}{2}\mathrm{II}_x(\bm{v},\bm{v}) - hA(x) + R_3(x,\bm{v}).
    \end{equation}
    We now prove the two statements.
    \begin{enumerate}[label=(\roman*)]
        \item By bilinearity and symmetry of $\mathrm{II}_x$,
        \begin{equation*}
            \mathrm{II}_x(\bm{v},\bm{v}) = h\mathrm{II}_x\bc{\bm{\xi}_x,\bm{\xi}_x} + 2h^{3/2}\mathrm{II}_x\bc{\bm{\xi}_x,V(x)} + h^2\mathrm{II}_x\bc{V(x),V(x)}. 
        \end{equation*}
        Since Lemma \ref{lem:expectation_second_A} implies
        \begin{equation*}
            \E\bj{\mathrm{II}_x\bc{\bm{\xi}_x,\bm{\xi}_x}} = 2A(x),
        \end{equation*}
        it follows from (\ref{eq:pf_comp_one_step_formula_Delta}) that
        \begin{equation*}
            \mu_h(x) = \E\bj{\bm{\Delta}_h(x)} = \frac{1}{2}h^2\mathrm{II}_x\bc{V(x),V(x)} + \E\bj{R_3(x,\bm{v})}.
        \end{equation*}
        Note that $\E[\mathrm{II}_x(\bm{\xi}_x,V(x))]=0$ since $\bm{\xi}_x$ is centered and $\mathrm{II}_x(\,\cdot\,,V(x))$ is linear. By Assumption \ref{assum:geometric_iota},
        \begin{equation}\label{eq:pf_comp_one_step_bound_kappa}
            \kappa \defeq \sup_{x \in \M} \norm{\mathrm{II}_x}_{\op{op}} < \infty,
        \end{equation}
        and Lemma \ref{lem:taylor_exponential} gives
        \begin{equation}\label{eq:pf_comp_one_step_bound_R3}
            \sup_{x \in \M}\norm*{R_3(x,\bm{v})} \leq C_R\norm{\bm{v}}^3.
        \end{equation}
        Therefore, using (\ref{eq:pf_comp_one_step_bound_CV}), (\ref{eq:pf_comp_one_step_bound_vq}), (\ref{eq:pf_comp_one_step_bound_kappa}), and (\ref{eq:pf_comp_one_step_bound_R3}),
        \begin{align*}
            \norm*{\mu_h(x)} &\leq \frac{1}{2}h^2\norm*{\mathrm{II}_x\bc{V(x),V(x)}} + \E\bj{\norm*{R_3(x,\bm{v})}} \\
            &\leq \frac{1}{2}\kappa C_V^2 h^2 + C_R\E\bj{\norm{\bm{v}}^3} \\
            &\leq \frac{1}{2}\kappa C_V^2 h^2 + C_RC_vh^{3/2}.
        \end{align*}
        Since $h \leq 1$, there exists a constant $C_{\mu} > 0$, independent of $x$, such that
        \begin{equation}\label{eq:pf_comp_one_step_bound_mu}
            \norm*{\mu_h(x)} \leq C_{\mu}h^{3/2}\quad \Rightarrow \quad \sup_{x \in \M}\norm*{\mu_h(x)} \leq C_{\mu}h^{3/2}
        \end{equation}

        \item Using (\ref{eq:pf_comp_one_step_formula_Delta}),
        \begin{align*}
            \norm*{\bm{\Delta}_h(x)}^p &\leq 3^{p-1}\bc{\frac{1}{2^p}\norm*{\mathrm{II}_x(\bm{v},\bm{v})}^p + h^p\norm*{A(x)}^p + \norm*{R_3(x,\bm{v})}^p} \\
            &\leq 3^{p-1}\bc{\frac{1}{2^p}\kappa^p\norm*{\bm{v}}^{2p} + C_A^ph^p + C_R^p\norm*{\bm{v}}^{3p}},
        \end{align*}
        where we used (\ref{eq:pf_comp_one_step_bound_kappa}), (\ref{eq:pf_comp_one_step_bound_R3}), and $C_A \defeq \sup_{\M}\norm{A(x)} < \infty$ from Corollary \ref{cor:boundedness_of_A}. Taking expectations and applying (\ref{eq:pf_comp_one_step_bound_vq}) with $q=2p$ and $q=3p$, we obtain
        \begin{equation}\label{eq:pf_comp_one_step_bound_Delta} 
            \E\bj{\norm*{\bm{\Delta}_h(x)}^p} \leq \bc{\frac{1}{2^p}\kappa^pC_vh^p + C_A^ph^p + C_R^pC_vh^{3p/2}} \leq Ch^p, 
        \end{equation}
        for some $C > 0$ independent of $x$. Combining (\ref{eq:pf_comp_one_step_bound_mu}) and (\ref{eq:pf_comp_one_step_bound_Delta}),
        \begin{equation*}
            \E\bj{\norm*{\widetilde{\bm{\Delta}}_h(x)}^p} \leq 2^{p-1}\bc{\E\bj{\norm*{\bm{\Delta}_h(x)}^p} + \norm*{\mu_h(x)}^p} \leq 2^{p-1}\bc{Ch^p + C_{\mu}^ph^{3p/2}}.
        \end{equation*}
        Since $h \leq 1$, there exists a constant $C_{\delta} > 0$ such that
        \begin{equation*}
            \sup_{x \in \M}\E\bj{\norm*{\widetilde{\bm{\Delta}}_h(x)}^p} \leq C_{\delta}h^p.
        \end{equation*}
    \end{enumerate}
\end{proof}

\begin{lem}\label{lem:pf_comp_moment}
    Assume that $\M$ satisfies Assumption \ref{assum:geometric_iota} and Assumption \ref{assum:uniform_tubular}, and that $V \in \Gamma(T\M)$ satisfies Assumption \ref{assum:vector_field}. Let $1 \leq p < \infty$. For $\bm{X}^h_k$ constructed by (\ref{eq:gem}) with $\E\bj{\norm*{\bm{X}_0}^p} < \infty$,
    \begin{equation*}
        \E\bj{\max_{0\leq k \leq N} \norm*{\bm{X}^h_k}^p} \leq \mathcal{O}(T^p) <\infty.
    \end{equation*}
\end{lem}
\begin{proof}
    We first assume $p \geq 2$ and use the following steps to prove the moment boundedness.
    \begin{itemize}
        \item Step $1$. Martingale decomposition: Let
        \begin{equation*}
            \bm{v}_k = hV(\bm{X}^h_k) + P(\bm{X}^h_k)\Delta \bm{W}_k,
        \end{equation*}
        which is $\mathcal{F}_{k+1}$-measurable. Then, by the Taylor formula (Lemma \ref{lem:taylor_exponential}),
        \begin{equation*}
            \bm{X}^h_{k+1} = \bm{X}^h_{k} + \underbrace{\bm{v}_k + \frac{1}{2}\mathrm{II}_{\bm{X}^h_{k}}(\bm{v}_k,\bm{v}_k) + R_3(\bm{X}^h_{k},\bm{v}_k)}_{\eqdef \bm{u}^X_k}.
        \end{equation*}
        By the uniform boundedness of $V,\mathrm{II}$, and $R_3$, $\bm{u}^X_k \in L^1$ (or from following (\ref{eq:pf_comp_moment_uk})). So it can define
        \begin{equation*}
            \bm{a}^X_k \defeq \E\bj{\bm{u}^X_k \mid \mathcal{F}_k},\quad \bm{m}^X_k \defeq \bm{u}^X_k - \bm{a}^X_k.
        \end{equation*}
        Let
        \begin{equation*}
            \bm{A}_k^X \defeq \sum_{j=0}^{k-1}\bm{a}^X_j,\quad \bm{M}^X_k \defeq \sum_{j=0}^{k-1} \bm{m}^X_j,
        \end{equation*}
        with $\bm{A}_0^X = \bm{M}_0^X = 0$. Then
        \begin{equation*}
            \bm{X}^h_k = \bm{X}_0 + \bm{A}_k^X + \bm{M}^X_k.
        \end{equation*}
        Furthermore, it is obvious that $\bm{M}^X_k$ is a $\mathcal{F}_k$-martingale, since $\bm{m}^X_k$ is $\mathcal{F}_{k+1}$-measurable and $\E\bj{\bm{m}^X_k \mid \mathcal{F}_k} = 0$.

        Then we have
        \begin{equation}\label{eq:pf_comp_moment}
            \begin{aligned}
                \E\bj{\max_{0 \leq k \leq N}\norm*{\bm{X}^h_k}^p} &\leq 3^{p-1}\Big(\E\bj{\norm*{\bm{X}_0}^p} + \E\bj{\max_{0 \leq k \leq N}\norm*{\bm{A}_k^X}^p} \\
                &+ \E\bj{\max_{0 \leq k \leq N}\norm*{\bm{M}_k^X}^p}\Big).
            \end{aligned}
        \end{equation}

        \item Step $2$. Boundedness of $\bm{A}_k^X$: By Assumption \ref{assum:geometric_iota} and Lemma \ref{lem:taylor_exponential},
        \begin{equation}\label{eq:pf_comp_moment_geo}
            \kappa \defeq \sup_{x \in \M} \norm{\mathrm{II}_x}_{\op{op}} < \infty,\quad \sup_{x \in \M}\norm*{R_3(x,\bm{v})} \leq C_R\norm{\bm{v}}^3.
        \end{equation}
        Note that $\E\bj{\bm{v}_k \mid \mathcal{F}_k} = hV(\bm{X}^h_k)$ because $\bm{X}^h_k$ is $\mathcal{F}_k$-measurable and $\Delta \bm{W}_k$ is independent of $\mathcal{F}_k$. Therefore,
        \begin{equation}\label{eq:pf_comp_moment_mk_0}
            \begin{aligned}
                \norm*{\bm{a}^X_k} \leq \norm*{\E\bj{\bm{u}^X_k \mid \mathcal{F}_k}} &\leq \norm*{\E\bj{\bm{v}_k \mid \mathcal{F}_k}} + \frac{1}{2}\kappa \E\bj{\norm*{\bm{v}_k}^2 \mid \mathcal{F}_k} \\
                &\quad+ C_R\E\bj{\norm*{\bm{v}_k}^3 \mid \mathcal{F}_k}.
            \end{aligned}
        \end{equation}
        Moreover, as shown in Lemma \ref{lem:pf_comp_one_step}, (\ref{eq:pf_comp_one_step_bound_vq}) implies that for $1 \leq q < \infty$, there exists a uniform constant $C_v$ such that
        \begin{equation}\label{eq:pf_comp_moment_vk}
            \E\bj{\norm*{\bm{v}_k}^q \mid \mathcal{F}_k} = \E\bj{\norm*{\bm{v}_k}^q \mid \sigma(\bm{X}^h_k)} \leq C_vh^{q/2},
        \end{equation}
        where the first equality is because of (\ref{eq:pf_comp_indep}). Because $C_V \defeq \sup_{x \in\M}\norm{V(x)} < \infty$ (Assumption \ref{assum:vector_field}), combining (\ref{eq:pf_comp_moment_vk}) with (\ref{eq:pf_comp_moment_mk_0}) yields
        \begin{equation}\label{eq:pf_comp_moment_mk_1}
            \norm*{\bm{a}^X_k} \leq C_Vh + \frac{1}{2}\kappa C_vh + C_RC_vh^{3/2} \leq C_ah,
        \end{equation}
        for some constant $C_a > 0$. Therefore,
        \begin{equation}\label{eq:pf_comp_moment_Ak}
            \E\bj{\max_{0 \leq k \leq N}\norm*{\bm{A}_k^X}^p} \leq \E\bj{\max_{0 \leq k \leq N} k^{p-1}\sum_{j=0}^{k-1} \norm*{\bm{a}^X_j}^p}\leq N^{p-1}\sum_{j=0}^{N-1}C_a^ph^p = C_a^pT^p.
        \end{equation}

        \item Step $3$. Boundedness of $\bm{M}_k^X$: Since $\bm{M}^X_k$ is a $\mathcal{F}_k$-martingale, the discrete-time BDG inequality \citep{zorin2022martingale} implies that there exists $C_{M,p} > 0$ such that
        \begin{equation*}
            \E\bj{\max_{0 \leq k \leq N}\norm*{\bm{M}_k^X}^p} \leq C_{M,p}\E\bj{\bj{\bm{M}^X}_N^{p/2}},
        \end{equation*}
        where
        \begin{equation*}
            \bj{\bm{M}^X}_N = \sum_{j=0}^{N-1} \norm*{\bm{m}^X_j}^2.
        \end{equation*}
        Since $p \geq 2$, Lemma \ref{lem:holder_p_geq_2} implies that
        \begin{equation}\label{eq:pf_comp_moment_Mk_0}
            \E\bj{\max_{0 \leq k \leq N}\norm*{\bm{M}_k^X}^p} \leq C_{M,p}N^{\frac{p}{2}-1}\sum_{j=0}^{N-1}\E\bj{\norm*{\bm{m}^X_j}^p}.
        \end{equation}
        Note that $\bm{m}^X_k = \bm{u}^X_k - \bm{a}^X_k$, so
        \begin{equation*}
            \E\bj{\norm*{\bm{m}^X_k}^p \mid \mathcal{F}_k} \leq 2^{p-1}\bc{\E\bj{\norm*{\bm{a}^X_k}^p \mid \mathcal{F}_k} + \E\bj{\norm*{\bm{u}^X_k}^p \mid \mathcal{F}_k}}.
        \end{equation*}
        By (\ref{eq:pf_comp_moment_mk_1}),
        \begin{equation*}
            \E\bj{\norm*{\bm{a}^X_k}^p \mid \mathcal{F}_k} \leq C_a^ph^p.
        \end{equation*}
        For $\bm{u}^X_k = \bm{v}_k + \frac{1}{2}\mathrm{II}_{\bm{X}^h_{k}}(\bm{v}_k,\bm{v}_k) + R_3(\bm{X}^h_{k},\bm{v}_k)$, by (\ref{eq:pf_comp_moment_geo}) and (\ref{eq:pf_comp_moment_vk}),
        \begin{equation}\label{eq:pf_comp_moment_uk}
            \begin{aligned}
                \E\bj{\norm*{\bm{u}^X_k}^p \mid \mathcal{F}_k}
                &\leq 3^{p-1}\bc{\E\bj{\norm*{\bm{v}_k}^p \mid \mathcal{F}_k} + \frac{\kappa^p}{2^p} \E\bj{\norm*{\bm{v}_k}^{2p} \mid \mathcal{F}_k} + C_R^p \E\bj{\norm*{\bm{v}_k}^{3p} \mid \mathcal{F}_k}} \\
                &\leq 3^{p-1}\bc{C_vh^{p/2} + \frac{\kappa^p}{2^p} h^p + C_R^p h^{3p/2}}.
            \end{aligned}
        \end{equation}
        Therefore,
        \begin{equation*}
            \E\bj{\norm*{\bm{m}^X_k}^p} = \E\bj{\E\bj{\norm*{\bm{m}^X_k}^p \mid \mathcal{F}_k}} \leq C_\eta h^{p/2}
        \end{equation*}
        for some constant $C_\eta > 0$. Combining this with (\ref{eq:pf_comp_moment_Mk_0}), we obtain
        \begin{equation}\label{eq:pf_comp_moment_Mk_1}
            \E\bj{\max_{0 \leq k \leq N}\norm*{\bm{M}_k^X}^p} \leq C_{M,p}N^{\frac{p}{2}-1}\sum_{j=0}^{N-1}C_\eta h^{p/2} = C_{M,p}C_\eta T^{p/2}.
        \end{equation}

        \item Step $4$. Final bound: Combining (\ref{eq:pf_comp_moment_Ak}), (\ref{eq:pf_comp_moment_Mk_1}), and (\ref{eq:pf_comp_moment}), we obtain
        \begin{equation*}
            \E\bj{\max_{0 \leq k \leq N}\norm*{\bm{X}^h_k}^p} \leq 3^{p-1}\bc{\E\bj{\norm*{\bm{X}_0}^p} + C_a^pT^p + C_{M,p}C_\eta T^{p/2}} \leq \mathcal{O}(T^p).
        \end{equation*}
    \end{itemize}
    Finally, for $1 \leq p < 2$, by the Lyapunov's inequality \citep[Corollary 2.12.21]{bogachev2007measure},
    \begin{equation*}
        \E\bj{\max_{0 \leq k \leq N}\norm*{\bm{X}^h_k}^p} \leq \E\bj{\max_{0 \leq k \leq N}\norm*{\bm{X}^h_k}^2}^{p/2} \leq \mathcal{O}(T^p).
    \end{equation*}
\end{proof}

\subsection{Proof of Theorem \ref{thm:intr_p_strong_conv_GEM}}\label{appen:proof_of_theorem_ref_thm_intr_p_strong_conv_gem}

Theorem \ref{thm:intr_p_strong_conv_GEM} can be viewed as a direct consequence of Theorem \ref{thm:ext_p_strong_conv_GEM} based on our discussion of the bi-Lipschitz continuity of the embedding map $\iota$ (Appendix \ref{appen:discuss_embedding_map}).

\begin{proof}[Proof of Theorem \ref{thm:intr_p_strong_conv_GEM}]
    Since $\M$ satisfies Assumption \ref{assum:geometric_iota} and the embedding map $\iota \colon \M \hookrightarrow \R^n$ is globally bi-Lipschitz continuous, Theorem \ref{thm:bound_second_to_uniform_tubular} implies that $\M$ also satisfies Assumption \ref{assum:uniform_tubular}. Therefore, Theorem \ref{thm:ext_p_strong_conv_GEM} implies that
    \begin{equation*}
        \E\bj{\max_{0 \leq k \leq N} \norm*{\bm{X}^h_k - \bm{X}_{t_k}}^p} \leq C_p(T)h^{p/2},
    \end{equation*}
    where $C_p(T) = \mathcal{O}(\exp(T^p))$.

    On the other hand, since
    \begin{equation*}
        L_{\iota} \defeq \sup_{x \neq y \in \M} \frac{d_\M(x,y)}{\norm{x - y}} < \infty,
    \end{equation*}
    we have
    \begin{equation*}
        d_\M \bc{\bm{X}^h_k, \bm{X}_{t_k}} \leq L_\iota \norm*{\bm{X}^h_k - \bm{X}_{t_k}}.
    \end{equation*}
    Therefore,
    \begin{equation*}
        \E\bj{\max_{0 \leq k \leq N} d_\M\bc{\bm{X}^h_k, \bm{X}_{t_k}}^p} \leq L^p_\iota\E\bj{\max_{0 \leq k \leq N} \norm*{\bm{X}^h_k - \bm{X}_{t_k}}^p} \leq \mathcal{O}(\exp(T^p))h^{p/2}.
    \end{equation*}
\end{proof}

\subsection{Proof of Corollary \ref{cor:p_strong_conv_GEM_cpt}}\label{appen:proof_of_corollary_ref_cor_p_strong_conv_gem_cpt}

\begin{proof}[Proof of Corollary \ref{cor:p_strong_conv_GEM_cpt}]
    Let $\M$ be a compact Riemannian manifold. By Nash's embedding theorem \citep{nash1956imbedding}, there exists an isometric $C^\infty$ embedding $\iota \colon \M \hookrightarrow \R^n$ for some $n$. As discussed in Appendix \ref{appen:examples_of_manifolds}, $\M$ satisfies Assumption \ref{assum:geometric_iota} and Assumption \ref{assum:uniform_tubular}. Moreover, by Theorem \ref{thm:cpt_bi-lipschitz}, the embedding $\iota$ is globally bi-Lipschitz continuous.

    For $V \in \Gamma(T\M)$, the compactness of $\M$ implies that $V$ satisfies Assumption \ref{assum:vector_field}. The condition $\E\bj{\norm*{\bm{X}_0}^p} < \infty$ also follows from compactness.

    Then Theorem \ref{thm:intr_p_strong_conv_GEM} implies that
    \begin{equation*}
        \E\bj{\max_{0 \leq k \leq N} d_\M\bc{\bm{X}^h_k, \bm{X}_{t_k}}^p} \leq \mathcal{O}(\exp(T^p))h^{p/2}.
    \end{equation*}
\end{proof}

%% file: appendix_further_gem.tex
\section{Time-inhomogeneous Case}\label{appen:time_inhomogeneous_case}

The result of Theorem \ref{thm:ext_p_strong_conv_GEM} extends to a time-inhomogeneous $\M$-valued SDE
\begin{equation}\label{eq:mfd_sde_Bm_inhomo}
    \dx{\bm{X}_t} = V(t,\bm{X}_t)~\dx{t} + g(t)~\dx{\bm{B}^{\M}_t},
\end{equation}
where $g \colon [0,T] \sto \R$ is continuous, and $V(t,x) \in \Gamma(T\M)$ is continuous in $t$. 

Assume that $\M \subset \R^n$ satisfies Assumption \ref{assum:geometric_iota} and Assumption \ref{assum:uniform_tubular}, and that for all $t \in [0,T]$, the vector field $V(t,\cdot)$ satisfies Assumption \ref{assum:vector_field}. In addition, assume that $V(t,x)$ and $g(t)$ are $1/2$-H\"older continuous in $t$, i.e.,
\begin{equation*}
    \norm*{V(t,x) - V(s,x)} + \abs{g(t) - g(s)} \leq H\abs{t - s}^{1/2}.
\end{equation*}

Then, as shown in Section \ref{sub:well_posed_extension_on_euclidean_space} (Appendix \ref{appen:global_lipschitz_continuity_of_extension}), the map $\widetilde{V}(t,y) \defeq \chi(y)V(t,\pi(y)) \colon [0,T]\times \R^n \sto \R^n$ is globally Lipschitz in $x$ and $1/2$-H\"older continuous in $t$. As before, we write $V$ for $\widetilde{V}$ and consider the extended SDE
\begin{equation}\label{eq:mfd_sde_eucl_w_inhomo}
    \dx{\bm{X}_t} = U(t,\bm{X}_t)~\dx{t} + g(t)P(\bm{X}_t)~\dx{\bm{W}_t},
\end{equation}
where $U(t,x) = V(t,x) + g(t)^2A(x)$. For SDE (\ref{eq:mfd_sde_Bm_inhomo}), note that
\begin{equation*}
    g(t)~\dx{\bm{B}^{\M}_t} = g(t)P(\bm{X}_t) \circ \dx{\bm{W}_t} = g(t)^2A(\bm{X}_t)~\dx{t} + g(t)P(\bm{X}_t)~\dx{\bm{W}_t},
\end{equation*}
by It\^o's formula \citep{le2016brownian}. Therefore, when $\bm{X}_0 \in \M$, SDE (\ref{eq:mfd_sde_Bm_inhomo}) and SDE (\ref{eq:mfd_sde_eucl_w_inhomo}) admit the same solution \citep{hsu2002stochastic}.

\begin{rmk}
    In contrast to the Euclidean setting, we do not consider a more general diffusion coefficient of the form
    $\dx{\bm{X}_t} = V(t,\bm{X}_t)~\dx{t} + \sigma(t,\bm{X}_t)~\dx{\bm{B}^{\M}_t}$.
    In general, rewriting $\sigma(t,\bm{X}_t)~\dx{\bm{B}^{\M}_t}$ in terms of $\dx{\bm{W}_t}$ introduces an additional drift term, which is not as simple as the term $g(t)^2A(x)~\dx{t}$ in (\ref{eq:mfd_sde_eucl_w_inhomo}).
\end{rmk}

Note that SDE (\ref{eq:mfd_sde_eucl_w_inhomo}) is a well-defined Euclidean SDE whose coefficients are globally Lipschitz in $x$ and $1/2$-H\"older continuous in $t$. Therefore, for
\begin{equation*}
    \bm{Y}^h_{k+1} = \bm{Y}^h_k + hU(t_k,\bm{Y}^h_k) + g(t_k)P(\bm{Y}^h_k) \Delta \bm{W}_k,\quad k = 0,\ldots,N-1,
\end{equation*}
Theorem \ref{thm:p_conver_eucl_EM} still implies
\begin{equation*}
    \E\bj{\max_{0 \leq k \leq N}\norm*{\bm{X}_{t_k}-\bm{Y}_k^h}^p} \leq E_p(T) h^{p/2},
\end{equation*}
where $E_p(T) = \mathcal{O}(\exp(T^p))$.

Therefore, to obtain the same conclusion as in Theorem \ref{thm:ext_p_strong_conv_GEM}, it remains to establish Lemma \ref{lem:comparison_EM_GEM}. First, we need the moment boundedness of $\bm{X}^h_k$, defined by
\begin{equation*}
    \bm{X}^h_{k+1} = \exp_{\bm{X}^h_k}\bc{hV(t_k,\bm{X}^h_k) + g(t_k)P(\bm{X}^h_k) \Delta \bm{W}_k},\quad k = 0,\ldots,N-1,
\end{equation*}
as in Lemma \ref{lem:pf_comp_moment}. The same argument applies here under the condition $\sup_{t \in [0,T]} g(t) < \infty$, which follows from continuity of $g$ on $[0,T]$. Finally, in the comparison between $\bm{X}^h_k$ and $\bm{Y}^h_k$, there is no time shift, so the remainder of the proof carries over with only minor changes due to the factor $g(t_k)$, which is already uniformly bounded.

\section{Relaxing Assumption \ref{assum:vector_field}}\label{appen:relaxing_assumption_ref_assum_vector_field}

In the proof of Theorem \ref{thm:ext_p_strong_conv_GEM}, Assumption \ref{assum:vector_field} is used in two places. (\rnum{1}) It is used to construct a globally Lipschitz continuous extension $\widetilde{V}$ (Lemma \ref{lem:global_lip_ext_of_Vec_and_Proj}). This extension is needed both to control the extrinsic EM error between $\bm{Y}^h_k$ and $\bm{X}_{t_k}$ in (\ref{eq:pf_thm1_extr_bound}) and to bound the discrepancy between $\bm{Y}^h_k$ and $\bm{X}^h_k$ in Lemma \ref{lem:comparison_EM_GEM}. (\rnum{2}) It is also used to control the moment boundedness of $\bm{X}^h_k$, which is a key ingredient in the proof of Lemma \ref{lem:comparison_EM_GEM}. 

Motivated by these two roles, when we work with $\M \subset \R^n$, we may replace Assumption \ref{assum:vector_field} by the following assumption.

\newtheorem{assumvecdag}{Assumption} 
\renewcommand{\theassumvecdag}{III$^\dag$}
\makeatletter
\renewcommand{\theHassumvecdag}{assumvecdag}
\makeatother
\begin{assumvecdag}\label{assum:vector_field_dag}
    For $\M \subset \R^n$ a Riemannian manifold, let $V \in \Gamma(T\M)$. Assume that
    \begin{equation*}
        \sup_{x_1 \neq x_2 \in \M} \frac{\norm*{V(x_1) - V(x_2)}}{\norm{x_1-x_2}} < \infty,\quad \sup_{x \in \M} \frac{\norm*{V(x)}}{1 + \norm{x}^{1/3}} < \infty.
    \end{equation*}
\end{assumvecdag}

If $V \in \Gamma(T\M)$ satisfies Assumption \ref{assum:vector_field_dag}, then Lemma \ref{lem:exten_lip_sublinear_gr} implies that $V$ admits a globally Lipschitz extension to $\R^n$, which we continue to denote by $V \colon \R^n \sto \R^n$. Moreover, this extension satisfies a sub-linear growth bound: there exists a constant $K > 0$ such that
\begin{equation}\label{eq:pf_comp_sublin_grow_V}
    \norm*{V(y)} \leq K\bc{1 + \norm{y}^{1/3}},\quad \forall~ y\in \R^n.
\end{equation}
Since the coefficients of the extended Euclidean SDE remain globally Lipschitz, Theorem \ref{thm:p_conver_eucl_EM} still gives
\begin{equation*}
    \E\bj{\max_{0 \leq k \leq N}\norm*{\bm{X}_{t_k}-\bm{Y}_k^h}^p} \leq E_p(T) h^{p/2},
\end{equation*}
where $E_p(T) = \mathcal{O}(\exp(T^p))$.

Therefore, after replacing Assumption \ref{assum:vector_field} by Assumption \ref{assum:vector_field_dag}, it remains to establish the relaxed counterpart Lemma \ref{lem:comparison_EM_GEM_relax} of Lemma \ref{lem:comparison_EM_GEM} in order to prove Theorem \ref{thm:ext_p_strong_conv_GEM}. The main difference between the proofs of Lemma \ref{lem:comparison_EM_GEM_relax} and Lemma \ref{lem:comparison_EM_GEM} is the moment boundedness of $\bm{X}^h_k$ (Lemma \ref{lem:pf_comp_moment_relax} versus Lemma \ref{lem:pf_comp_moment}), which relies on the sub-linear growth condition (\ref{eq:pf_comp_sublin_grow_V}). Once this is established, the remaining steps follow as before. For the sake of completeness, we provide the proofs below.

\begin{lem}\label{lem:comparison_EM_GEM_relax}
    Assume that $\M \subset \R^n$ satisfies Assumption \ref{assum:geometric_iota} and Assumption \ref{assum:uniform_tubular}, and that $V \in \Gamma(T\M)$ satisfies Assumption \ref{assum:vector_field_dag}. Let $1 \leq p < \infty$. For $\bm{X}_h^k$ constructing as (\ref{eq:gem}) and $\bm{Y}^h_k$ constructing as (\ref{eq:em_to_gem}), we have
    \begin{equation*}
        \E\bj{\max_{0 \leq k \leq N}\norm*{\bm{X}^h_k-\bm{Y}_k^h}^p} \leq G_p(T) h^{p/2},
    \end{equation*}
    where $G_p(T) = \mathcal{O}(\exp(T^p))$.
\end{lem}
\begin{proof}
    The proof is almost the same as that of Lemma \ref{lem:comparison_EM_GEM}, so we use the same notation and only emphasize the differences.

    For $\bm{\Delta}_h(\bm{X}^h_k) = \Phi^h_G\bc{\bm{X}^h_k,\Delta \bm{W}_k} - \Phi^h_E\bc{\bm{X}^h_k,\Delta \bm{W}_k}$, by (\ref{eq:pf_comp_one_step_bound_delta_relax}) in Lemma \ref{lem:pf_comp_one_step_relax} and the moment boundedness of $\bm{X}^h_k$ (Lemma \ref{lem:pf_comp_moment_relax}), we have $\bm{\Delta}_h(\bm{X}^h_k) \in L^1$. Hence we can define
    \begin{equation*}
        \bm{\mu}_k \defeq \E\bj{\bm{\Delta}_h(\bm{X}^h_k) \mid \mathcal{F}_k},\quad \bm{m}^{(2)}_k \defeq \bm{\Delta}_h(\bm{X}^h_k) - \bm{\mu}_k,
    \end{equation*}
    and obtain the same decomposition $\bm{e}_k = \bm{A}_k + \bm{M}_k$ with
    \begin{equation}\label{pf_comp_total_bound_relax}
        \E\bj{ \max_{0 \leq m \leq k} \norm*{\bm{e}_m}^p} \leq 2^{p-1}\bc{ \E\bj{ \max_{0 \leq m \leq k} \norm*{\bm{A}_m}^p} + \E\bj{ \max_{0 \leq m \leq k} \norm*{\bm{M}_m}^p}},
    \end{equation}
    as in Lemma \ref{lem:comparison_EM_GEM}. The existence of $\E\bj{ \max_{0 \leq m \leq k} \norm*{\bm{e}_m}^p}$ follows from Theorem \ref{thm:p_conver_eucl_EM} and
    \begin{equation}\label{eq:moment_relax}
        M_X(T) \defeq \E\bj{\max_{0\leq k \leq N} \norm*{\bm{X}^h_k}^p} \leq \mathcal{O}(\exp(T^p)) < \infty,
    \end{equation}
    as shown in Lemma \ref{lem:pf_comp_moment_relax}.

    For $\bm{A}_k$, Lemma \ref{lem:pf_comp_one_step_relax} gives
    \begin{equation}\label{pf_comp_bound_mu_relax}
        \norm*{\bm{\mu}_k} = \norm*{\E\bj{\bm{\Delta}_h(\bm{X}^h_k) \mid \sigma(\bm{X}^h_k)}} = \norm*{\mu_h(\bm{X}^h_k)}
        \leq C_\mu\bc{1+\norm*{\bm{X}^h_k}} h^{3/2}. 
    \end{equation}
    Combining (\ref{pf_comp_bound_mu_relax}) with (\ref{eq:moment_relax}), the same reasoning as in Lemma \ref{lem:comparison_EM_GEM} yields
    \begin{equation}\label{pf_comp_bound_Ak_relax}
        \E\bj{ \max_{0 \leq m \leq k} \norm*{\bm{A}_m}^p} \leq (2T)^{p-1}L_U^p\sum_{j=0}^{k-1}h S_j + 2^{2(p-1)} C_\mu^pT^p\bc{1 + M_X(T)} h^{p/2}.
    \end{equation}
    
    For $\bm{M}_k$, the bound for $\bm{M}^{(1)}_k$ is unchanged. For $\bm{M}^{(2)}_k$, Lemma \ref{lem:pf_comp_one_step_relax} together with (\ref{eq:moment_relax}) gives
    \begin{equation*}
        \E\bj{\norm*{\bm{m}^{(2)}_k}^p} = \E\bj{\norm*{\widetilde{\bm{\Delta}}_h(\bm{X}^h_k)}^p} \leq C_\delta \bc{1 + M_X(T)} h^p.
    \end{equation*}
    Therefore, as in Lemma \ref{lem:comparison_EM_GEM},
    \begin{equation}\label{pf_comp_bound_Mk_relax}
        \E\bj{ \max_{0 \leq m \leq k} \norm*{\bm{M}_m}^p} \leq C^{(1)}_p T^{\frac{p}{2} - 1}L_P^p \sum_{j=0}^{k-1}hS_j + C^{(2)}_pC_\delta  T^{p/2}\bc{1 + M_X(T)}h^{p/2}.
    \end{equation}
    Combining (\ref{pf_comp_bound_Ak_relax}), (\ref{pf_comp_bound_Mk_relax}), and (\ref{pf_comp_total_bound_relax}), we obtain
    \begin{equation*}
        S_k \leq C_{1,p}\bc{1 + M_X(T)}T^{p}h^{p/2}+ C_{2,p}T^{p-1}\sum_{j=0}^{k-1}hS_j,
    \end{equation*}
    which yields
    \begin{equation*}
        S_N \leq C_{1,p}\bc{1 + M_X(T)}T^{p}\exp\bc{C_{2,p}T^{p}} h^{p/2}.
    \end{equation*}
    Hence,
    \begin{equation*}
        \E\bj{\max_{0 \leq k \leq N}\norm*{\bm{X}^h_k-\bm{Y}_k^h}^p} \leq  \mathcal{O}(\exp(T^p))h^{p/2}. 
    \end{equation*}
\end{proof}

\begin{lem}\label{lem:pf_comp_one_step_relax}
    Assume that $\M \subset \R^n$ satisfies Assumption \ref{assum:geometric_iota} and Assumption \ref{assum:uniform_tubular}, and that $V \in \Gamma(T\M)$ satisfies Assumption \ref{assum:vector_field_dag}. Let $1 \leq p < \infty$. Using the notation as Lemma \ref{lem:pf_comp_one_step}, the following statements hold.
    \begin{enumerate}[label=(\roman*)]
        \item There exists a constant $C_\mu > 0$ such that
        \begin{equation*}
            \norm*{\mu_h(x)} \leq C_{\mu}\bc{1+\norm{x}}h^{3/2}.
        \end{equation*}

        \item There exists a constant $C_{\delta} > 0$ such that
        \begin{equation*}
            \E\bj{\norm*{\widetilde{\bm{\Delta}}_h(x)}^p} \leq C_{\delta}\bc{1 + \norm{x}^{p}}h^p.
        \end{equation*}
    \end{enumerate}
\end{lem}
\begin{proof}
    Fix $x \in \M$. For
    \begin{equation*}
        \bm{v} = hV(x) + P(x)\Delta \bm{W}_k = hV(x) + \sqrt{h}\bm{\xi}_x \in T_x\M,
    \end{equation*}
    where $\sqrt{h}\bm{\xi}_x = P(x)\Delta \bm{W}_k$, by (\ref{eq:pf_comp_sublin_grow_V}),
    \begin{align*}
        \E\bj{\norm{\bm{v}}^q} &\leq 2^{q-1}\bc{h^q\norm*{V(x)}^q + h^{q/2}\E\bj{\norm*{\bm{\xi}_x}^q}} \\
        &\leq 2^{q-1}\bc{h^qK^q2^{q-1}\bc{1 + \norm{x}^{q/3}} + h^{q/2}\E\bj{\norm*{\bm{\xi}_x}^q}}.
    \end{align*}
    Hence, by Lemma \ref{lem:bound_normal_norm}, for any $1 \leq q < \infty$ there exist constants $C_{1,v},C_{2,v} > 0$ such that
    \begin{equation}\label{eq:pf_comp_one_step_bound_vq_relax}
        \E\bj{\norm{\bm{v}}^q} \leq C_{1,v}\bc{1 + \norm{x}^{q/3}}h^{q} + C_{2,v}h^{q/2}.
    \end{equation}
    Following the proof of Lemma \ref{lem:pf_comp_one_step}, we have
    \begin{equation*}
        \norm*{\mu_h(x)} \leq \frac{1}{2}h^2\norm*{\mathrm{II}_x\bc{V(x),V(x)}} + \E\bj{\norm*{R_3(x,\bm{v})}}
        \leq \frac{1}{2}h^2\kappa \norm*{V(x)}^2 + C_R \E\bj{\norm*{\bm{v}}^3}.
    \end{equation*}
    Combining this with (\ref{eq:pf_comp_sublin_grow_V}) and (\ref{eq:pf_comp_one_step_bound_vq_relax}), 
    \begin{equation*}
        \norm*{\mu_h(x)} \leq h^2\kappa\bc{1+\norm{x}^{2/3}} + C_R\bc{C_{1,v}\bc{1 + \norm{x}}h^3 + C_{2,v}h^{3/2}}
    \end{equation*}
    Note that $\norm{x}^{2/3} \leq \norm{x} + 1$ and we always assume $h \leq 1$. Therefore, there exists a constant $C_{\mu} > 0$ such that
    \begin{equation}\label{eq:pf_comp_one_step_bound_mu_relax}
        \norm*{\mu_h(x)} \leq C_{\mu}\bc{1+\norm{x}}h^{3/2}.
    \end{equation}
    Next, for $\bm{\Delta}_h(x)$, as in Lemma \ref{lem:pf_comp_one_step},
    \begin{equation*}
        \norm*{\bm{\Delta}_h(x)}^p \leq 3^{p-1}\bc{\frac{1}{2^p}\kappa^p\norm*{\bm{v}}^{2p} + C_A^ph^p + C_R^p\norm*{\bm{v}}^{3p}}.
    \end{equation*}
    Therefore, (\ref{eq:pf_comp_one_step_bound_vq_relax}) implies that there exist constants $C_{1},C_{2}> 0$ such that
    \begin{equation}\label{eq:pf_comp_one_step_bound_delta_relax}
        \E\bj{\norm*{\bm{\Delta}_h(x)}^p} \leq C_{1}\bc{1 + \norm{x}^{p}}h^{2p} + C_{2}h^p,
    \end{equation}
    where we also use $\norm{x}^{2/3} \leq \norm{x} + 1$. Then, combining (\ref{eq:pf_comp_one_step_bound_delta_relax}) with (\ref{eq:pf_comp_one_step_bound_mu_relax}), we have
    \begin{align*}
        \E\bj{\norm*{\widetilde{\bm{\Delta}}_h(x)}^p} &\leq 2^{p-1}\bc{\E\bj{\norm*{\bm{\Delta}_h(x)}^p} + \norm*{\mu_h(x)}^p} \\
        &\leq C_{\delta}\bc{1 + \norm{x}^{p}}h^p,
    \end{align*}
    for some constants $C_{\delta}$. 
\end{proof}

\begin{lem}\label{lem:pf_comp_moment_relax}
    Assume that $\M \subset \R^n$ satisfies Assumption \ref{assum:geometric_iota} and Assumption \ref{assum:uniform_tubular}, and that $V \in \Gamma(T\M)$ satisfies Assumption \ref{assum:vector_field_dag}. Let $1 \leq p < \infty$. For $\bm{X}^h_k$ constructed by (\ref{eq:gem}) with $\E\bj{\norm*{\bm{X}_0}^p} < \infty$,
    \begin{equation*}
        \E\bj{\max_{0\leq k \leq N} \norm*{\bm{X}^h_k}^p} \leq \mathcal{O}(\exp(T^p)) < \infty.
    \end{equation*}
\end{lem}
\begin{rmk}
    In contrast to the proof of Lemma \ref{lem:pf_comp_moment}, we employ a localization argument, as in the proof of Theorem \ref{thm:p_conver_eucl_EM}, to justify the required integrability. We then apply the discrete Gr\"onwall inequality.
\end{rmk}
\begin{proof}
    We first treat the case $p \geq 2$.
    \begin{itemize}
        \item Step $1$. Localization: For any $\ell \in \N$, define the stopping time
        \begin{equation*}
            \bm{N}_\ell \defeq \min\bb{k > 0 \colon \norm*{\bm{X}^h_k} > \ell},
        \end{equation*}
        which is an $\mathcal{F}_k$-stopping time, and consider the stopped process $\bm{X}^{(\ell)}_k \defeq \bm{X}^h_{k \wedge \bm{N}_\ell}$. Define
        \begin{equation*}
            \bm{u}_k^{(\ell)} \defeq \bm{X}^{(\ell)}_{k+1} - \bm{X}^{(\ell)}_k = \mathbb{I}_{\bb{k < \bm{N}_\ell}}\bc{\bm{X}^h_{k+1} - \bm{X}^h_k}.
        \end{equation*}
        Using the same notation as in Lemma \ref{lem:pf_comp_moment}, for $\bm{v}_k = hV(\bm{X}^h_k) + P(\bm{X}^h_k)\Delta \bm{W}_k$, let
        \begin{equation*}
            \bm{u}_k^X = \bm{X}^h_{k+1} - \bm{X}^h_k = \bm{v}_k + \frac{1}{2}\mathrm{II}_{\bm{X}^h_{k}}(\bm{v}_k,\bm{v}_k) + R_3(\bm{X}^h_{k},\bm{v}_k).
        \end{equation*}
        Then
        \begin{equation*}
            \bm{u}_k^{(\ell)} =  \mathbb{I}_{\bb{k < \bm{N}_\ell}}\bm{u}_k^X.
        \end{equation*}
        Note that on $\bb{k < \bm{N}_\ell}$, we have $\norm*{V(\bm{X}^h_k)} \leq K(1 + \ell^{1/3}) < \infty$. Hence (\ref{eq:pf_comp_moment_uk}) implies that $\bm{u}_k^{(\ell)} \in L^p \subset L^1$. Therefore, we may define
        \begin{equation*}
            \bm{a}^{(\ell)}_k \defeq \E\bj{\bm{u}_k^{(\ell)} \mid \mathcal{F}_k},\quad \bm{m}^{(\ell)}_k \defeq \bm{u}_k^{(\ell)} - \bm{a}_k^{(\ell)}.
        \end{equation*}
        Then
        \begin{equation*}
            \bm{X}^{(\ell)}_k = \bm{X}_0 + \bm{A}_k^{(\ell)} + \bm{M}^{(\ell)}_k,
        \end{equation*}
        where
        \begin{equation*}
            \bm{A}_k^{(\ell)} \defeq \sum_{j=0}^{k-1}\bm{a}^{(\ell)}_j,\quad \bm{M}^{(\ell)}_k \defeq \sum_{j=0}^{k-1} \bm{m}^{(\ell)}_j.
        \end{equation*}
        Consequently,
        \begin{equation}\label{eq:pf_comp_moment_relax}
            \begin{aligned}
                \E\bj{\max_{0 \leq m \leq k}\norm{\bm{X}^{(\ell)}_m}^p} &\leq 3^{p-1}\Big(\E\bj{\norm*{\bm{X}_0}^p} + \E\bj{\max_{0 \leq m \leq k}\norm{\bm{A}^{(\ell)}_m}^p} \\
                &\quad+ \E\bj{\max_{0 \leq m \leq k}\norm{\bm{M}^{(\ell)}_m}^p}\Big).
            \end{aligned}
        \end{equation}

        \item Step $2$. Boundedness of $\bm{A}^{(\ell)}_k$: First, by (\ref{eq:pf_comp_one_step_bound_vq_relax}), since $\mathbb{I}_{\bb{k < \bm{N}_\ell}}$ is $\mathcal{F}_k$-measurable,
        \begin{equation}\label{eq:pf_comp_moment_vk_relax}
            \E\bj{\norm*{\bm{v}_k}^q \mathbb{I}_{\bb{k < \bm{N}_\ell}} \mid \mathcal{F}_k} \leq C_{1,v}\bc{1 + \norm{\bm{X}^{(\ell)}_k}^{q/3}}h^{q} + C_{2,v}h^{q/2}.
        \end{equation}
        Combining (\ref{eq:pf_comp_moment_vk_relax}) with (\ref{eq:pf_comp_sublin_grow_V}),
        \begin{equation}\label{eq:pf_comp_moment_mk_1_relax}
            \begin{aligned}
                \norm{\bm{a}^{(\ell)}_k} &\leq \norm{\E\bj{\bm{v}_k\mathbb{I}_{\bb{k < \bm{N}_\ell}} \mid \mathcal{F}_k}} + \frac{1}{2}\kappa \E\bj{\norm{\bm{v}_k}^2\mathbb{I}_{\bb{k < \bm{N}_\ell}} \mid \mathcal{F}_k} \\
                &\quad+ C_R\E\bj{\norm{\bm{v}_k}^3\mathbb{I}_{\bb{k < \bm{N}_\ell}} \mid \mathcal{F}_k} \\
                &\leq C_a\bc{1 + \norm{\bm{X}^{(\ell)}_k}}h,
            \end{aligned}
        \end{equation}
        for some $C_a > 0$, where we use $a^{1/3} \leq a + 1$ and $a^{2/3} \leq a + 1$ for all $a \geq 0$. It follows that
        \begin{align*}
            \max_{0 \leq m \leq k}\norm{\bm{A}^{(\ell)}_m}^p
            &\leq \max_{0 \leq m \leq k} m^{p-1} \sum_{j=0}^{m-1}\norm{\bm{a}^{(\ell)}_j}^p
            \leq N^{p-1} \sum_{j=0}^{k-1}\norm{\bm{a}^{(\ell)}_j}^p \\
            &\leq N^{p-1}C_a^p h^p 2^{p-1} \sum_{j=0}^{k-1}\bc{1 + \norm{\bm{X}^{(\ell)}_j}^p} \\
            &\leq 2^{p-1}C_a^p T^{p-1} \sum_{j=0}^{k-1}h\bc{1 + \max_{0 \leq m \leq j}\norm{\bm{X}^{(\ell)}_m}^p}.
        \end{align*}
        Let
        \begin{equation*}
            \alpha_k^{(\ell)} \defeq \E\bj{\max_{0 \leq m \leq k} \norm{\bm{X}^{(\ell)}_{m}}^p}.
        \end{equation*}
        Then
        \begin{equation}\label{eq:pf_comp_moment_Ak_relax}
            \E\bj{\max_{0 \leq m \leq k}\norm{\bm{A}^{(\ell)}_m}^p} \leq 2^{p-1}C_a^p T^{p-1}\sum_{j=0}^{k-1}h\bc{1 + \alpha_j^{(\ell)}}.
        \end{equation}

        \item Step $3$. Boundedness of $\bm{M}^{(\ell)}_k$: For $\bm{m}^{(\ell)}_k$, by definition,
        \begin{equation*}
            \E\bj{\norm{\bm{m}^{(\ell)}_k}^p \mid \mathcal{F}_k} \leq 2^{p-1}\bc{\E\bj{\norm{\bm{a}^{(\ell)}_k}^p \mid \mathcal{F}_k} + \E\bj{\norm{\bm{u}^{(\ell)}_k}^p \mid \mathcal{F}_k}},
        \end{equation*}
        where, by (\ref{eq:pf_comp_moment_mk_1_relax}),
        \begin{equation}\label{eq:pf_comp_moment_ak_0_relax}
            \E\bj{\norm{\bm{a}^{(\ell)}_k}^p \mid \mathcal{F}_k} = \norm{\bm{a}^{(\ell)}_k}^p \leq 2^{p-1} C_a^p\bc{1 + \norm{\bm{X}^{(\ell)}_k}^p}h^p.
        \end{equation}
        For $\bm{u}^{(\ell)}_k = \mathbb{I}_{\bb{k < \bm{N}_\ell}}\bm{u}^X_k$, combining (\ref{eq:pf_comp_moment_uk}) with (\ref{eq:pf_comp_moment_vk_relax}),
        \begin{equation}\label{eq:pf_comp_moment_uk_0_relax}
            \E\bj{\norm{\bm{u}^{(\ell)}_k}^p \mid \mathcal{F}_k} \leq C_{1,u}\bc{1 + \norm{\bm{X}^{(\ell)}_k}^p}h^p + C_{2,u}h^{p/2},
        \end{equation}
        for some constants $C_{1,u},C_{2,u} > 0$, where we also use $a^{1/3} \leq a + 1$ and $a^{2/3} \leq a + 1$ for all $a \geq 0$. Therefore, combining (\ref{eq:pf_comp_moment_ak_0_relax}) (\ref{eq:pf_comp_moment_uk_0_relax}), and assuming $h \leq 1$,
        \begin{equation*}
            \E\bj{\norm{\bm{m}^{(\ell)}_k}^p \mid \mathcal{F}_k} \leq C_m\bc{1 + \norm{\bm{X}^{(\ell)}_k}^p}h^{p/2},
        \end{equation*}
        for some constant $C_m > 0$, which implies
        \begin{equation}\label{eq:pf_comp_moment_mk_2}
            \E\bj{\norm{\bm{m}^{(\ell)}_k}^p} \leq C_m\bc{1 + \E\bj{\norm{\bm{X}^{(\ell)}_k}^p}}h^{p/2} \leq C_m\bc{1 + \alpha_k^{(\ell)}}h^{p/2}.
        \end{equation}
        Since $\bm{M}^{(\ell)}_k$ is a martingale, the discrete-time BDG inequality \citep[Corollary 3.2]{zorin2022martingale} implies that there exists $C_{M,p} > 0$ such that
        \begin{equation*}
            \E\bj{\max_{0 \leq m \leq k}\norm{\bm{M}_m^{(\ell)}}^p} \leq C_{M,p}\E\bj{[\bm{M}^{(\ell)}]_k^{p/2}},
        \end{equation*}
        where
        \begin{equation*}
            [\bm{M}^{(\ell)}]_k^{p/2} = \bc{\sum_{j=0}^{k-1} \norm{\bm{m}^{(\ell)}_j}^2}^{p/2} \leq k^{\frac{p}{2} - 1}\sum_{j=0}^{k-1} \norm{\bm{m}^{(\ell)}_j}^p
        \end{equation*}
        by Lemma \ref{lem:holder_p_geq_2} ($p \geq 2$). Hence, (\ref{eq:pf_comp_moment_mk_2}) gives
        \begin{equation}\label{eq:pf_comp_moment_Mk_1_relax}
            \E\bj{\max_{0 \leq m \leq k}\norm{\bm{M}_m^{(\ell)}}^p} \leq C_{M,p}N^{\frac{p}{2} - 1}\sum_{j=0}^{k-1} \E\bj{\norm{\bm{m}^{(\ell)}_j}^p} \leq C_{M,p}C_m T^{\frac{p}{2} - 1} \sum_{j=0}^{k-1}h\bc{1 + \alpha_k^{(\ell)}}.
        \end{equation}

        \item Step $4$. Final bound: Combining (\ref{eq:pf_comp_moment_Ak_relax}) and (\ref{eq:pf_comp_moment_Mk_1_relax}) with (\ref{eq:pf_comp_moment_relax}), we obtain
        \begin{equation*}
            \alpha_k^{(\ell)} \leq C_{1,p} + C_{2,p}T^{p-1}\sum_{j=0}^{k-1}h\bc{1 + \alpha_k^{(\ell)}},
        \end{equation*}
        for some constants $C_{1,p},C_{2,p} > 0$ (here we assume $T \geq 1$ such that $T^{p/2} \leq T^p$). Then the discrete Gr\"onwall inequality \citep[Lemma 10.5]{thomee2007galerkin} implies
        \begin{equation*}
            \alpha_N^{(\ell)} \leq (1 + C_{1,p})e^{C_{2,p}T^{p}} - 1,
        \end{equation*}
        that is,
        \begin{equation*}
            \E\bj{\max_{0\leq k \leq N} \norm*{\bm{X}^{h}_{k\wedge \bm{N}_\ell}}^p} \leq (1 + C_{1,p})e^{C_{2,p}T^{p}} - 1.
        \end{equation*}
        As $\ell \to \infty$, by Fatou's lemma,
        \begin{equation*}
            \E\bj{\max_{0\leq k \leq N} \norm*{\bm{X}^h_k}^p} \leq \mathcal{O}(\exp(T^p)) < \infty.
        \end{equation*}
    \end{itemize}
    For $1 \leq p < 2$, we prove similarly via Lyapunov's inequality. 
\end{proof}

\begin{lem}\label{lem:exten_lip_sublinear_gr}
    Let $\mathcal{S} \subset \R^n$ be a nonempty subset. If $f \colon \mathcal{S} \sto \R^m$ satisfies that there exist $L,K > 0$ such that
    \begin{equation*}
        \norm*{f(x_1) - f(x_2)} \leq L\norm{x_1 - x_2},\quad \norm{f(x)} \leq K\bc{1 + \norm{x}^{1/3}},\qquad \forall~x,x_1,x_2\in \mathcal{S},
    \end{equation*}
    then there exists an $F \colon \R^n \sto \R^m$ such that $F|_{\mathcal{S}} = f$ and
    \begin{equation*}
        \norm*{F(y_1) - F(y_2)} \leq (L+2K)\norm{y_1 - y_2},\quad \norm{F(y)} \leq 2K\bc{1 + \norm{y}^{1/3}},\qquad \forall~y,y_1,y_2\in \R^n.
    \end{equation*}
\end{lem}
\begin{proof}
    By the Kirszbraun's theorem \citep[Theorem 4.2.3]{cobzas2019lipschitz}, there exists a $G \colon \R^n \sto \R^m$ such that $G|_{\mathcal{S}} = f$ and
    \begin{equation*}
        \norm*{G(y_1) - G(y_2)} \leq L\norm{y_1 - y_2},\quad \forall~ y_1,y_2 \in \R^n.
    \end{equation*}
    To preserve the sub-linear growth, define a radius function $r \colon \R^n \sto (0,\infty)$ by
    \begin{equation*}
        r(y) \defeq 2K \bc{1 + \min \bb{\norm{y},\norm{y}^{1/3}}}.
    \end{equation*}
    Clearly,
    \begin{equation}\label{eq:pf_lip_ext_bound_r}
        K\bc{1 + \norm{y}^{1/3}} \leq r(y) \leq 2K\bc{1 + \norm{y}^{1/3}}.
    \end{equation}
    Moreover, since the map $t \mapsto \min\bb{t,t^{1/3}}$ is $1$-Lipschitz on $(0,\infty)$, the function $r$ is $2K$-Lipschitz continuous.

    Next, consider the projection map $\Pi \colon (0,\infty) \times \R^m \sto \R^m$ defined by
    \begin{equation*}
        \Pi(r,z) \defeq \begin{cases}
            z,& \norm{z} \leq r,\\
            r\frac{z}{\norm{z}},& \norm{z} > r.
        \end{cases}
    \end{equation*}
    Define
    \begin{equation*}
        F(y) \defeq \Pi(r(y), G(y)).
    \end{equation*}
    For any $x \in \mathcal{S}$, since $G(x)=f(x)$, (\ref{eq:pf_lip_ext_bound_r}) gives
    \begin{equation*}
        \norm{f(x)} \leq K\bc{1 + \norm{x}^{1/3}} \leq r(x),
    \end{equation*}
    and hence
    \begin{equation*}
        F(x) = \Pi(r(x),G(x)) = G(x) = f(x),
    \end{equation*}
    i.e., $F|_{\mathcal{S}} = f$. The sub-linear growth bound follows directly from the definition of $F$ and (\ref{eq:pf_lip_ext_bound_r}):
    \begin{equation*}
        \norm{F(y)} = \norm*{\Pi(r(y),G(y))} \leq r(y) \leq 2K\bc{1 + \norm{y}^{1/3}}.
    \end{equation*}
    Finally, using the Lipschitz continuity of $r$ and $\Pi$ (Lemma \ref{lem:lip_project}), we obtain
    \begin{align*}
        \norm*{F(y_1) - F(y_2)} &\leq \norm*{\Pi(r(y_1), G(y_1)) - \Pi(r(y_1), G(y_2))}
            \\
        &\quad+ \norm*{\Pi(r(y_1), G(y_2)) - \Pi(r(y_2), G(y_2))} \\
        &\leq \norm*{G(y_1) - G(y_2)} + \abs{r(y_1) - r(y_2)} \\
        &\leq \bc{L + 2K}\norm*{y_1 - y_2}. 
    \end{align*}
\end{proof}

%% file: appendix_rld.tex
\section{More Details of Section \ref{sec:p_wasserstein_convergence_of_rld}}

In this section, we bound the initial Wasserstein error under the Bakry--\'Emery curvature condition; see Section \ref{appen:initial_error}. We also provide an extrinsic Wasserstein bound for RLD via GEM; see Section \ref{appen:extrinsic_wasserstein_bound}.

\subsection{Initial Error}\label{appen:initial_error}

\begin{lem}\label{eq:ini_error_BE}
    Let $(\M,g)$ be a complete and connected Riemannian manifold without boundary. For $\phi \in C^\infty(\M)$ with $C_\phi \defeq \int_\M \phi ~\dx{\op{vol}} < \infty$, let $\mu_\phi$ be the probability distribution on $\M$ defined as $\dx{\mu_\phi} \defeq C_\phi^{-1} e^{-\phi}~\dx{\op{vol}}$. Let $1 \leq p < \infty$. If $\phi$ satisfies the Bakry--\'Emery curvature condition
    \begin{equation*}
        \op{Ric} + \nabla^2 \phi \geq \lambda_\kappa g,\quad \lambda_\kappa > 0,
    \end{equation*}
    then for any $x \in \M$,
    \begin{equation*}
        \mathcal{W}_p\bc{\mu_\phi,\delta_x} < \infty.
    \end{equation*}
\end{lem}
\begin{proof}
    Since $\phi$ satisfies the Bakry--\'Emery curvature condition, $\mu_\phi$ satisfies the curvature--dimension condition $\op{CD}(\lambda_\kappa,\infty)$ \citep[Theorem 14.8]{villani2008optimal}. This implies that $\mu_\phi$ satisfies the Talagrand inequality $T_2(\lambda_\kappa)$, and hence $T_1(\lambda_\kappa)$ \citep[Theorem 22.14]{villani2008optimal}. 

    Fix $x \in \M$. Let $m_x$ be a median of $\rho_x \defeq d_\M(x,\cdot) \colon \M \sto \R$, defined by
    \begin{equation*}
        m_x \defeq \inf\bb{t \geq 0 \colon \mu_{\phi}\bc{\rho_x \leq t} \geq \frac{1}{2}} < \infty.
    \end{equation*}
    Clearly, $\rho_x$ is $1$-Lipschitz continuous. Therefore, by Theorem 22.10 in \citet{villani2008optimal}, for any $r \geq 0$,
    \begin{equation}\label{eq:pf_init_error_tail}
        \mu_\phi\bc{\rho_x \geq m_x + r} \leq \exp\bc{-\frac{\lambda_\kappa}{2}r^2}.
    \end{equation}
    Using the tail integral formula \citep[Proposition 6.23]{folland2013real}, we have
    \begin{align*}
        \int_\M \rho_x(y)^p ~\dx{\mu_\phi(y)} &= \int_0^\infty pt^{p-1}\mu_\phi\bc{\rho_x \geq t}~\dx{t} \\
        &= \int_0^{m_x}pt^{p-1}\mu_\phi\bc{\rho_x \geq t}~\dx{t} + \int_{m_x}^\infty pt^{p-1}\mu_\phi\bc{\rho_x \geq t}~\dx{t}.
    \end{align*}
    The first term is bounded by $m_x^p$. For the second term, (\ref{eq:pf_init_error_tail}) yields
    \begin{align*}
        \int_{m_x}^\infty pt^{p-1}\mu_\phi\bc{\rho_x \geq t}~\dx{t} &= \int_0^\infty p \bc{m_x + r}^{p-1}\mu_\phi\bc{\rho_x \geq m_x + r}~\dx{r} \\
        &\leq \int_0^\infty p \bc{m_x + r}^{p-1}\exp\bc{-\frac{\lambda_\kappa}{2}r^2}~\dx{r} <\infty.
    \end{align*}
    Therefore,
    \begin{equation*}
        \mathcal{W}_p\bc{\mu_\phi,\delta_x}^p \leq \int_\M d_\M(x,y)^p~\dx{\mu_\phi(y)} < \infty. 
    \end{equation*}
\end{proof}

\subsection{Extrinsic Wasserstein Bound}\label{appen:extrinsic_wasserstein_bound}

Because $\M \subset \R^n$ with the embedding map $\iota \colon \M \hookrightarrow \R^n$, any probability distribution $\mu$ on $\M$ has a push-forward probability distribution $\widetilde{\mu} = \iota_{\#}\mu$ on $\R^n$, i.e.,
\begin{equation*}
    \widetilde{\mu}(\mathcal{V}) \defeq \mu(\mathcal{V} \cap \M). 
\end{equation*}
For $\mu,\nu$ on $\M$, let $\widetilde{\mu},\widetilde{\nu}$ be their push-forward measures. Then
\begin{equation*}
    \widetilde{\mathcal{W}}_p(\mu,\nu) \defeq \mathcal{W}_p(\widetilde{\mu},\widetilde{\nu}) = \inf_{\bm{X} \sim \mu,\bm{Y}\sim \nu} \E\bj{\norm{\bm{X} - \bm{Y}}^p}^{1/p}.
\end{equation*}
Since $\iota$ is naturally $1$-Lipschitz (Appendix \ref{appen:bi_lipschitz_continuity_of_embedding_map}), we have
\begin{equation*}
    \widetilde{\mathcal{W}}_p(\mu,\nu) \leq \mathcal{W}_p(\mu,\nu).
\end{equation*}

\begin{thm}\label{thm:ext_convergence_RLD}
    Let $1 \leq p < \infty$. Let $\M \subset \R^n$ be a Riemannian submanifold. Let $\phi \in C^\infty(\M)$ with $C_\phi \defeq \int_\M \phi ~\dx{\op{vol}} < \infty$ and $\dx{\mu_\phi} \defeq C_\phi^{-1} e^{-\phi}~\dx{\op{vol}}$. Assume that $\M$ satisfies Assumption \ref{assum:geometric_iota} and Assumption \ref{assum:uniform_tubular}, that $\phi$ satisfies (\ref{eq:assum_phi_BE}), and that $\nabla \phi$ satisfies Assumption \ref{assum:vector_field}. Then for $\bm{X}^h_k \sim \hat{\mu}_k$ constructed by (\ref{eq:RLD_gem}), we have
    \begin{equation*}
        \widetilde{\mathcal{W}}_p\bc{\mu_\phi,\hat{\mu}_N} \leq \mathcal{O}\bc{e^{-T}+\exp(T^p)h^{1/2}}.
    \end{equation*}
\end{thm}
\begin{proof}
    Following the statements in Section \ref{sec:p_wasserstein_convergence_of_rld}, since $\M$ satisfies Assumption \ref{assum:geometric_iota} and $\phi$ satisfies (\ref{eq:assum_phi_BE}), the mixing error is bounded by
    \begin{equation}\label{eq:pf_ext_RLD_mix}
        \widetilde{\mathcal{W}}_p\bc{\mu_\phi,\mu_T} \leq \mathcal{W}_p\bc{\mu_\phi,\mu_T} \leq \mathcal{O}(e^{-T}).
    \end{equation}
    For the discretization error, since $\M$ satisfies Assumption \ref{assum:geometric_iota} and Assumption \ref{assum:uniform_tubular} and $\nabla \phi$ satisfies Assumption \ref{assum:vector_field}, Theorem \ref{thm:ext_p_strong_conv_GEM} implies that
    \begin{equation}\label{eq:pf_ext_RLD_discretization}
        \widetilde{\mathcal{W}}_p\bc{\mu_T,\hat{\mu}_N} \leq \E\bj{\norm*{\bm{X}_T-\bm{X}^h_N}^p}^{1/p} \leq \mathcal{O}\bc{\exp(T^p)h^{1/2}}.
    \end{equation}
    Therefore, (\ref{eq:pf_ext_RLD_mix}) and (\ref{eq:pf_ext_RLD_discretization}) together imply that
    \begin{equation*}
         \widetilde{\mathcal{W}}_p\bc{\mu_\phi,\hat{\mu}_N} \leq \widetilde{\mathcal{W}}_p\bc{\mu_\phi,\mu_T}+ \widetilde{\mathcal{W}}_p\bc{\mu_T,\hat{\mu}_N} \leq \mathcal{O}\bc{e^{-T}+\exp(T^p)h^{1/2}}. 
    \end{equation*}
\end{proof}

%% file: appendix_technique.tex
\section{Technique Lemmas}

\begin{lem}\label{lem:bound_normal_norm}
    Let $\bm{\xi} \sim \mathcal{N}(0,I_m)$. Then for any $1 \leq p < \infty$, there exists a constant $C > 0$ such that
    \begin{equation*}
        \E\bj{\norm{\bm{\xi}}^p} \leq \bc{\sqrt{m} + C\sqrt{p}}^p.
    \end{equation*}
\end{lem}
\begin{proof}
    For $\bm{\xi} = (\bm{\xi}_1,\ldots,\bm{\xi}_m) \in \R^m$, each $\bm{\xi}_i$ is subgaussian, and hence
    \begin{equation*}
        K \defeq \max_i \norm*{\bm{\xi}_i}_{\psi_2} <\infty.
    \end{equation*}
    It follows from \citet[Theorem 3.1.1]{vershynin2018high} that
    \begin{equation*}
        \norm*{\norm{\bm{\xi}} - \sqrt{m}}_{\psi_2} \leq C_0K,
    \end{equation*}
    for some constant $C_0 > 0$. Let $\bm{Y} = \norm{\bm{\xi}} - \sqrt{m}$. Then Proposition 2.2.6 in \citet{vershynin2018high} gives
    \begin{equation*}
        \E\bj{\abs{\bm{Y}}^p}^{1/p} \leq C_1\norm*{\bm{Y}}_{\psi_2}\sqrt{p} \leq C_1 C_0K\sqrt{p}. 
    \end{equation*}
    Finally, by Minkowski's inequality,
    \begin{equation*}
        \E\bj{\norm{\bm{\xi}}^p}^{1/p} \leq \sqrt{m} + \E\bj{\abs{\bm{Y}}^p}^{1/p} \leq \sqrt{m} + C\sqrt{p},
    \end{equation*}
    where $C = C_1 C_0K$. 
\end{proof}

\begin{lem}\label{lem:holder_p_geq_2}
    Let $h_0,\ldots,h_{N-1} \geq 0$ with $T = \sum_i h_i$, and let $a_0,\ldots,a_{N-1} \geq 0$. For any $p \geq 2$,
    \begin{equation*}
        \left(\sum_{i=0}^{N-1} h_i a_i^2\right)^{p / 2} \leq T^{\frac{p}{2}-1} \sum_{i=0}^{N-1} h_i a_i^p.
    \end{equation*}
    Moreover, for $p \geq 2$, if $a(t) \geq 0$, then
    \begin{equation*}
        \bc{\int_0^T a(t)^{2} ~\dx{t}}^{p/2}  \leq T^{\frac{p}{2} - 1} \int_0^T a(t)^p ~\dx{t}.
    \end{equation*}
\end{lem}
\begin{proof}
    The case $p=2$ is immediate. Assume $p>2$, and set $r = p/2 > 1$. Let $r^\prime$ be the conjugate exponent of $r$, i.e.,
    \begin{equation*}
        \frac{1}{r} + \frac{1}{r^\prime} = 1 \quad \Rightarrow \quad r^\prime = \frac{p}{p-2}.
    \end{equation*}
    Define
    \begin{equation*}
        x_i \defeq h_i a_i^2, \quad y_i \defeq h_i.
    \end{equation*}
    By H\"older's inequality,
    \begin{equation*}
        \sum_{i=0}^{N-1} x_i = \sum_{i=0}^{N-1} \bigl(x_i y_i^{-1/r^\prime}\bigr) y_i^{1/r^\prime} \leq \left(\sum_{i=0}^{N-1} x_i^r y_i^{1-r}\right)^{1/r} \left(\sum_{i=0}^{N-1} y_i\right)^{1/r^\prime}.
    \end{equation*}
    Since $x_i^r y_i^{1-r} = (h_i a_i^2)^r h_i^{1-r} = h_i a_i^{2r} = h_i a_i^p$ and $\sum_i y_i = \sum_i h_i = T$, we obtain
    \begin{equation*}
        \sum_{i=0}^{N-1} h_i a_i^2 \leq \left(\sum_{i=0}^{N-1} h_i a_i^p\right)^{2/p} T^{1-2/p}.
    \end{equation*}
    Therefore,
    \begin{equation*}
        \left(\sum_{i=0}^{N-1} h_i a_i^2\right)^{p/2} \leq T^{\frac{p}{2}-1} \sum_{i=0}^{N-1} h_i a_i^p.
    \end{equation*}
    The continuous inequality follows by the same argument, replacing sums with integrals and applying the integral form of H\"older's inequality. 
\end{proof}

\begin{lem}\label{lem:cut_func}
    For any $r > 0$, there exists a $\psi \in C^\infty([0,\infty))$ such that $\psi \in [0,1]$,
    \begin{equation*}
        \psi|_{[0,r/2]} \equiv 1,\quad \psi|_{[r,\infty)} \equiv 0,\quad \text{ and } \sup_{t \geq 0}\abs{\psi^\prime(t)} < \infty.
    \end{equation*}
\end{lem}
\begin{proof}
    Let
    \begin{equation*}
        \phi(s) \defeq \begin{cases}
            e^{-1/s},&s > 0,\\
            0,& s\leq 0.
        \end{cases}
    \end{equation*}
    Then $\phi \in C^\infty(\R)$. Define $\theta(s) \defeq \frac{\phi(s)}{\phi(s) + \phi(1-s)}$. Then $\theta \in C^\infty(\R)$, $\theta \in [0,1]$, and
    \begin{equation*}
        \theta(s) = 0~(s \leq 0),\quad \theta(s) = 1~(s \geq 1).
    \end{equation*}
    In particular, $\theta^\prime(s) = 0$ for $s \in (-\infty,0] \cup [1,\infty)$, and hence
    \begin{equation*}
        C_\theta \defeq \sup_{s \in \R} \abs{\theta^\prime(s)} <\infty.
    \end{equation*}
    Now define
    \begin{equation*}
        \psi(t) \defeq 1 - \theta\bc{\frac{2t}{r} - 1},\quad \forall~ t \in [0,\infty).
    \end{equation*}
    Then $\psi \in C^\infty([0,\infty))$ and $\psi \in [0,1]$. By construction, $\psi(t) = 1$ for $t \in [0,r/2]$ and $\psi(t) = 0$ for $t \geq r$. Moreover,
    \begin{equation*}
        \psi^\prime(t) = - \frac{2}{r}\theta^\prime\bc{\frac{2t}{r} - 1},
    \end{equation*}
    so
    \begin{equation*}
        \sup_{t \geq 0}\abs{\psi^\prime(t)} \leq \frac{2}{r} \sup_{s \in \R}\abs{\theta^\prime(s)} = \frac{2C_\theta}{r} < \infty. 
    \end{equation*}
\end{proof}

\begin{lem}\label{lem:lip_project}
    Let $\Pi \colon (0,\infty) \times \R^n \sto \R^n$ be defined by
    \begin{equation*}
        \Pi(r,z) \defeq \begin{cases}
            z,& \norm{z} \leq r,\\
            r\frac{z}{\norm{z}},& \norm{z} > r.
        \end{cases}
    \end{equation*}
    Then:
    \begin{enumerate}[label=(\roman*)]
        \item For each fixed $z \in \R^n$, the map $\Pi(\cdot,z) \colon (0,\infty) \sto \R^n$ is $1$-Lipschitz continuous.
        \item For each fixed $r>0$, the map $\Pi(r,\cdot) \colon \R^n \sto \R^n$ is $1$-Lipschitz continuous.
    \end{enumerate}
\end{lem}
\begin{proof}
    \begin{enumerate}[label=(\roman*)]
        \item Fix $z \in \R^n$ and set $\rho \defeq \norm{z}$. Let $r_1,r_2 > 0$. If $\rho \leq \min\bb{r_1,r_2}$, then $\Pi(r_1,z)=\Pi(r_2,z)=z$ and hence $\norm{\Pi(r_1,z)-\Pi(r_2,z)}=0$. If $\rho \geq \max\bb{r_1,r_2}$, then $\Pi(r_i,z)=r_i z/\rho$, so
        \begin{equation*}
            \norm{\Pi(r_1,z)-\Pi(r_2,z)} = \norm*{(r_1-r_2)\frac{z}{\rho}} = \abs{r_1-r_2}.
        \end{equation*}
        If $r_1 \geq \rho \geq r_2$, then $\Pi(r_1,z)=z$ and $\Pi(r_2,z)=r_2 z/\rho$, so
        \begin{equation*}
            \norm{\Pi(r_1,z)-\Pi(r_2,z)} = \norm*{z-\frac{r_2}{\rho}z} = \rho-r_2
            \leq r_1-r_2.
        \end{equation*}
        In all cases, $\norm{\Pi(r_1,z)-\Pi(r_2,z)} \leq \abs{r_1-r_2}$, which proves the claim.

        \item Fix $r>0$. The map $\Pi(r,\cdot)$ is the metric projection onto the closed ball $\clo{B(r)}=\bb{x\in\R^n:\norm{x}\le r}$. Thus, for any $z\in\R^n$ and any $x\in\clo{B(r)}$, \citet[Proposition 1.1.9]{bertsekas2009convex} shows
        \begin{equation*}
            \inn{z-\Pi(r,z),\,x-\Pi(r,z)} \leq 0.
        \end{equation*}
        Let $y,z\in\R^n$, and set $u \defeq \Pi(r,y)$ and $v \defeq \Pi(r,z)$. Applying the above inequality with $(z,x)=(y,v)$ and $(z,x)=(z,u)$ yields
        \begin{equation*}
            \inn{y-u,\,v-u} \leq 0, \quad \inn{z-v,\,u-v} \leq 0.
        \end{equation*}
        Adding these inequalities gives
        \begin{equation*}
            \norm{u-v}^2 \leq \inn{y-z,\,u-v} \leq \norm{y-z}\norm{u-v}.
        \end{equation*}
        If $u=v$ there is nothing to prove; otherwise dividing by $\norm{u-v}$ implies $\norm{u-v} \leq \norm{y-z}$, that is,
        \begin{equation*}
            \norm*{\Pi(r,y)-\Pi(r,z)} \leq \norm{y-z}. 
        \end{equation*}
    \end{enumerate}
\end{proof}

\begin{lem}\label{lem:Moore_Penrose}
    Let $A \in \R^{k \times n}$ ($k < n$) with $\rank A = k$. Then $A$ admits the Moore–-Penrose right inverse
    \begin{equation*}
        A^\dagger = A^T \bc{AA^T}^{-1},
    \end{equation*}
    and $P = A^\dagger A$ is the orthogonal projection onto $(\ker A)^\perp$.
\end{lem}
\begin{proof}
    Since $\rank A = k$, the matrix $AA^T \in \R^{k \times k}$ is invertible, and hence the Moore--Penrose right inverse $A^\dagger$ is well-defined. For $P = A^\dagger A$, we have
    \begin{equation*}
        P^T = P,\quad P^2 = P.
    \end{equation*}
    Therefore, $P$ is the orthogonal projection onto $\Img P$. Moreover, since $A$ has full rank,
    \begin{equation*}
        \Img P = \Img A^\dagger = \Img A^T = (\ker A)^\perp. 
    \end{equation*}
\end{proof}

\begin{lem}\label{lem:bound_operator}
    Let $(\M,g)$ be an $m$-dimensional Riemannian manifold. For any $x \in \M$, fix a basis $\bb{E_i(x)}_{i=1}^m$ of $T_x\M$ and let $g_{ij}(x) = g(x)(E_i(x), E_j(x))$. Assume that there exists $\lambda_{\min} > 0$ such that the matrix $g(x) = \bc{g_{ij}(x)}$ satisfies
    \begin{equation*}
        g(x) \succeq \lambda_{\min}I_m.
    \end{equation*}
    Let $F \in \Gamma(\M, \op{Hom}(T\M \otimes T\M, \M \times \R^n))$, i.e., for any $x \in \M$, $F_x \colon T_x\M \times T_x\M \sto \R^n$ is a bilinear map. If there exists a constant $C > 0$ such that for any $x \in \M$,
    \begin{equation*}
        \norm*{F_x(E_i(x),E_j(x))} \leq C,
    \end{equation*}
    then
    \begin{equation*}
        \sup_{x \in \M} \norm*{F_x}_{\op{op}} \leq \frac{Cm}{\lambda_{\min}}.
    \end{equation*}
\end{lem}
\begin{proof}
    Fix $x \in \M$. For any $u,v \in T_x\M$, write
    \begin{equation*}
        u = \sum_{i=1}^m u_iE_i,\quad v = \sum_{j=1}^m v_j E_j.
    \end{equation*}
    Then
    \begin{equation*}
        F_x(u,v) = \sum_{i,j = 1}^m u_iv_j F_x(E_i(x),E_j(x)),
    \end{equation*}
    and therefore
    \begin{equation*}
        \norm*{F_x(u,v)} \leq C \sum_{i,j = 1}^m \abs{u_iv_j} = C\bc{\sum_{i=1}^m \abs{u_i}}\bc{\sum_{j=1}^m \abs{v_j}} \leq Cm \abs{u}_2\abs{v}_2,
    \end{equation*}
    where $\abs{u}_2 \defeq (\sum_i u_i^2)^{1/2}$ and $\abs{v}_2 \defeq (\sum_j v_j^2)^{1/2}$, and the last inequality uses the norm inequality $\ell_1 \leq \sqrt{m}\,\ell_2$.

    On the other hand, since $g(x) \succeq \lambda_{\min}I_m$,
    \begin{equation*}
        \norm{u}^2 = (u_1,\ldots,u_m)\, g(x)\, (u_1,\ldots,u_m)^T \geq \lambda_{\min}\abs{u}_2^2,
    \end{equation*}
    and similarly $\norm{v}^2 \geq \lambda_{\min}\abs{v}_2^2$. Hence,
    \begin{equation*}
        \norm*{F_x(u,v)} \leq Cm\abs{u}_2\abs{v}_2 \leq \frac{Cm}{\lambda_{\min}}\norm{u}\norm{v}.
    \end{equation*}
    It follows that
    \begin{equation*}
        \norm*{F_x}_{\op{op}} = \sup_{u,v \neq 0} \frac{\norm*{F_x(u,v)}}{\norm{u}\norm{v}} \leq \frac{Cm}{\lambda_{\min}}.
    \end{equation*}
    Since $C$ and $\lambda_{\min}$ are independent of $x$, we conclude that
    \begin{equation*}
        \sup_{x \in \M} \norm*{F_x}_{\op{op}} \leq \frac{Cm}{\lambda_{\min}}. 
    \end{equation*}
\end{proof}
\begin{rmk}
    The same argument extends obviously to the trilinear case $F_x \colon T_x\M \times T_x\M \times T_x\M \sto \R^n$.
\end{rmk}

\begin{lem}\label{lem:ext_hess_intr_hess}
    Let $\M \subset \R^n$ be a Riemannian submanifold with Levi--Civita connection $\nabla$. Let $f \in C^\infty(\R^n)$ and $g = f|_\M \in C^\infty(\M)$. Then
    \begin{equation*}
        \nabla^2 g(X,Y) = D^2 f[X,Y] + Df[\mathrm{II}(X,Y)],
    \end{equation*}
    where $\mathrm{II}$ is the second fundamental form.
\end{lem}
\begin{proof}
    By the definition of the Hessian and the Gauss formula \citep{lee2018introduction}, for any $X,Y \in \Gamma(T\M)$,
    \begin{align*}
        \nabla^2 g(X,Y)
        &= X\bigl(Y(g)\bigr) - (\nabla_XY)(g) \\
        &= X\bigl(Y(f)\bigr) - \bigl(D_XY - \mathrm{II}(X,Y)\bigr)(f) \\
        &= \bc{X\bigl(Y(f)\bigr) - (D_XY)(f)} + \mathrm{II}(X,Y)(f) \\
        &= D^2 f[X,Y] + Df[\mathrm{II}(X,Y)]. 
    \end{align*}
\end{proof}

\begin{lem}\label{lem:tensor_nabla_second}
    Let $\M \subset \R^n$ be a Riemannian submanifold. Then the covariant derivative $\nabla \mathrm{II}$ is pointwise determined, i.e., for any $x \in \M$ and any $X,Y,Z,\tilde{X},\tilde{Y},\tilde{Z} \in \Gamma(T\M)$ with $X_x = \tilde{X}_x$, $Y_x = \tilde{Y}_x$, and $Z_x = \tilde{Z}_x$,
    \begin{equation*}
        (\nabla \mathrm{II})(X,Y,Z)(x) = (\nabla \mathrm{II})(\tilde{X},\tilde{Y},\tilde{Z})(x).
    \end{equation*}
\end{lem}
\begin{proof}
    It suffices to show that $\nabla \mathrm{II}$ is an $N\M$-valued $(0,3)$-tensor, i.e., it is $C^\infty(\M)$-trilinear. By definition,
    \begin{equation*}
        (\nabla \mathrm{II})(X,Y,Z) = \bc{D_Z\mathrm{II}(X,Y)}^\perp - \mathrm{II}(\nabla_ZX,Y) - \mathrm{II}(X,\nabla_ZY).
    \end{equation*}
    By \citet[Proposition 8.1]{lee2018introduction}, $\mathrm{II}$ is $C^\infty(\M)$-bilinear and symmetric. Let $f \in C^\infty(\M)$. For the $X$-position, using linearity and the Leibniz rule of the connection,
    \begin{align*}
        \bc{D_Z\mathrm{II}(fX,Y)}^\perp
        &= \bc{D_Z\bigl(f\mathrm{II}(X,Y)\bigr)}^\perp \\
        &= \bc{Z(f)\mathrm{II}(X,Y) + f D_Z\mathrm{II}(X,Y)}^\perp \\
        &= Z(f)\mathrm{II}(X,Y) + f\bc{D_Z\mathrm{II}(X,Y)}^\perp,
    \end{align*}
    and
    \begin{equation*}
        \mathrm{II}(\nabla_Z(fX),Y) = \mathrm{II}\bigl(Z(f)X + f\nabla_ZX,\,Y\bigr) = Z(f)\mathrm{II}(X,Y) + f\mathrm{II}(\nabla_ZX,Y),
    \end{equation*}
    while
    \begin{equation*}
        \mathrm{II}(fX,\nabla_ZY) = f\mathrm{II}(X,\nabla_ZY).
    \end{equation*}
    Substituting these identities into the definition of $\nabla \mathrm{II}$ gives
    \begin{equation*}
        (\nabla \mathrm{II})(fX,Y,Z) = f(\nabla \mathrm{II})(X,Y,Z).
    \end{equation*}
    Since $\mathrm{II}$ is symmetric, the same argument applies to the $Y$-position, i.e., $(\nabla \mathrm{II})(X,fY,Z) = f(\nabla \mathrm{II})(X,Y,Z)$.

    For the $Z$-position, by linearity of the connection,
    \begin{equation*}
        \bc{D_{fZ}\mathrm{II}(X,Y)}^\perp = \bc{fD_Z\mathrm{II}(X,Y)}^\perp = f\bc{D_Z\mathrm{II}(X,Y)}^\perp,
    \end{equation*}
    and
    \begin{equation*}
        \mathrm{II}(\nabla_{fZ}X,Y) = f\mathrm{II}(\nabla_ZX,Y),\quad\mathrm{II}(X,\nabla_{fZ}Y) = f\mathrm{II}(X,\nabla_ZY).
    \end{equation*}
    Hence $(\nabla \mathrm{II})(X,Y,fZ) = f(\nabla \mathrm{II})(X,Y,Z)$. Additivity in each position is immediate. 
\end{proof}